\documentclass[final,12pt]{colt2025}

\usepackage[normalem]{ulem}

\title[Optimal Online Bookmaking for Any Number of Outcomes]{Optimal Online Bookmaking for Any Number of Outcomes}
\usepackage{times}
\coltauthor{\Name{Hadar Tal} \Email{hadar.tal@mail.huji.ac.il}\and
 \Name{Oron Sabag} \Email{oron.sabag@mail.huji.ac.il}\\
 \addr School of Computer Science and Engineering, The Hebrew University of Jerusalem, Israel}

\usepackage{mathtools}

\usepackage{pgfplots}
\pgfplotsset{width=10cm,compat=1.9}
\usepackage{tikz}
\usetikzlibrary {arrows.meta}
\usepackage{pgfplots}
\usepackage{standalone}
\usepackage{wrapfig}

\newcommand{\teq}[1]{\texorpdfstring{#1}{Lg}}

\usepackage{dsfont}

\usepackage{float}
\usepackage{tgcursor}
\newcommand{\algfont}[1]{{\fontfamily{qcr}\selectfont #1}}

\SetCommentSty{mycommfont}

\usepackage{enumitem}

\usepackage[toc,page]{appendix}

\usepackage{mdframed} 
\def\thmMargin{8pt}
\newmdenv[
    linewidth=1pt, 
    linecolor=black, 
    backgroundcolor=white, 
    innerleftmargin=\thmMargin, 
    innerrightmargin=\thmMargin, 
    innertopmargin=\thmMargin, 
    innerbottommargin=\thmMargin,
    skipabove=10pt, 
    skipbelow=10pt, 
]{TheoremBox}

\newtheorem{maintheorem}{Theorem}
 
\newenvironment{boxedmaintheorem}[1][]{
    \begin{TheoremBox}
    \begin{maintheorem}[#1]
}{
    \end{maintheorem}
    \end{TheoremBox}
}

\newcommand{\Nfield}[0]{\mathbb{N}}

\newcommand{\Nplus}[0]{\Nfield_{\scriptscriptstyle{+}}}

\newcommand{\Rfield}[0]{\mathbb{R}}
\newcommand{\Rdim}[1]{\Rfield^{#1}}
\newcommand{\Rplus}[1]{\Rfield^{#1}_{\scriptscriptstyle{+}}}
\newcommand{\Rpplus}[1]{\Rfield^{#1}_{\scriptscriptstyle{++}}}

\newcommand{\simplex}[1]{\Delta^{#1}}
\newcommand{\intsimplex}[1]{\operatorname{Int}\left(\simplex{#1}\right)}

\newcommand{\roots}[1]{\text{Roots}\left( #1 \right)}
\newcommand{\argmaxroot}[1]{\arg \max \roots{#1}}
\newcommand{\vvset}[1]{\mathfrak{P}_{{#1}}}
\newcommand{\nset}[1]{ \left[ #1  \right]}

\newcommand{\stdbasis}[1]{\mathcal{E}_{#1}}
\newcommand{\basis}[1]{ \mathbf{e}_{#1}}

\newcommand{\indneq}[3]{#1 \in #2 \setminus \{ #3 \}}

\newcommand{\hermp}{\mathrm{He}}

\newcommand{\potfunction}[2]{{\mathcal{V}}_{{#1}}}
\newcommand{\pot}[3]{\mathcal{V}_{{#1}} \left( #3 \right)}

\newcommand{\optimalloss}[2]{L^{\star}_{{#1}, {#2}}}

\newcommand{\esp}[2]{
    \sigma_{#1} \left( #2 \right)
}

\newcommand{\denomfunction}[2]{
    \mathcal{D}_{{{#1},{#2}}}
}
\newcommand{\denom}[3]{
    \denomfunction{#1}{#2} \left( #3 \right)
}

\newcommand{\biaseddenomfunction}[2]{
    \mathcal{F}_{#1, #2}
}
\newcommand{\biaseddenom}[3]{
    \biaseddenomfunction{#1}{#2} \left( #3 \right)
}

\newcommand{\numfunction}[2]{
    \mathcal{N}_{{{#1},{#2}}}
}
\newcommand{\num}[3]{
    \numfunction{#1}{#2} \left( #3 \right)
}

\newcommand{\ppolyfunction}[2]{
    \mathcal{P}_{{#1},{#2}} 
}
\newcommand{\ppoly}[3]{
    \ppolyfunction{#1}{#2} \left( #3 \right)
}

\newcommand{\hatppolyfunction}[2]{
    \widehat{\mathcal{P}}_{{#1},{#2}} 
}
\newcommand{\hatppoly}[3]{
    \hatppolyfunction{#1}{#2} \left( #3 \right)
}

\newcommand{\betapolyfunction}[2]{
    \widetilde{\mathcal{P}}_{{#1},{#2}} 
}
\newcommand{\betapoly}[3]{
    \betapolyfunction{#1}{#2} \left( #3 \right)
}

\newcommand{\fallingfact}[2]{
    {#1}^{\underline{{#2}}} 
}

\newcommand{\risingfact}[2]{
    {#1}^{\overline{{#2}}} 
}

\newcommand{\unstirling}[2]{
   \genfrac{[}{]}{0pt}{}{#1}{#2}  
}

\newcommand{\stirling}[2]{
    \mathrm{s}\left( #1,  #2 \right) 
}

\newcommand{\minind}[1]{^{\setminus #1}}
\newcommand{\univec}[2]{#1 \cdot\kern-0.2em \mathbf{1}_{\scriptscriptstyle#2}}
\newcommand{\zerovec}[1]{\mathbf{0}_{\scriptscriptstyle#1}}

\newcommand{\seq}[2]{#1^{ {#2} }}
\newcommand{\seqelem}[2]{{ #1 }_{ #2 }}
\newcommand{\vecind}[2]{#1 (#2)}

\newcommand{\naive}[0]{na\"{\i}ve }

\newcommand{\bm}[0]{\Psi}
\newcommand{\gamb}[0]{\Lambda}
\newcommand{\overround}[0]{\Gamma}
\newcommand{\eqdef}[0]{\coloneqq}
\newcommand{\opt}[1]{{#1}^{\star}}
\newcommand{\optr}[0]{\opt{r}}

\newcommand{\corr}[0]{\mathcal{Z}}

\begin{document}

\maketitle

\begin{abstract}
We study the \emph{Online Bookmaking} problem, where a bookmaker dynamically updates betting odds on the possible outcomes of an event. In each betting round, the bookmaker can adjust the odds based on the cumulative betting behavior of gamblers, aiming to maximize profit while mitigating potential loss. We show that for any event and any number of betting rounds, in a worst-case setting over all possible gamblers and outcome realizations, the bookmaker’s optimal loss is the largest root of a simple polynomial. Our solution shows that bookmakers can be as fair as desired while avoiding financial risk, and the explicit characterization reveals an intriguing relation between the bookmaker’s regret and Hermite polynomials. We develop an efficient algorithm that computes the optimal bookmaking strategy: when facing an optimal gambler, the algorithm achieves the optimal loss, and in rounds where the gambler is suboptimal, it reduces the achieved loss to the \emph{optimal opportunistic} loss, a notion that is related to subgame perfect Nash equilibrium. The key technical contribution to achieve these results is an explicit characterization of the \emph{Bellman-Pareto frontier}, which unifies the dynamic programming updates for Bellman's value function with the multi-criteria optimization framework of the Pareto frontier in the context of vector repeated games.
\end{abstract}


\section{Introduction}\label{sec:introduction}

The online betting industry, particularly in the realm of sports betting, 
has experienced remarkable growth in recent years \citep{Zion2024SportsBetting}. 
Central to this market is the role of bookmakers, 
who set odds for uncertain events and accept wagers, 
while managing their exposure to risk. 
In online bookmaking, odds are dynamically adjusted based on incoming bets and evolving information, 
making it an intricate algorithmic task. 
The challenge is to dynamically calibrate the odds such that the bookmaker guarantees a profit regardless of the event's outcome.
In real-world settings, betting markets exhibit a wide range in the number of possible outcomes $K$.
For instance, in a National Football League (NFL) game, $K = 2$, 
as gamblers can bet on either of the two competing teams,
while the Eurovision Song Contest presents about $K = 40$ possible outcomes,
each indicating a possible winner of the contest.
As $K$ increases, the bookmaker faces greater uncertainty and must balance risk across a broader set of outcomes.

We consider \emph{online bookmaking} as a sequential game between a bookmaker and a gambler \citep{bhatt2025optimal},
noting that the behavior of multiple gamblers can be unified as a single gambler from the bookmaker's perspective.
The game unfolds over a horizon \( T \) and involves an event with \( K \) possible outcomes, where exactly one outcome $k \in [K]$ materializes. At time $t$, the bookmaker assigns the odds \( \gamma_t(k) \) for each possible outcome \( k \), and the gambler distributes a betting unit as a probability vector $q_t$, where $q_t(k)$ is the bet placed on outcome $k$. Following $T$ rounds of this procedure, the event's outcome is disclosed. If $k$ is indeed the outcome of the event, the gambler receives a payout of \( \gamma_t(k) \cdot q_t(k) \), and otherwise the gambler loses their placed bet $q_t(k)$.

The bookmaker's fairness is often measured using the \emph{overround} parameter,
also known as the bookmaker margin or \emph{vigorish} \citep{cortis2015expected},
\begin{equation}
    \overround = \sum_{k \in [K]} \frac1{\gamma_t(k)},
\end{equation}
assumed here to be time-independent.
The bookmaker's odds are considered \emph{fair} when $\Gamma = 1$, but in practice, $\Gamma > 1$.
For instance, in UK football betting, 
the overround typically\footnote{According to \citeauthor{football_data_overround}} 
remains below \(1.1\) for \( K=2 \) markets.
As the number of possible outcomes \(K\) increases, the overround also tends to rise:
reaching $\Gamma \approx 1.15$ for \( K=3 \) markets (e.g., football match winner), 
and potentially up to  \(1.6\) for \( K=24 \) markets (e.g., final score predictions).
The overround also serves to define the bookmaker's offered odds as a probability vector  $r_t(k) = \frac{1}{\Gamma \gamma_t(k)}$.

In each round, the bookmaker collects a betting unit from the gambler, 
resulting in a total of $T$ units collected over the game.  
The return to the gambler upon collecting the bets $q_1,\dots,q_T$ is 
\[
    \frac{1}{\Gamma} \sum_{k=1}^{K} \mathds{1}\{I = k\} \sum_{t=1}^{T} \frac{q_t(k)}{r_t(k)},
\]
where \(\mathds{1}\{I = k\}\) is the indicator function of the materialized outcome. 
In this paper, we consider the bookmaker's worst-case scenario:  
(i) the gambler aims to maximize the bookmaker's loss, and 
(ii) the event's outcome results in the largest possible return. 
The optimal bookmaking loss is then defined as
\begin{equation}\label{eq:intro_opt_loss}
    \optimalloss{T}{K}
    = \inf_{ r_1, \ldots r_T } \; \max_{ q_1, \ldots , q_T } \; \max_{ k \in \nset{K} } \; 
    \sum_{t=1}^{T}  
        \frac{\vecind{\seqelem{q}{t}}{k}}
        {\vecind{\seqelem{r}{t}}{k}},
\end{equation}
where $r_t,q_t$ depend on past offered bets and wagers (formally defined in \sectionref{sec:general_problem_formulation}).
Note that the overround is omitted from \eqref{eq:intro_opt_loss}, but the optimal loss still directly characterizes the fundamental limits of the overround.
In particular, after collecting $T$ money units and paying the worst-case loss to the gambler, 
the bookmaker is left with $T - \frac1{\Gamma}\optimalloss{T}{K}$ money units. 
Thus, to guarantee a positive gain, the bookmaker must choose the overround to be $\Gamma > \frac{\optimalloss{T}{K}}{T}$. 
We will show the convergence of $\frac{\optimalloss{T}{K}}{T}$ to $1$ as $T$ grows for all $K$. This implies that, even when facing events with many possible outcomes,
bookmakers can choose the overround to be arbitrarily close to $1$, given a sufficiently large $T$.

The online bookmaking problem is a \emph{repeated zero-sum game} 
falling under the broader framework of game theory pioneered by 
Von Neumann's Minimax Theorem \citeyearpar{vonNeumann1928}.
Blackwell’s approachability theorem \citeyearpar{blackwell1956analog} and 
Hannan’s no-regret algorithm for competing with the best fixed action \citeyearpar{Hannan1957} 
gave rise to online learning algorithms with provable regret bounds.
Unlike the classical experts problem \citep{cover1966behavior,vovk1990aggregating,littlestone1994weighted} 
where the optimal adversarial strategy for more than $4$ experts is unknown \citep{gravin2014optimal}, our setting benefits from a fully characterized optimal adversarial strategy \citep[][detailed below]{bhatt2025optimal}. 
Sequential prediction and adversarial decision-making are central to online learning, 
and incorporate concepts such as potential-based strategies \citep{nicol_cesabianchi_2003} and approachability \citep{abernethy2011blackwell}.
These settings have also been examined from an information-theoretic perspective \citep{merhav2002universal,Sht87},
where the minimax log-loss plays a key role, corresponding to the optimal code length in lossless compression \citep[Ch.~9]{LugosiPrediction}.

The optimal bookmaking loss in \eqref{eq:intro_opt_loss} is difficult to compute for several reasons. 
First, the bookmaker's strategy is a sequence of mappings that depend on the entire history of placed bets, 
resulting in a domain that grows exponentially with time. 
Second, the event's outcome becomes apparent only after the final betting round. 
In particular, the maximization over $k$ implies that the bookmaker does not know the \emph{true} objective function during the game,
turning it into a vector-valued game in which the bookmaker must balance the losses across outcomes throughout the betting process. 
Lastly, each possible objective function is unbounded and non-Lipschitz. 
While many online learning problems can be cast as instances of online convex optimization (OCO) \citep[see][]{hazan2019oco},  
the regret analyses in OCO hinge on losses being bounded and Lipschitz.  
In our bookmaking setting, however, neither condition holds.
In particular, an optimal bookmaking strategy may approach the simplex boundary.
Likewise, the loss is not convex–concave, ruling out minimax theorems such as \citet{sion1958minimax}.
Constant-factor considerations further complicate the analysis: 
while a loss of \(\$1\mathrm{M}\) or \(\$100\mathrm{M}\) might be considered equivalent in an asymptotic sense, their practical implications differ significantly.
To address the latter, one might formulate the problem as an \emph{extensive-form} game with a discrete action space \citep{Kroer2015DiscretizationOC}; 
however, the exponential growth of the game tree makes backward induction computationally infeasible.

Most sports betting studies adopt the gambler’s perspective, 
emphasizing outcome prediction and bankroll management.
\emph{Kelly betting} \citep{Kelly1956,RotandoThorp1992}, 
which optimizes long-term wealth through bet sizing, 
is commonly used to model gambler behavior in bookmaking algorithms \citep[e.g.,][]{zhu2024}.
Various techniques have been developed to adjust the bookmaker's \emph{implied probabilities},
such as additive, multiplicative, \emph{Shin}, and power methods \cite[see][]{clarke2017adjusting}. 
Other works examine dynamic pricing strategies
\citep[e.g.,][]{rollin2018riskneutralpricinghedginginplay,lorig2021optimal},
but do not consider adversarial betting behavior.
Most closely related to our work is \cite{bhatt2025optimal}, 
which introduced the problem of online bookmaking and solved it 
in terms of an explicit loss and an optimal algorithm for binary games, i.e., $K=2$. 
In contrast, we solve the online bookmaking problem for all $K$ and extend it to the opportunistic setting when the gambler is suboptimal (see \sectionref{subsec:problem_opportunistic}).

\paragraph{Contributions}
We present an analysis of the online bookmaking problem. 
\begin{itemize}
    \item 
    We derive an exact and computable expression for the optimal bookmaking loss $\optimalloss{T}{K}$ 
    for any number of outcomes \( K \) and betting rounds \( T \). 
    \item 
    We identify the regret  
    \(
        R_{T,K} = \optimalloss{T}{K} - T,
    \) measured with respect to a bookmaker that foresees the aggregated wagers but does not know the event's outcome.
    We prove that the regret scales as $\sqrt{T}$ and, for any $K$, 
    its scaling factor converges to the largest root of the $K$-th Hermite polynomial. 
    This establishes the fact that bookmakers can be as fair as desired.
    \item 
    We introduce the notion of opportunistic bookmaking, 
    which generalizes the setting above by accounting for the possibility that the gambler’s past behavior was not optimal. 
    We characterize the \emph{optimal opportunistic bookmaking loss} and present an efficient algorithm that achieves this loss at any time.
    \item 
    A key technique to achieve our explicit results is the characterization of the \emph{Bellman-Pareto Frontier},
    corresponding to the union of all optimal future payouts in vector games.
\end{itemize}

\paragraph{Paper Structure}
\sectionref{sec:problem_setting} includes the preliminaries 
and the formulation of the online bookmaking problem, along with its generalization to opportunistic bookmaking strategies. 
\sectionref{sec:main_results} presents our main results. 
\sectionref{sec:proof_outline} provides the main ideas and proofs.
The paper is concluded in \sectionref{sec:conclusion}.

\section{Problem Setting} \label{sec:problem_setting}

\subsection{Preliminaries} \label{sec:preliminaries}

\paragraph{Notation}
Constants are denoted by uppercase letters (e.g. \(T\)).
For \(K \in \Nplus\), we define \(\nset{K} = \{1, \dots, K\}\). 
The standard basis of \(\Rdim{K}\) is denoted \(\stdbasis{K} = \{\basis{1},\ldots , \basis{K}\}\).
Vectors are denoted with lowercase letters, e.g., \(r\).  
Vector sequences are expressed as \(\seq{r}{T} = (\seqelem{r}{1}, \dots, \seqelem{r}{T})\), 
where subscripts indicate time indices, i.e.,  
\(\seqelem{r}{t}\) is the \(t\)-th element of \(\seq{r}{T}\).  
For vectors \(r, r_t \in \Rdim{K}\), their \(k\)-th elements are \(\vecind{r}{k}\) and \(r_t(k)\), respectively.  
The notation \(r\minind{k}\) denotes \(r\) without its \(k\)-th element.
The simplex in \(\Rdim{K}\) is denoted by \(\simplex{} \equiv \simplex{K-1}\) when the dimension is clear from context.
For vectors $x,y \in \Rdim{K}$, the notation $x \oslash y$ represents element-wise division.
The \emph{weak coordinate-wise partial order} on \(\Rdim{K}\) is denoted by \(\succeq\) and defined as \( x \succeq y \) 
if \( \vecind{x}{k} \geq \vecind{y}{k} \) for all \( k \in \nset{K} \),
while \emph{strict coordinate-wise partial order} is denoted by \(\succ\) and holds if \( x \succeq y \) 
and there exists some \( k \in \nset{K} \) such that \( \vecind{x}{k} > \vecind{y}{k} \). 
For $x \in \Rfield$ and $m \in \Nfield$,
the \emph{falling factorial} is defined as
\(
    \fallingfact{x}{m}
    \, \eqdef \, \prod_{i=0}^{m-1} (x - i),
\)
and the \emph{rising factorial} is defined as
\(
    \risingfact{x}{m}
    \, \eqdef \, \prod_{i=0}^{m-1} (x + i).
\)

\begin{definition}[Elementary Symmetric Polynomial] \label{def:esp}
For $s \in \Rdim{K}$ and $m \in \Nfield$, the \emph{elementary symmetric polynomial (ESP)}
of degree \(m\) over the elements of \(s\) is 
\begin{equation}    
    \label{eq:def_of_esp}
    \esp{m}{s} 
    \eqdef
    \displaystyle\sum_{\mathfrak{I} \in \binom{\nset{K}}{m}}
        \: \prod_{k \in \mathfrak{I}} \: \vecind{s}{k},
\end{equation}
where \( \binom{\nset{K}}{m} \) denotes the set of all subsets of \(\nset{K}\) with cardinality \(m\). 
\end{definition}
\begin{example}
The non-zero elementary symmetric polynomials of \(s \in \Rdim{3}\) are:
\begin{equation*}
\begin{aligned}
    \esp{0}{s} &= 1, 
    &\quad \esp{1}{s} &= \vecind{s}{1} + \vecind{s}{2} + \vecind{s}{3},      
    \\
    \esp{2}{s} &= \vecind{s}{1} \vecind{s}{2} + \vecind{s}{1} \vecind{s}{3} + \vecind{s}{2} \vecind{s}{3},  
    & \esp{3}{s} &= \vecind{s}{1} \vecind{s}{2} \vecind{s}{3}.
\end{aligned}
\end{equation*}
\end{example}

\subsection{Problem Formulation} \label{sec:general_problem_formulation}

The \emph{online bookmaking problem} is modeled as a repeated game of \( T \) rounds with a fixed number of outcomes \( K \) and a fixed overround parameter \( \Gamma \geq 1\).
In each round \( t \), 
the bookmaker selects a probability distribution \( r_t \in \simplex{K-1} \),
and publishes betting odds 
\(
\gamma_t(k)=\frac{1}{\Gamma r_t(k)}
\)
for every outcome $k \in [K]$.
The gambler then chooses $q_t \in \simplex{K-1}$, which represents the distribution of its bets across the different outcomes. 
The bookmaker and the gambler repeat this procedure $T$ times,
where at each round they can use the previous odds $r^{t-1}$ and placed bets $q^{t-1}$. 
Formally, the bookmaker's strategy \( \bm^T \) is a sequence of mappings, defined by 
\begin{equation}\label{eq:def_strategy}
    \bm_t : \simplex{t-1} \times \Delta^{t-1} \to \Delta,  \quad t\in [T],
\end{equation}  
that determines the odds \( r_t \) as 
\(  
    r_t = \bm{}_t \, (r^{t-1},\, q^{t-1})
\).
Similarly, the gambler's strategy \( \gamb^{T} \) is a sequence of mappings, defined by
\begin{equation}
    \gamb_t : \simplex{t} \times \Delta^{t-1} \to \simplex{}, \quad t\in [T], 
\end{equation}
which produces the bet vector \( q_t \) as  
\(
    q_t = \gamb_t \, (r^t,\, q^{t-1})
\).

We consider the worst-case scenario, which defines the \emph{optimal bookmaking loss} as
\begin{equation}
    \label{eq:optimal_bookmaker_problem_setting}
    \optimalloss{T}{K}
    = \inf_{ \Psi^T } \; \max_{ q^T\in\simplex{T} } \; \max_{ k \in \nset{K} } \; 
        \sum_{t=1}^{T}  
        \frac{\vecind{\seqelem{q}{t}}{k}}
        {\vecind{\seqelem{r}{t}}{k}},
\end{equation}
where \( r^T \) is determined by \(\Psi^T\). 
Note that the maximization over $\gamb^T$ is replaced with $q^T \in \simplex{T}$ without loss of optimality, 
as the gambler acts as an adversary that responds to the bookmaker's strategy. 
Also, since the bookmaker's strategy is deterministic, we consider, without loss of optimality, a bookmaker's strategy $\Psi_t$ that depends on $q^{t-1}$. 
Our main objective is to characterize the optimal loss as a simple computable expression for all $T$ and $K$,
and to find an optimal bookmaking strategy $\bm^{T}$ that achieves the infimum 
in \eqref{eq:optimal_bookmaker_problem_setting}.

By \citet[Theorem~2]{bhatt2025optimal}, the maximization over \(q^T \in \simplex{T}\) 
in \eqref{eq:optimal_bookmaker_problem_setting} can be restricted to the vertices of the simplex, 
i.e., it suffices to maximize over \(q^T \in \stdbasis{K}^T \).
This restriction is related to the notion of decisive gamblers.
If the gambler adheres to this optimal behavior, 
we derive a strategy \(\Psi^T\) that achieves the optimal bookmaking loss simultaneously for all decisive gamblers. 
In real-world settings, however, the gambler may deviate from optimality,
reflecting the aggregate behavior of multiple gamblers. 
This motivates the design of a bookmaker's strategy capable of handling---and even benefiting from---such non-optimal behavior.
In the next section, 
we refine the notion of online bookmaking loss to capture the idea of \emph{opportunistic strategies}.

\subsection{Opportunistic Bookmaking Strategy}\label{subsec:problem_opportunistic} 
The optimal loss $\optimalloss{T}{K}$ is characterized as the solution to a repeated min-max game.
In this game, the first player, a bookmaker, selects a strategy, while the second player, the gambler, responds to it. 
The second player can be viewed as an adversarial environment that
consistently chooses the worst-case scenario for the bookmaker at each step.
However, the gambler may not always play optimally; that is, the environment is not as adversarial as expected. 
We then ask
\begin{align*}
    \text{
        \emph{Can the bookmaker reduce its optimal loss when facing a suboptimal gambler, and if so, how?}
    }
\end{align*}

This question is structured as follows: we provide the bookmaker with a state vector $s\in\mathbb{R}^K$ corresponding to the payouts that are already committed to the possible event outcomes. 
The state vector, for example, may arise from past betting rounds as $s = \sum_{i=1}^{t-1} q_i \oslash r_i$.
The dependence of the state on $t$ is omitted since only the committed payouts matter, and not the number of rounds that led to the state vector.
The bookmaker's objective is then to design a strategy for $H$ future betting rounds.
An opportunistic bookmaking strategy \( \hat{\bm}^{H} \) is a sequence of mappings,
defined by 
\[
    \hat{\bm{}}_h : \mathbb{R}^K\times \simplex{h-1} \times \simplex{h-1} \to \simplex{}, 
    \quad h \in [H],
\]
such that, for each of the remaining rounds $h \in [H]$, 
the odds are given by $r_h = \hat{\bm}_h(s,\, r^{h-1},\, q^{h-1})$.

The optimal opportunistic bookmaking loss (OOBL) given a state $s$ is defined as
\begin{equation}
    \label{eq:def_oppor_loss}
    \optimalloss{H}{K} \left( s \right)
    = \inf_{ \hat{\Psi}^H } \; \max_{q^H\in\simplex{H}} \; \max_{ k \in \nset{K} } \;
    s(k) + \sum_{h=1}^{H} 
        \frac{q_h(k)}
        {r_h(k)}.
\end{equation}
This general formulation allows us to define the loss with respect to an arbitrary state, 
rather than a state resulting from optimal adversarial dynamics.
The expression $s(k) + \sum_{h=1}^{H} \frac{q_h(k)}{r_h(k)}$ quantifies the total payout for outcome $k$,
combining the prior commitments encoded in $s(k)$ and the additional payout incurred over the remaining $H$ rounds.
Although the gambler may have acted suboptimally, 
we retain the worst-case maximization over $q^H$, since the bookmaker cannot predict future behavior. 
This conservative formulation avoids making optimistic assumptions about how the gambler will act going forward.
Clearly, \eqref{eq:def_oppor_loss} specializes to the online bookmaking loss $\optimalloss{T}{K}$ in \eqref{eq:optimal_bookmaker_problem_setting} 
by choosing $s = \zerovec{K}$ and $H=T$. 
Even though the latter is a special case,
we present the results for this setting first, 
as they are cleaner and provide stronger insights.

\section{Main Results} \label{sec:main_results}
This section presents our main results.
We first present \theoremref{theorem:optimal_loss,theorem:regret_factor}, 
which characterize the optimal bookmaking loss and its regret growth rate.
In \sectionref{subsec:results_opportunistic}, 
we establish the optimal opportunistic bookmaking loss, a strategy to achieve it (\theoremref{theorem:optimal_bookmaking_algorithm}), 
and present \algorithmref{algo:optimal_bookmaking}, which provides an opportunistically optimal strategy at all times.

Our first result characterizes the optimal bookmaking loss.
\begin{boxedmaintheorem}[The Optimal Bookmaking Loss]
    \label{theorem:optimal_loss}
    The optimal bookmaking loss
    \(
    \optimalloss{T}{K} 
    \)
    is equal to the largest root of the polynomial 
    \begin{equation}
        \label{eq:optimal_cost_polynomial}
        \ppoly{T}{K}{x} = \sum_{m=0}^{K} \binom{K}{m} \risingfact{(-T)}{K-m} \, x^{m},
    \end{equation}
    where \(\risingfact{\alpha}{\beta}\) denotes the rising factorial (defined above).
\end{boxedmaintheorem}
\theoremref{theorem:optimal_loss} provides an explicit and computable expression for any horizon $T$ and any number of outcomes $K$. 
In particular, to compute the loss for any pair of parameters, all that is required is to find the maximal root of a polynomial of degree $K$. That is, the computational complexity does not scale with $T$. 
We illustrate the benefits of the explicit loss expression via special cases.
\begin{itemize}
    \item 
For $K=2$, we obtain 
\(
    \ppoly{T}{2}{x} = x^2 - 2Tx + T(T-1)
\) 
whose largest root is
\begin{equation} \label{eq:regret_K2}
    \optimalloss{T}{2} = T + \sqrt{T}.
\end{equation}  
This reveals the optimal online bookmaking loss for binary games in \citet{bhatt2025optimal}.
\item 
For $K=3$, we obtain 
\( 
    \ppoly{T}{3}{x} = x^3 - 3Tx^2 + 3T(T - 1)x - T(T-1)(T-2)
\). 
It~is a cubic equation, and by \emph{Cardano's method} the largest root is 
\begin{equation} \label{eq:regret_K3}
    \optimalloss{T}{3} =  T + 2\sqrt{T} \textstyle\cos\left( \frac{1}{3} \arccos \left( T^{-\frac{1}{2}} \right) \right).
\end{equation} 
\end{itemize}

For $K\ge4$, explicit computation of $\optimalloss{T}{K}$ is difficult, but we can identify a structure for the optimal loss based on $K=2,3$:
\begin{align}
    \label{eq:main_results_optimal_loss_decomposition}
     \optimalloss{T}{K} = T + R_{T,K},
\end{align}
where \(R_{T,K}\) denotes the \emph{regret}. 
The notion of regret is not arbitrary: 
the linear \(T\) term corresponds to a lower bound on the optimal loss, 
achieved by a bookmaker that knows the sequence of bets in hindsight but does not know the event's outcome, so its optimal action is $r_t=q_t$. 
Recall that if \( {R_{T,K}}/{T} \) converges to zero, the overround can be arbitrarily close to $1$ (fair regime).
We establish this convergence by characterizing the growth rate of the regret.

Observe that for $K=2,3$, the regret scales as \(O(\sqrt{T})\);
as it holds that \(R_{T,2}~=~\sqrt{T} \), and
\(
    \lim_{T \to \infty} R_{T,3}/\sqrt{T} = \sqrt{3}.
\)
This suggests studying the asymptotic growth of regret by defining the \emph{asymptotic regret factor}
\begin{align}\label{eq:regret_factor}
     \beta_K \eqdef \lim_{T \to \infty} \frac{R_{T,K}}{\sqrt{T}},
\end{align}
provided that the limit exists. Here, \( K \) is treated as a fixed constant, independent of \( T \).  
Analysis shows the limit indeed exists and, in some cases, can even be expressed by radicals, e.g., \(\beta_4= \sqrt{3+\sqrt{6}}\) and \(\beta_5= \sqrt{5+\sqrt{10}}\). 
We provide a precise characterization of $\beta_K$ for all $K$.
\begin{boxedmaintheorem}[The Asymptotic Regret Factor]
    \label{theorem:regret_factor}
    The asymptotic regret factor $\beta_{K}$ in \eqref{eq:regret_factor} is equal to the largest root of 
    the $K$-th Hermite polynomial 
        \begin{equation*}
            \hermp_K(x) = (-1)^K \frac{1}{\varphi(x)} \frac{\mathrm{d}^K}{\mathrm{d} x^K} \varphi(x),
        \end{equation*}
    where \(\varphi(x)\) is the probability density function of the standard normal distribution.
    
    \noindent Consequently, 
    \begin{equation}
        \beta_{K} = 2 \sqrt{K} + o( \sqrt{K} ), \nonumber
    \end{equation}
    and explicit bounds on $\beta_{K}$ are given below in \eqref{eq:bounds_hermite_root}.
\end{boxedmaintheorem}

\theoremref{theorem:regret_factor} 
establishes the $\sqrt{T}$-scaling of the regret for all $K$,  
and provides a simple method to compute the optimal regret factor via Hermite polynomials.
The roots of Hermite polynomials are well-studied: all roots are real, 
and the largest root satisfies $A_K\le \beta_K \le  B_K$ where the bounds are given by
\begin{equation}\label{eq:bounds_hermite_root}
    A_K = 2 \sqrt{K} - 9 \cdot {2}^{-5/3} \cdot {K}^{-1/6}, \; B_K = \sqrt{4K + 2} - \sqrt{2} \cdot (6^{-1/3}i_1) \cdot (2K + 1)^{-1/6},
\end{equation}
and \( i_1 \) denotes the smallest positive root of the \emph{Airy function} (\( 6^{-1/3}i_1 \approx 1.85574 \)) \citep{krasikov2004bounds}.
\begin{remark} \label{remark:opt_loss_grows_to_infinity_for_fixed_T}
    For any fixed \(T \in \Nplus \), 
    \begin{equation} \label{eq:opt_regret_grows_to_infinity_for_fixed_T}
        \lim_{K \to \infty} \frac{R_{T,K}}{\sqrt{T}} = \infty.
    \end{equation}
    Thus, the optimal bookmaking loss grows to infinity as the number of outcomes $K$ grows.
    The proof is given in \appendixref{appendix:proof_remark_opt_loss_grows_to_infinity_for_fixed_T}.
\end{remark}

\newpage
\subsection{The Opportunistically Optimal Bookmaking Strategy} \label{subsec:results_opportunistic}

\begin{wrapfigure}{r}{0.36\textwidth}
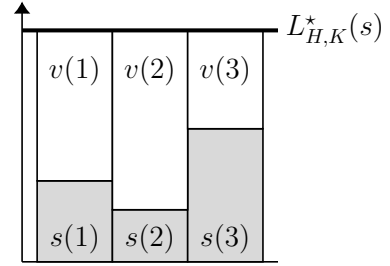

\hspace*{-6.em}
\floatconts
{fig:waterlevel}
{\caption{Illustration of a single step in a game with $K=3$. 
The state $s(k)$ corresponds to the committed payouts to each possible outcome so far, 
and the residual losses $v(k)$ correspond to the anticipated future losses. 
Theorem \ref{theorem:nash} establishes the Nash equilibrium that these vectors sum up to the optimal loss $\optimalloss{H}{K}(s)$.}}
{\includeteximage[width=0.49\textwidth]{residual_loss_vector_figure}}
\end{wrapfigure}

We illustrate in \figureref{fig:waterlevel} the main idea of the optimal algorithm presented in this section. 
The gray bars correspond to the state vector $s$, representing the payouts committed by the bookmaker prior to designing an opportunistic strategy with horizon~$H$. Given a state vector and $H$, the bookmaker computes $\optimalloss{H}{K}(s)$, 
which corresponds to the total loss—comprising the current state and the forthcoming worst-case payouts. 
In \theoremref{theorem:nash}, we establish a Nash equilibrium of the game: for all state vectors, 
if the gambler bets decisively in each of the remaining betting rounds (a strategy we refer to as a \emph{decisive gambler}), 
the bookmaker is forced to return $\optimalloss{H}{K}(s)$—that is, the maximum among the gray bars will reach the water level. We emphasize that a decisive gambler may choose a different outcome in each betting round. Conversely, regardless of the initial state vector, the optimal bookmaker is able to guarantee the same loss $\optimalloss{H}{K}(s)$ for any decisive gambler's betting sequence.
 This equilibrium resembles a \emph{water-filling} solution: the bookmaker balances the losses across outcomes such that the maximum loss reaches the water level $\optimalloss{H}{K}(s)$.
The idea behind the opportunistic strategy is now apparent: when a decisive bet is placed by the gambler, the corresponding gray bar increases, while all other coordinates of the state remain unchanged. However, if the gambler distributes its bet among multiple outcomes, the gambler's strategy is suboptimal, and the bookmaker exploits this opportunity by lowering the water level accordingly. The difference between the water level and the current state vector is called the \emph{residual loss} vector.

The optimal opportunistic bookmaking loss will be defined using the polynomial
\begin{align} \label{eq:biased_denominator_function}
    \begin{aligned}
    \biaseddenom{H}{s}{x}
    = 
    \sum_{m=0}^{K} c_{H,m}(s) \, x^{K - m},     \ \
    \text{where} \ \   c_{H,m}(s) 
    =  
    (-1)^{m} \sum_{n=0}^{m} \fallingfact{ H }{m-n} \, \binom{K - n}{m - n} \, \esp{n}{s},
    \end{aligned}
\end{align}
where the dependence on $K$ is omitted from the notation  $\biaseddenomfunction{H}{s}$ for brevity.
The optimal strategy will be calculated from the residual loss vector using the polynomial
\begin{equation}
\label{eq:main_results_def_of_denom_poly}
    \denom{H}{K}{v} \eqdef 
        \sum_{m=0}^{K} \risingfact{ (-H) }{K-m} \, \esp{m}{v}.
\end{equation}

\algorithmref{algo:optimal_bookmaking} provides a simple procedure to sequentially compute the odds based on the bets placed so far. 
The main idea is to compute the odds based on the current residual loss vector. 
In each round, if the gambler is not decisive, the loss can be decreased to the OOBL, 
and if the gambler is optimal, the bookmaker sticks to the strategy computed using the previous loss. 
We formalize the performance of \algorithmref{algo:optimal_bookmaking} in \theoremref{theorem:optimal_bookmaking_algorithm}.

\begin{boxedmaintheorem}[Optimal Opportunistic Bookmaking: Loss and Algorithm]
    \label{theorem:optimal_bookmaking_algorithm}
    For any number of outcomes $K$, horizon $H$, and state $s \in \Rdim{K}$,
    the optimal opportunistic bookmaking loss $\optimalloss{H}{K}(s)$ is equal to the largest real root of the polynomial 
    \(\biaseddenomfunction{H}{s}\) in \eqref{eq:biased_denominator_function}.
    The optimal odds for the $k$-th outcome are 
    \begin{equation}
        \label{eq:optimal_odd_thm}
        r(k) = \frac{\denom{H-1}{K-1}{v\minind{k}}}{\denom{H-1}{K}{v}},
    \end{equation}
    where $v = \univec{\optimalloss{H}{K} \left( s \right)}{K} - s$ is the residual loss vector.         
    Consequently, \algorithmref{algo:optimal_bookmaking} is opportunistically optimal at any time.
\end{boxedmaintheorem}
We note that \eqref{eq:optimal_odd_thm} constitutes a bookmaking strategy: 
since $r(k)$ depends only on the residual loss vector $v$ and the horizon $H$, 
recomputing $v$ after each bet and reusing this formula yields the bookmaker's action at every step. This is the main procedure in \algorithmref{algo:optimal_bookmaking}.
\begin{remark} \label{remark:all_polys_are_denom}
The polynomials $\ppoly{T}{K}{x}$ and $\biaseddenom{H}{s}{x}$ that are used to compute the optimal bookmaking loss and the optimal opportunistic bookmaking loss are related to $\denom{H}{K}{v}$ via the relation 
\begin{equation}
    \biaseddenom{H}{s}{x} \eqdef \denom{H}{K}{\univec{x}{K} - s}.
    \label{eq:remark_def_of_biaseddenom}
\end{equation}
Consequently, we also have $\ppoly{T}{K}{x} 
    \eqdef \denomfunction{T}{K}(\univec{x}{K})$. The particular vector argument of $\denom{H}{K}{\cdot }$ in \eqref{eq:remark_def_of_biaseddenom} corresponds to a residual vector when the water level is $x$. That is, the OOBL is the minimal water level such that the difference to the state constitutes a valid residual vector. 
    The proof of the identity in \eqref{eq:remark_def_of_biaseddenom} appears in \appendixref{appendix:proof_of_all_polys_are_denom}.
\end{remark}

\begin{remark}[Computational Complexity]
\label{remark:algo_complexity}
\algorithmref{algo:optimal_bookmaking} is implemented in a forward manner, 
so that only $T$ sequential operations are required.
This is noteworthy since it is uncommon for an exactly optimal algorithm to be both forward-running and to exhibit a per-round computational complexity that is independent of the horizon $T$. 
The overall complexity of the algorithm is \(O(TK^2)\) with a per-round complexity of \(O(K^2)\).
The complexity analysis can be found in \appendixref{appendix:complexity_analysis}.
\end{remark}

\begin{algorithm2e}
\caption{Opportunistically Optimal Bookmaking Algorithm}
\label{algo:optimal_bookmaking}
\SetKw{Initialization}{Initialization:}
\SetKw{output}{output}
\SetNoFillComment
\DontPrintSemicolon
\LinesNumbered
\KwIn{$K$, $T$, $\seq{q}{T-1}$ (bets obtained sequentially)}
\KwOut{$r_1,\dots,r_T$ (outputs $\seqelem{r}{t}$ after observing $\seq{q}{t-1}$)}
\Initialization{$s \gets \zerovec{K},\; L \gets \optimalloss{T}{K}$ }

\output $\seqelem{r}{1}\gets \univec{\frac{1}{K}}{K}$ \label{line:output_first_odds}

\For{$t = 2:T$}{
    
    $s \gets s + \seqelem{q}{t-1} \oslash \seqelem{r}{t-1}$  \tcp*{update the state vector} 
    \label{line:update_the_state_vector}

    \If(\hfill \tcp*[h]{check if the gambler is not optimal (decisive)})
    {$\seqelem{q}{t-1} \not\in \stdbasis{K}$ \label{line:check_if_decisive}}
    {
        $L \gets \argmaxroot{\biaseddenomfunction{T-t+1}{s}}$ \tcp*{reduce loss to the optimal opportunistic loss} 
        \label{line:update_the_loss}

    }

    $v \gets \left( \univec{L}{K} - s \right)$
    \label{line:opt_res_loss_vector_update}
    \tcp*{calculate the residual (future) loss vector} 

    \For(\label{line:iterations_for_optimal_odds_update}){$k=1:K$}{
        $r(k) \gets \denom{T-t}{K-1}{v\minind{k}}$
        \label{line:optimal_odds_update}
        \tcp*{$r(k) \propto$ the optimal betting odds in \eqref{eq:optimal_odd_thm}} 
    }\label{line:end_iterations_for_optimal_odds_update}
    
    \output $\seqelem{r}{t} \gets r/\|r\|_1$  \label{line:output_odds}
}
\end{algorithm2e}

\begin{remark}
\label{remark:prev_algo_comparison}
\citet{bhatt2025optimal} consider the binary case (\(K = 2\)) and design two algorithms for the bookmaker, 
depending on the gambler's behavior. 
Against a decisive gambler, they propose the \emph{Optimal Strategy For Decisive Gamblers} algorithm (\algfont{ODG}), 
which coincides with \algorithmref{algo:optimal_bookmaking}. 
Against a non-decisive gambler, they construct a mixture of \algfont{ODG} strategies, 
based on the observed betting sequence. This approach is shown to be suboptimal in \appendixref{appendix:prev_algo_comparison}. 
They also propose a Monte Carlo approximation of the mixture, which incurs a computational cost of \( \omega(T^3) \).
In contrast, \algorithmref{algo:optimal_bookmaking} applies to both gambler types and achieves optimality with significantly lower complexity.
\end{remark}

\begin{remark}[The Effect of \teq{$\varepsilon$}-Approximate Root-Finding]
    \label{remark:approximate_root_finding}
    \algorithmref{algo:optimal_bookmaking} computes the optimal opportunistic bookmaking loss 
    by evaluating the largest real root of a polynomial (Line~\ref{line:update_the_loss}).
    In \appendixref{appendix:oracle_optimal_bookmaking}, we present a modified version of the algorithm 
    that approximates this root with precision \(\varepsilon\).
    In this case, for every non-decisive gambler action (Line~\ref{line:check_if_decisive}),  
    an additional loss of at most \(2\varepsilon\) is incurred.
    Therefore, the additional loss accumulated over the course of the game is at most \(2T\varepsilon\), 
    which is \(o(1)\) when \(\varepsilon = O(T^{-1})\).
    Utilizing the standard Newton–Raphson method---
    which converges in \(O(K \log{\varepsilon^{-1}})\) iterations---for root-finding, 
    the resulting algorithm incurs only an additional constant loss, 
    with overall \(O(TK^2 + KT \log T)\) complexity.
\end{remark}

\section{Proofs and Key Ideas} \label{sec:proof_outline}
This section presents the proofs of the main results.
In \sectionref{sec:nash}, we establish a Nash Equilibrium of the game (\theoremref{theorem:nash}).
\sectionref{sec:bpf} introduces and characterizes the \emph{Bellman-Pareto}.
The proofs of the main results in \sectionref{sec:main_results} are provided in \sectionref{sec:proofs_of_main_thms}.

\subsection{Nash Equilibrium of the Game}
\label{sec:nash}
We present the optimal opportunistic bookmaking loss (OOBL), defined in \eqref{eq:def_oppor_loss},
as the result of a sequential optimization process, 
where at each step, we optimize a function that depends on the previous function in a backward manner. This formulation allows optimization only over the instantaneous bookmaker and gambler actions. 
\begin{definition}[Value Function] 
    \label{def:value_function}
    The value function quantifies the OOBL given a specified horizon \(H\) and state vector \(s\). 
    It is defined recursively for $H > 0$
    \begin{equation}
        \label{eq:value_function}
        \pot{ H }{K}{s} \; \eqdef \;
            \inf_{ r \in \simplex{} } \; \max_{ q \in \simplex{}} \;
                \pot{ H - 1}{K}{
                     s + q \oslash r 
                 } 
    \end{equation}
    with the initial condition $\pot{0}{K}{s}  \eqdef  \max_{ k \in \nset{K} } \vecind{s}{k}$.
\end{definition}
Note that the value function is akin to that in Markov Decision Processes using the dynamic programming principle. 
Indeed, it corresponds to the Bellman operator, 
but the maximization over $k\in[K]$ at the last round implies that the immediate reward is zero at all times, 
and the overall reward is measured at the final time. 
It is straightforward to see that 
$\pot{ H }{K}{s} = \optimalloss{H}{K} \left( s \right)$. 
In the special case $s=\zerovec{K}$, we obtain 
$\pot{ H }{K}{\zerovec{K}} = \optimalloss{H}{K}$.
Before presenting a Nash equilibrium, 
we define the notion of a \emph{decisive gambler}.
\begin{definition}[Decisive Gambler]\label{def:decisive_action}
    A gambler is decisive if it bets on a single outcome, i.e., $q \in \stdbasis{K}$ in \eqref{eq:value_function}.
\end{definition}
We now present a result that characterizes a Nash equilibrium 
and establishes the existence and uniqueness of an optimal opportunistic bookmaking strategy. 
\begin{theorem}[Nash Equilibrium of the Game]
    \label{theorem:nash}  
    For any horizon \( H \geq 1 \) and state \( s \), the optimization of the value function in \eqref{eq:value_function} satisfies the following properties: 
    (i) for any \( r \in \simplex{}\), there exists an optimal gambler who is decisive, 
    meaning it is sufficient to maximize over \( q \in \stdbasis{K} \),
    and (ii) the optimal $\optr \in \simplex{}$ is unique and satisfies
    \begin{equation}
        \label{eq:main_results_thm_the_game_dynamics}
        \pot{ H }{K}{s} = \pot{ H-1 }{K}{s + q \oslash \optr } \quad \forall q \in \stdbasis{K}.
    \end{equation}
\end{theorem}

\theoremref{theorem:nash} characterizes optimal strategies for both the bookmaker and the gambler: 
a decisive gambler maximizes the bookmaker's return,
and the optimal bookmaker equalizes the loss associated with each decisive gambler's
bet $q \in \stdbasis{K}$. 
This balancing mechanism resembles the \emph{water-filling algorithm} \citep[Ex.~5.2]{Boyd2004}, 
but only the loss of a single outcome will reach the water level as we illustrate in \figureref{fig:waterlevel}.
We remark that the optimality of decisive gamblers was established in \cite{bhatt2025optimal} for \(s = \zerovec{K}\),
while \theoremref{theorem:nash} extends it to general \(s\) and proves the existence of a balancing \(\optr\). 
The proof of \theoremref{theorem:nash} is provided in \appendixref{appendix:the_adversarial_game_dynamics}.

Although the value function enables us to characterize the players' optimal strategies, 
its recursive structure makes it computationally intractable for large $H$. 
A further challenge lies in the state space, 
as \eqref{eq:value_function} must be solved for any state $s \in \Rdim{K}$.
Leveraging the implication of \theoremref{theorem:nash}, 
we restrict our attention to decisive gamblers without loss of optimality.

\subsection{Bellman-Pareto Frontier}
\label{sec:bpf}

\theoremref{theorem:nash} implies that the optimal opportunistic bookmaking strategy $\hat{\Psi}^H$
and any decisive gambler's betting sequence $q^H \in \left(\stdbasis{K} \right)^H$ satisfy
\begin{equation} \label{eq:conclusion_from_nash_and_value_function}
    \pot{H}{K}{s} \;=\; 
        \max_{k \in [K]} \left( s(k) + \sum_{h=1}^{H} \frac{q_h(k)}{r_h(k)} \right).
\end{equation}
The next lemma is a simple observation that the remaining optimization over $k\in[K]$ in \eqref{eq:conclusion_from_nash_and_value_function} is determined by the gambler's last bet $q_H$.
\begin{lemma} \label{lemma:last_gamblers_bet}
The index $k \in [K]$ that attains the maximum in \eqref{eq:conclusion_from_nash_and_value_function}
is determined by the gambler's bet in the final round as $q_H = \basis{k}$. 
\end{lemma}
The proof of \lemmaref{lemma:last_gamblers_bet} is provided in \appendixref{appendix:technical_C}.
Combining \eqref{eq:conclusion_from_nash_and_value_function} and Lemma \ref{lemma:last_gamblers_bet} allows us to omit the maximization over $k$, 
thereby obtaining the simpler problem of finding the unique bookmaking strategy that equates  
\begin{equation} \label{eq:oron}
  \pot{H}{K}{s} =  \left( 
        s(k) + \sum_{h=1}^{H} \frac{q_h(k)}{r_h(k)} \right), 
\end{equation}
for all $ q^{H}\in (\stdbasis{K})^{H}$, where $k$ is determined by the last bet as $q_H = \basis{k}$. 
Recall that the sum in \eqref{eq:oron} is the \emph{residual loss} vector (RLV), 
illustrated in \figureref{fig:waterlevel}, as the difference between the current state and the water level. 
Our solution approach is not to solve \eqref{eq:oron} directly, but to look at a more general problem. 
Our proposed idea is to find the set of all possible residual vectors (possible sums in \eqref{eq:oron}). 
That is, we ignore the state vector and the optimal strategy looked for, 
and aim to characterize the set of all possible residual vectors when the gambler is decisive. 
If we are able to characterize the set of all residual vectors, 
then it is straightforward to find the residual vector that solves \eqref{eq:oron} since the sum of the state and the residual vectors should be a constant vector. 
We proceed to formalize the set of possible residual vectors. 
\begin{definition}[Bellman-Pareto Frontier] \label{def:bpf}
    A strategy \(\bm^{H}\) achieves a residual vector \( v \in \Rdim{K} \)
    if for any sequence $\seq{q}{H} \in \left(\stdbasis{K}\right)^H$ it holds that
    \begin{equation}  \label{eq:def_bpf}
        \sum_{h=1}^{H}  
        \frac{\vecind{\seqelem{q}{h}}{k}}
            {\vecind{\seqelem{r}{h}}{k}} = \vecind{v}{k},
    \end{equation}
    where \( k \in [K] \) corresponds to the gambler's bet at the final betting round \( q_H = \basis{k} \).
    A vector $v$ is \emph{$H$-achievable} if there exists a strategy $\bm^{H}$ that achieves it. 
    The set of all $H$-achievable vectors, denoted by \( \vvset{H, K} \), is the \emph{Bellman-Pareto (BP) frontier} of the online bookmaking game.
\end{definition}
The BP frontier captures the fundamental trade-offs in the bookmaker's decision-making. 
Each residual vector in the frontier represents a possible configuration 
of future payouts that cannot be improved for one outcome without worsening another.

The following theorem characterizes the Bellman-Pareto frontier.
\begin{theorem}[The Bellman-Pareto Frontier] \label{theorem:characterization_of_bpf}
    The Bellman-Pareto frontier of the online bookmaking game is given by 
     \begin{equation} \label{eq:thm_characterization_of_bpf}
        \vvset{H,K} = \left\{ \univec{\optimalloss{H}{K}(s)}{K} - s \; \middle| \; 
        s \in \Rdim{K},\,
        \min_{k \in [K]} s(k) = 0 \right\},
     \end{equation}
     where
     \(
         \optimalloss{H}{K}(s) 
     \)
     is equal to the largest real root of 
    \(\biaseddenomfunction{H}{s}\) in \eqref{eq:biased_denominator_function}.
\end{theorem}
\theoremref{theorem:characterization_of_bpf} identifies the BP frontier through its connection to the OOBL.
The fact that $\optimalloss{H}{K}(s)$ is efficiently computable for any state $s$ allows the bookmaker to determine all achievable RLVs, and thus, all valid future payouts.
To prove \theoremref{theorem:characterization_of_bpf}, we first establish several intermediate results.

\begin{lemma}[Recurrence Relation for \( \denomfunction{ H }{K} \)] 
    \label{lemma:recurrence_relation_or_the_denom}
    For every \(H \geq 1 \) and \(k \in \nset{K}\),
    \[
        \denom{ H }{K}{v} = \vecind{v}{k} \cdot \denomfunction{ H }{K-1}( v\minind{k} ) - H \cdot \denomfunction{ H-1 }{K-1}(v\minind{k} ).
    \]
\end{lemma}
\lemmaref{lemma:recurrence_relation_or_the_denom}, proved in \appendixref{appendix:technical_C},
is established primarily by leveraging the recurrence relation for ESPs (\lemmaref{lemma:elementary_symmetric_polynomial_recurrence_relation}) and the identity $H \cdot \fallingfact{(H-1)}{m} = \fallingfact{H}{m+1}\,$ for all $H, m \in \Nfield$.

\begin{lemma}[Necessary Constraints for $H$-Achievable Vectors] \label{lemma:main_lemma_in_conditions}
If \( v \in \vvset{H, K} \) then
\begin{enumerate}[label={{\arabic*.}}, ref={\ref{lemma:main_lemma_in_conditions}.\arabic*}]
    \item \label{lemma:proof_of_bp_1}
    The first action of a bookmaker that achieves $v$ is given by \eqref{eq:optimal_odd_thm}.
    \item  \label{lemma:proof_of_bp_2}
   For every \( k \in \nset{K} \), the value $v(k)$ is determined by the remaining $K-1$ entries  \( v\minind{k} \) via
   \begin{equation} \label{eq:necessary_conditions_proof_value_of_v_k}
        \vecind{v}{k} 
        = \frac{
            H \cdot \denom{ H-1 }{K-1}{v\minind{k}}
        }{
            \denom{ H }{K-1}{v\minind{k}}
        }.
   \end{equation}
   Consequently, by \lemmaref{lemma:recurrence_relation_or_the_denom}, $\denom{H}{K}{v} = 0$.
    \item \label{lemma:proof_of_bp_3}
    For all \( u \in \Rdim{K}, \ \ u \succ v \implies \denom{H}{K}{u} > 0 \).
\end{enumerate}
\end{lemma}
\lemmaref{lemma:main_lemma_in_conditions} establishes structural constraints on any vector $v$ that lies on the Bellman-Pareto frontier.  
It highlights the built-in trade-off among the coordinates: reducing one component necessarily increases another.
The proof of \lemmaref{lemma:main_lemma_in_conditions} is deferred to \appendixref{appendix:value_vectors}.

\begin{proof}[Proof of \theoremref{theorem:characterization_of_bpf}]
We first prove that $u$ is an $H$-achievable vector if and only if there exists a state $s$ such that $u$ is the RLV for $s$.
\begin{itemize}
    \item[$\Leftarrow$]
    Let $u$ be the RLV for $s$, that is, $s+u = \univec{\optimalloss{H}{K}(s)}{K}$.
    By Theorem \ref{theorem:nash}, there exists a bookmaking strategy $\hat{\Psi}^H$ whose loss is $\optimalloss{H}{K}(s)$;
    combined with \lemmaref{lemma:last_gamblers_bet}, this implies that for any betting sequence $q^H=(q_1,\dots,q_{H-1},\basis{k}) \in (\stdbasis{K})^H$, 
    $$
        \optimalloss{H}{K}(s) - s(k) = \sum_{h=1}^H \frac{q_h(k)}{r_h(k)} = u(k).
    $$
    \item[$\Rightarrow$]
    Suppose that $u$ is an $H$-achievable vector.
    We prove that $u$ is the RLV for the state $\hat{s} = -u$.
    Let $x=\pot{H}{K}{\hat{s}}$ denote the optimal loss given the state $\hat{s}$.
    By \theoremref{theorem:nash}, an optimal bookmaking strategy exists whose loss equals $x$. 
    As shown in \eqref{eq:oron} and the subsequent discussion, 
    this strategy achieves an RLV of the form  
    $  {v} = \univec{ x }{K} - \hat{s}$ for some $x$. We proceed to show that $x=0$, and thus $v=u$ is an RLV. 
    The vector $u$ satisfies 
    \(
        u = \univec{0}{K} -\hat{s},
    \)
    i.e., it follows the structure of ${v}$ with $x=0$. 
    Since $u$ is $H$-achievable, by \lemmaref{lemma:proof_of_bp_2}, $\denom{H}{K}{u} = 0$.
    Also, by \lemmaref{lemma:proof_of_bp_3},
    for any $x < 0$, the vector $\univec{x}{K}- \hat{s} \prec u$ is not $H$-achievable.
    As a result, the RLV for $\hat{s}$ is $u$. 
\end{itemize}
We conclude that
\begin{align*}
    \vvset{H,K} 
    \; &= \;
    \left\{ \univec{\pot{H}{K}{s}}{K} - s  \; \middle| \; s \in \Rdim{K}  \right\}
    \\
    \; &= \; 
    \left\{ \univec{\pot{H}{K}{s - \univec{\min_{k\in[K]}s(k)}{K}}}{K} - \left( s - \univec{\min_{k\in[K]}s(k)}{K} \right) 
    \; \middle| \; s \in \Rdim{K}  \right\},
\end{align*}
where the second equality follows from the uniform translation property of the value function, 
stated in \lemmaref{lemma:uniform_translation_property}.
Substituting $\pot{H}{K}{s} = \optimalloss{H}{K}(s)$ yields the characterization in \eqref{eq:thm_characterization_of_bpf}.

We next show that $\optimalloss{H}{K}(s)$ equals the largest real root of 
$\biaseddenom{H}{s}{x} \coloneqq \denom{H}{K}{\univec{x}{K} - s}$,
given explicitly in \eqref{eq:biased_denominator_function}, following \remarkref{remark:all_polys_are_denom}.
Let $v = \univec{\optimalloss{H}{K}(s)}{K}-s$ be the RLV for a state $s \in \Rdim{K}$.
Given that $v\in \vvset{H,K}$, \lemmaref{lemma:proof_of_bp_2} implies
$\denomfunction{ H }{K}(\univec{ \optimalloss{H}{K}(s) }{K} - s) = 0$.
Let \( \rho = \left\{ \rho_1 , \ldots , \rho_Z \right\} \) denote the real roots of \( \biaseddenomfunction{H}{s} \), ordered increasingly. 
Since \( \optimalloss{H}{K}(s) \in \rho \), this set 
contains a unique \( z \in [Z] \) such that \( \optimalloss{H}{K}(s) = \rho_z \). 
If $z < Z$, it holds that $\univec{\rho_z}{K} - s \prec \univec{\rho_Z}{K} - s$, 
which, by \lemmaref{lemma:proof_of_bp_3}, 
implies that $\univec{\rho_z}{K} - s$ is not $H$-achievable. 
Thus, $\optimalloss{H}{K}(s)$ must equal $\rho_Z$.
\end{proof}

\subsection{Proofs of Main Results} \label{sec:proofs_of_main_thms}
This section includes the proofs for \theoremref{theorem:optimal_loss,theorem:optimal_bookmaking_algorithm}. 
The proof of \theoremref{theorem:regret_factor} is proved in \appendixref{appendix:proof_of_theorem_B}.

\begin{proof}[Proof of \theoremref{theorem:optimal_bookmaking_algorithm}]
\theoremref{theorem:characterization_of_bpf} shows that $\optimalloss{H}{K}(s)$ equals the largest real root of \(\biaseddenomfunction{H}{s}\),
for any $H,K \in \Nplus$ and $s \in \Rdim{K}$.
Once $\optimalloss{H}{K}(s)$ is computed, 
the RLV $v = \univec{\optimalloss{H}{K}(s)}{K} - s$ determines the optimal bookmaker's first action $r$ via \eqref{eq:optimal_odd_thm},
as established in \lemmaref{lemma:proof_of_bp_1}.

The correctness of \algorithmref{algo:optimal_bookmaking} for a game with $T$ rounds then follows.
Before the first betting round, 
the state is $s = \zerovec{K}$, 
and the loss is initialized at $L = \optimalloss{T}{K}$.
Since $s$ is a constant vector, so is the RLV $v$;
by the symmetry in \eqref{eq:optimal_odd_thm}, $r_1 = \univec{\frac{1}{K}}{K}$ is the optimal action.
Let $t \in [2, T]$ be a round in the game, with state $s$ and horizon $T - t + 1$.
In Line~\ref{line:update_the_state_vector}, 
the state is updated based on the gambler's previous action $q_{t-1}$.
In Lines~\ref{line:iterations_for_optimal_odds_update}-\ref{line:output_odds}, 
the odds \( r_t \) are computed from the vector \( v \) according to \eqref{eq:optimal_odd_thm}.
It remains to show that Line~\ref{line:opt_res_loss_vector_update} correctly computes the RLV $v$. 
Suppose that upon entering round~$t$, it holds that $L$ equals the optimal loss prior to observing $q_{t-1}$,
and that the previous odds $r_{t-1}$ were computed optimally. 
This holds for $t = 2$ by initialization and is preserved inductively.  
If $q_{t-1} \in \stdbasis{K}$, i.e., the gambler is decisive,  
then by \theoremref{theorem:nash} the optimal loss cannot decrease. 
As $r_{t-1}$ is chosen to achieve the RLV from the previous round,
the updated $v$ remains consistent with the definition of the RLV.
If $q_{t-1} \not\in \stdbasis{K}$, then $L$ is recomputed using the updated state $s$ and horizon $T-t+1$
(Line~\ref{line:optimal_odds_update}), and Line~\ref{line:opt_res_loss_vector_update} produces the RLV.
\end{proof}

\begin{proof}[Proof of \theoremref{theorem:optimal_loss}]
    The optimal bookmaking loss is a special case of the optimal opportunistic bookmaking loss. Thus, we can utilize \theoremref{theorem:optimal_bookmaking_algorithm} to conclude that $\optimalloss{T}{K}$ is the 
    largest root of the polynomial $\biaseddenom{T}{\zerovec{K}}{x}$ as follows
    \[
    \begin{aligned}
        \biaseddenom{T}{\zerovec{K}}{x} \;
        &\stackrel{(a)}{=} \; \denom{ T }{K}{\univec{ x }{K}} 
        \\
        &\stackrel{(b)}{=} \;
        \sum_{m=0}^{K} \binom{K}{m} \risingfact{(-T)}{K-m} \, x^{m},
    \end{aligned}
    \]
    where $(a)$ follows from $\biaseddenom{H}{s}{x} \coloneqq \denom{H}{K}{\univec{x}{K} - s}$ in \eqref{eq:remark_def_of_biaseddenom} and $(b)$ follows from $\denom{H}{K}{v} \eqdef \sum_{m=0}^{K} \risingfact{(-H)}{K-m}  \esp{m}{v}$ in \eqref{eq:main_results_def_of_denom_poly}, 
    together with \definitionref{def:esp}, noting that 
    $\esp{m}{\univec{x}{K}} = \binom{K}{m} x^m$. This concludes the proof that $\biaseddenom{T}{\zerovec{K}}{x} = \ppoly{T}{K}{x}$.
\end{proof}

\section{Conclusions} \label{sec:conclusion}

We considered the online bookmaking problem, 
and derived the optimal bookmaking loss for any number of possible outcomes \( K \) and number of betting rounds \( T \).
Our solution, expressed as the largest root of a polynomial, 
provides a tractable and exact characterization of the optimal loss and strategy,
a form of exact characterization that is uncommon in repeated vector games.
This solution enabled us to show that the regret scales as \( \sqrt{T} \) 
and, for any $K$, its scaling factor converges to the largest root of the $K$-th Hermite polynomial. 
We introduced the concept of opportunistic strategies
and characterized the optimal opportunistic bookmaking loss, 
as well as an efficient algorithm that achieves it. 
This can be viewed as a water-filling solution in which the water level decreases whenever the stronger player (in our case, the gambler) fails to play optimally. 
The key to our analysis was the characterization of the Bellman-Pareto frontier, 
which may have broader applications to vector-valued problems in online learning and game theory.
Future work may consider extending our approach to alternative loss functions or constrained settings, 
building on the structure offered by the Bellman-Pareto framework.
\acks{
    This work was supported by the Israel Science Foundation (ISF), grant No. 1096/23.
}

\bibliography{colt_bib.bib}

\begin{thebibliography}{35}
\providecommand{\natexlab}[1]{#1}
\providecommand{\url}[1]{\texttt{#1}}
\expandafter\ifx\csname urlstyle\endcsname\relax
  \providecommand{\doi}[1]{doi: #1}\else
  \providecommand{\doi}{doi: \begingroup \urlstyle{rm}\Url}\fi

\bibitem[Abernethy et~al.(2011)Abernethy, Bartlett, and Hazan]{abernethy2011blackwell}
Jacob Abernethy, Peter~L Bartlett, and Elad Hazan.
\newblock Blackwell approachability and no-regret learning are equivalent.
\newblock In \emph{Proceedings of the 24th Annual Conference on Learning Theory}, 2011.

\bibitem[Beavis and Dobbs(1990)]{beavis1990optimization}
Brian Beavis and Ian~M Dobbs.
\newblock \emph{Optimization and stability theory for economic analysis}.
\newblock Cambridge university press, 1990.

\bibitem[Ben-Sasson et~al.(2023)Ben-Sasson, Carmon, Kopparty, and Levit]{ben2021elliptic}
Eli Ben-Sasson, Dan Carmon, Swastik Kopparty, and David Levit.
\newblock Elliptic curve fast fourier transform (ecfft) part i: Low-degree extension in time $o(n \log n)$ over all finite fields.
\newblock In \emph{Proceedings of the 2023 ACM-SIAM Symposium on Discrete Algorithms (SODA)}, pages 700--737, 2023.

\bibitem[Bhatt et~al.(2025)Bhatt, Ordentlich, and Sabag]{bhatt2025optimal}
Alankrita Bhatt, Or~Ordentlich, and Oron Sabag.
\newblock Optimal online bookmaking for binary games, 2025.
\newblock arXiv:2501.06923.

\bibitem[Blackwell(1956)]{blackwell1956analog}
David Blackwell.
\newblock An analog of the minimax theorem for vector payoffs.
\newblock \emph{Pacific Journal of Mathematics}, 6\penalty0 (1):\penalty0 1--8, 1956.

\bibitem[Boyd and Vandenberghe(2004)]{Boyd2004}
Stephen Boyd and Lieven Vandenberghe.
\newblock \emph{Convex Optimization}.
\newblock Cambridge University Press, 2004.

\bibitem[Cesa-Bianchi and Lugosi(2006)]{LugosiPrediction}
N.~Cesa-Bianchi and G.~Lugosi.
\newblock \emph{Prediction, Learning, and Games}.
\newblock Cambridge University Press, 2006.

\bibitem[Cesa-Bianchi and Lugosi(2003)]{nicol_cesabianchi_2003}
Nicolo Cesa-Bianchi and G{\'a}bor Lugosi.
\newblock Potential-based algorithms in on-line prediction and game theory.
\newblock \emph{Machine Learning}, 51:\penalty0 239--261, 2003.

\bibitem[Clarke et~al.(2017)Clarke, Kovalchik, and Ingram]{clarke2017adjusting}
Stephen Clarke, Stephanie Kovalchik, and Martin Ingram.
\newblock Adjusting bookmaker’s odds to allow for overround.
\newblock \emph{American Journal of Sports Science}, 5\penalty0 (6):\penalty0 45--49, 2017.

\bibitem[Cortis(2015)]{cortis2015expected}
Dominic Cortis.
\newblock Expected values and variances in bookmaker payouts: A theoretical approach towards setting limits on odds.
\newblock \emph{Journal of Prediction Markets}, 9\penalty0 (1):\penalty0 1--14, 2015.

\bibitem[Cover(1965)]{cover1966behavior}
T.~M. Cover.
\newblock Behavior of sequential predictors of binary sequences.
\newblock In \emph{Proceedings of the Fourth Prague Conference on Information Theory, Statistical Decision Functions, Random Processes}, pages 263--272. Publishing House of the Czechoslovak Academy of Sciences, 1965.

\bibitem[Divos et~al.(2018)Divos, Del Bano~Rollin, Bihari, and Aste]{rollin2018riskneutralpricinghedginginplay}
Peter Divos, Sebastian Del Bano~Rollin, Zsolt Bihari, and Tomaso Aste.
\newblock Risk-neutral pricing and hedging of in-play football bets.
\newblock \emph{Applied Mathematical Finance}, 25\penalty0 (4):\penalty0 315--335, 2018.

\bibitem[Football-Data.co.uk()]{football_data_overround}
Football-Data.co.uk.
\newblock The overround.
\newblock URL \url{https://betting.football-data.co.uk/overround.php}.
\newblock Accessed: 2025-02-06.

\bibitem[Gessel and Stanley(1978)]{GESSEL197824}
Ira Gessel and Richard~P. Stanley.
\newblock Stirling polynomials.
\newblock \emph{Journal of Combinatorial Theory, Series A}, 24\penalty0 (1):\penalty0 24--33, 1978.

\bibitem[Gould et~al.(2015)Gould, Kwong, and Quaintance]{gould2015stirling}
H.~W. Gould, Harris Kwong, and Jocelyn Quaintance.
\newblock On certain sums of stirling numbers with binomial coefficients.
\newblock \emph{Journal of Integer Sequences}, 18\penalty0 (9):\penalty0 15.9.6, 2015.

\bibitem[Graham et~al.(1994)Graham, Knuth, and Patashnik]{graham1994concrete}
R.~L. Graham, D.~E. Knuth, and O.~Patashnik.
\newblock \emph{Concrete Mathematics: A Foundation for Computer Science}.
\newblock Pearson Education, 1994.

\bibitem[Gravin et~al.(2016)Gravin, Peres, and Sivan]{gravin2014optimal}
Nick Gravin, Yuval Peres, and Balasubramanian Sivan.
\newblock Towards optimal algorithms for prediction with expert advice.
\newblock In \emph{Proceedings of the twenty-seventh annual ACM-SIAM symposium on Discrete algorithms}, pages 528--547. SIAM, 2016.

\bibitem[Hannan(1957)]{Hannan1957}
James Hannan.
\newblock Approximation to bayes risk in repeated play.
\newblock \emph{Contributions to the Theory of Games}, 3\penalty0 (2):\penalty0 97--139, 1957.

\bibitem[Harvey and van~der Hoeven(2021)]{harvey2021integer}
David Harvey and Joris van~der Hoeven.
\newblock Integer multiplication in time $o(n \log n)$.
\newblock \emph{Annals of Mathematics}, 193\penalty0 (2):\penalty0 563--617, 2021.

\bibitem[Hazan(2016)]{hazan2019oco}
Elad Hazan.
\newblock Introduction to online convex optimization.
\newblock \emph{Foundations and Trends{\textregistered} in Optimization}, 2\penalty0 (3-4):\penalty0 157--325, 2016.

\bibitem[Kelly(1956)]{Kelly1956}
John~L Kelly.
\newblock A new interpretation of information rate.
\newblock \emph{the bell system technical journal}, 35\penalty0 (4):\penalty0 917--926, 1956.

\bibitem[Krasikov(2004)]{krasikov2004bounds}
Ilia Krasikov.
\newblock New bounds on the hermite polynomials, 2004.
\newblock arXiv:math/0401310.

\bibitem[Kroer and Sandholm(2015)]{Kroer2015DiscretizationOC}
Christian Kroer and Tuomas Sandholm.
\newblock Discretization of continuous action spaces in extensive-form games.
\newblock In \emph{Proceedings of the 2015 international conference on autonomous agents and multiagent systems}, pages 47--56, 2015.

\bibitem[Littlestone and Warmuth(1994)]{littlestone1994weighted}
Nick Littlestone and Manfred~K Warmuth.
\newblock The weighted majority algorithm.
\newblock \emph{Information and computation}, 108\penalty0 (2):\penalty0 212--261, 1994.

\bibitem[Lorig et~al.(2021)Lorig, Zhou, and Zou]{lorig2021optimal}
Matthew Lorig, Zhou Zhou, and Bin Zou.
\newblock Optimal bookmaking.
\newblock \emph{European Journal of Operational Research}, 295\penalty0 (2):\penalty0 560--574, 2021.

\bibitem[Macdonald(1998)]{macdonald1998symmetric}
Ian~Grant Macdonald.
\newblock \emph{Symmetric functions and Hall polynomials}.
\newblock Oxford university press, 1998.

\bibitem[Merhav and Feder(2002)]{merhav2002universal}
Neri Merhav and Meir Feder.
\newblock Universal prediction.
\newblock \emph{IEEE Transactions on Information Theory}, 44\penalty0 (6):\penalty0 2124--2147, 2002.

\bibitem[Patarroyo(2019)]{patarroyo2020digressionhermitepolynomials}
Keith~Y Patarroyo.
\newblock A digression on hermite polynomials, 2019.
\newblock arXiv:1901.01648.

\bibitem[Rotando and Thorp(1992)]{RotandoThorp1992}
Louis~M Rotando and Edward~O Thorp.
\newblock The kelly criterion and the stock market.
\newblock \emph{The American Mathematical Monthly}, 99\penalty0 (10):\penalty0 922--931, 1992.

\bibitem[Shtar'kov(1987)]{Sht87}
Yu.~M. Shtar'kov.
\newblock Universal sequential coding of single messages.
\newblock \emph{Problems Inform. Transmission}, 23\penalty0 (3):\penalty0 175--186, 1987.

\bibitem[Sion(1958)]{sion1958minimax}
Maurice Sion.
\newblock {On general minimax theorems.}
\newblock \emph{Pacific Journal of Mathematics}, 8\penalty0 (1):\penalty0 171 -- 176, 1958.

\bibitem[v.~Neumann(1928)]{vonNeumann1928}
J~v.~Neumann.
\newblock Zur theorie der gesellschaftsspiele.
\newblock \emph{Mathematische annalen}, 100\penalty0 (1):\penalty0 295--320, 1928.

\bibitem[Vovk(1990)]{vovk1990aggregating}
Volodimir~G Vovk.
\newblock Aggregating strategies.
\newblock In \emph{Proceedings of the third annual workshop on Computational learning theory}, pages 371--386, 1990.

\bibitem[Zhu et~al.(2024)Zhu, Soen, Cheung, and Xie]{zhu2024}
Haiqing Zhu, Alexander Soen, Yun~Kuen Cheung, and Lexing Xie.
\newblock Online learning in betting markets: Profit versus prediction.
\newblock In \emph{Forty-first International Conference on Machine Learning}, 2024.

\bibitem[{Zion Market Research}(2024)]{Zion2024SportsBetting}
{Zion Market Research}.
\newblock Sports betting market size, share report, analysis, trends, growth 2032.
\newblock Technical report, Zion Market Research, 2024.

\end{thebibliography}

\appendix
\numberwithin{equation}{section}

\paragraph{Appendix Structure}  
\appendixref{appendix:preliminaries} presents preliminaries for the analysis.  
\appendixref{appendix:the_adversarial_game_dynamics} establishes the proof of \theoremref{theorem:nash}.  
\appendixref{appendix:value_vectors} provides the proof of \lemmaref{lemma:main_lemma_in_conditions}.  
\appendixref{appendix:main_results} contains the omitted proofs from \sectionref{sec:main_results}.  
\appendixref{appendix:technical_proofs} provides technical proofs that support the overall analysis.

\section{Preliminaries}
\label{appendix:preliminaries}

This appendix presents preliminary results used in the main analysis.
\appendixref{appendix:additional_notation} defines notation used in the appendix,
\appendixref{appendix:correspondences} reviews formal definitions and properties of correspondences.
\appendixref{appendix:falling_factorials} provides key identities involving falling and rising factorials 
and Stirling numbers.
\appendixref{appendix:elementary_symmetric_polynomials} provides several properties of elementary symmetric polynomials.

\subsection{Additional Notation} \label{appendix:additional_notation}

We denote $\Rpplus{} \eqdef \{x \in \Rfield \mid x > 0 \}$.
For vectors $x, y \in \Rdim{K}$, the notation $x \odot y$ represents element-wise multiplication.
For a set \(\mathfrak{I} \subseteq \nset{K}\), we write \(x\minind{\mathfrak{I}}\) for \(x\) excluding indices in \(\mathfrak{I}\). 
The interior of the simplex in \(\Rdim{K}\) is denoted by \(\intsimplex{K-1}\).

\subsection{Correspondences}
\label{appendix:correspondences}

\begin{definition}[Correspondence]
Let $\Theta$ and $X$ be topological spaces. 
A \emph{correspondence} from $\Theta$ to $X$, 
denoted $\corr: \Theta \rightrightarrows X$, 
is a mapping that assigns to each $\theta \in \Theta$ a subset of $X$, that is,
\(
   \corr(\theta) \subseteq X,  \; \text{for every} \, \theta \in \Theta.
\)
\end{definition}

\begin{definition}[Compact-valued Correspondence]
    Let $X$ and $\Theta$ be topological spaces. 
    A correspondence $\corr: \Theta \rightrightarrows X$ is \emph{compact-valued} 
    if for every $\theta \in \Theta$, the set $\corr(\theta)$ is compact in $X$.
\end{definition}

\begin{definition}[Continuity of a Correspondence] \label{def:continuity_of_a_correspondence}
Let $X$ and $\Theta$ be topological spaces, and let $\corr: \Theta \rightrightarrows X$ be a correspondence.
\begin{itemize}
    \item $\corr$ is \textbf{\emph{upper hemicontinuous}} at $\theta \in \Theta$ if for every open set $V$ with $\corr(\theta) \subset V$, 
    there exists a neighborhood $U$ of $\theta$ such that  
    \[
    \corr(\theta') \subset V, \quad \forall \theta' \in U.
    \]
    \item $\corr$ is \textbf{\emph{lower hemicontinuous}} at $\theta \in \Theta$ if for every open set $V$ with $V \cap \corr(\theta) \neq \emptyset$, there exists a neighborhood $U$ of $\theta$ such that  
    \[
    \corr(\theta') \cap V \neq \emptyset, \quad \forall \theta' \in U.
    \]
    \item $\corr$ is \textbf{\emph{continuous}} if it is both upper and lower hemicontinuous at every $\theta \in \Theta$.
\end{itemize}
\end{definition}

\begin{lemma}[Berge's Maximum Theorem] \label{lemma:berges_maximum_theorem}
Let $X$ and $\Theta$ be topological spaces, and let $f: X \times \Theta \to \mathbb{R}$ be a continuous function. 
Suppose that $\corr: \Theta \rightrightarrows X$ is a compact-valued correspondence and that $\corr(\theta) \neq \emptyset$ for all $\theta \in \Theta$. Define the function
\[
    f^\star(\theta) = \sup_{x \in \corr(\theta)} f(x, \theta)
\]
and the solution correspondence
\[
    \corr^\star(\theta) = \arg\max_{x \in \corr(\theta)} f(x, \theta).
\]
If $\corr$ is continuous, then the function $f^\star: \Theta \to \mathbb{R}$ is continuous and
the solution correspondence $\corr^\star: \Theta \rightrightarrows X$ is compact-valued and upper hemicontinuous.
\end{lemma}
For more information see \cite{beavis1990optimization}.

\subsection{Falling and Rising Factorials and Stirling Numbers}
\label{appendix:falling_factorials}

For clarity, we restate the factorial operators given in \sectionref{sec:preliminaries}.
\begin{definition}[Falling and Rising Factorials] \label{def:falling_factorial}
For $x \in \Rfield$ and $m \in \Nfield$ ,
the \emph{falling factorial} is defined as
\begin{equation*} 
    \fallingfact{x}{m}
    \eqdef \prod_{i=0}^{m-1} (x - i)
    = x (x - 1) \cdots (x - m + 1),
\end{equation*}
and the \emph{rising factorial} is defined as
\begin{equation*}
    \risingfact{x}{m}
    \eqdef \prod_{i=0}^{m-1} (x + i)
    = x (x + 1) \cdots (x + m - 1).
\end{equation*}
In both cases, the empty product is equal to one so that 
$\fallingfact{x}{0}=\risingfact{x}{0}=1$.
\end{definition}

\begin{lemma}
    \label{lemma:falling_factorial_identities}
    For all $T, m \in \Nfield$,
    \begin{enumerate}[label={\arabic*.}, ref=\ref{lemma:falling_factorial_identities}.\arabic*]
        \item
        \label{lemma:falling_factorial_identities_1}
        $\fallingfact{T}{m+1} = T \cdot \fallingfact{(T-1)}{m}$
        
        \item
        \label{lemma:falling_factorial_identities_2}
        $\fallingfact{T}{m} + m \cdot \fallingfact{T}{m-1} = \fallingfact{(T+1)}{m}$

        \item
        \label{lemma:falling_factorial_identities_3}
        \(\fallingfact{T}{m} = (-1)^m \risingfact{(-T)}{m}\)
    \end{enumerate}
\end{lemma}

\begin{definition}[The Stirling Numbers of the First Kind]
    \label{def:stirling_numbers_of_the_first_kind}
    The \emph{Stirling numbers of the first kind} $ \unstirling{n}{k}$
    count the number of permutations of \( n \) elements consisting of exactly \( k \) disjoint cycles, where $1 \leq k \leq n$.
    They satisfy the recurrence relation:
    \[
        \unstirling{n+1}{k} = n \unstirling{n}{k} + \unstirling{n}{k-1}, \quad \text{for}\ \ n \geq 1,\, k \geq 1,
    \]
    with base cases:
    \[
        \unstirling{0}{0} = 1, \qquad \unstirling{n}{0} = 0 \ \ \text{for}\ \ n > 0, \quad \unstirling{0}{k} = 0 \ \ \text{for}\ \ k > 0. 
    \]
    The \emph{signed Stirling numbers of the first kind} $\stirling{n}{k}$ are defined by:
    \[
        \stirling{n}{k} = (-1)^{n-k} \unstirling{n}{k}.
    \]
\end{definition}

\begin{lemma}[Expansion of Falling Factorials]
    \label{lemma:falling_factorial_expansion}
    The rising and falling factorial can be expressed in terms of the Stirling numbers of the first kind as:
    \begin{equation*}
        \risingfact{x}{n} = \sum_{k=0}^{n} \unstirling{n}{k} x^k, 
        \quad \quad
        \fallingfact{x}{n} = \sum_{k=0}^{n} \stirling{n}{k} x^k.
    \end{equation*}    
\end{lemma}
\citet{graham1994concrete}  provides a thorough overview of rising and falling factorials, 
along with Stirling numbers and their combinatorial properties.

\subsection{Properties of Elementary Symmetric Polynomials}
\label{appendix:elementary_symmetric_polynomials}

\begin{remark}[Generating Function for ESPs]
    \label{remark:esps_gen_func}
    The ESPs of a vector \( x \in \mathbb{R}^K \)
    can be expressed through their generating function by
    \[
        \sum_{m=0}^{K} \sigma_m(x) t^m = \prod_{k=1}^{K} (1 + x(k) \cdot t) 
    \]
\end{remark}

\begin{lemma}[Recurrence Relation for ESPs] 
    \label{lemma:elementary_symmetric_polynomial_recurrence_relation}
    Let $m, K \in \Nplus$ and $x \in \Rdim{K}$.
    It holds that
    \begin{equation*}
    \esp{m}{x} = x(k) \cdot \esp{m-1}{ x\minind{k} } + \esp{m}{ x\minind{k} },
    \end{equation*} 
    for every  $k \in [K]$.
\end{lemma}
See \cite{macdonald1998symmetric} for more information about recurrence relations in symmetric polynomials.

\begin{lemma}[Summation of reduced ESPs] 
    \label{lemma:elementary_symmetric_polynomial_sum_identity}
    For all $K \in \Nplus$ and $x \in \Rdim{K}$,
    \begin{equation*}
        \sum_{i=1}^{K} \esp{m}{ x\minind{i} } = (K - m) \cdot \esp{m}{x}
        , \ \  0\le m\le K.
    \end{equation*}
\end{lemma}

\begin{lemma}[Shifted ESP Expansion]
    \label{lemma:elementary_symmetric_polynomial_with_shift}
    Let \( x \in \mathbb{R}^K \), \( t \in \mathbb{R} \), and \( n \in \Nfield \). 
    Then, 
    \begin{equation*}
        \sigma_n(\univec{t}{K} - x) = \sum_{i=0}^{n} (-1)^i \sigma_i(x) \binom{K - i}{n - i} t^{n - i}.
    \end{equation*}
\end{lemma}
The proofs of 
\lemmaref{lemma:elementary_symmetric_polynomial_sum_identity,lemma:elementary_symmetric_polynomial_with_shift} 
are provided in \appendixref{appendix:technical_A}.

\section{Proof of Theorem \ref{theorem:nash}}
\label{appendix:the_adversarial_game_dynamics}

In this section, we prove \theoremref{theorem:nash}, which characterizes a Nash equilibrium of the game.
\begin{lemma}[Compact Action Space and Continuity]
\label{lemma:continuity_of_the_value_function}
For every $H \geq 1$, there exists a \emph{continuous, non-empty and compact-valued correspondence}
$\corr_H : \Rdim{K} \rightrightarrows \intsimplex{K-1}$
(see \appendixref{appendix:correspondences} for definitions),
such that 
\begin{equation}
    \label{eq:compact_selection_value_function}
    \pot{H}{K}{s} = \min_{r \in \corr_H(s)} \; \max_{q \in \Delta} \; \pot{H-1}{K}{s + q \oslash r}.
\end{equation}
Consequently, $\potfunction{ H }{K}$ is continuous for all \(H \in \Nfield\).
\end{lemma}
While no fixed compact subset of the simplex interior contains 
the optimal bookmaker's actions for all $H$ and $s$, 
\lemmaref{lemma:continuity_of_the_value_function} establishes that for each specific pair $(H, s)$, 
such a compact set, $\corr_H(s)$, exists. 
The continuity of the correspondence $\corr_H$ in $s$ ensures that, for all $H$, 
the value function $\potfunction{H}{K}$ is continuous---a key property supporting our analysis and results.
\lemmaref{lemma:continuity_of_the_value_function} 
is built on bounding the value function by the loss of a \naive bookmaker
acting with the uniform distribution $r = \univec{\frac{1}{K}}{K}$ in all rounds.
See \appendixref{appendix:compact_selection_and_continuity} for the full proof.
\begin{lemma}\label{lemma:value_function_if_convex}
    Let \( H \geq 1 \) and suppose that the value function \( \potfunction{H-1}{K} \) is convex.  
    Then, for every state \( s \), there exists an optimal gambler who is decisive; 
    that is, 
    \begin{equation}
        \label{eq:proof_of_the_game_dynamics_paper_eq}
        \pot{H}{K}{s} = \min_{r \in \corr_H(s)} \; \max_{k \in [K]} \; \pot{H-1}{K}{s + \basis{k} \oslash r}.
    \end{equation}
\end{lemma}
\begin{proof}[Proof of \lemmaref{lemma:value_function_if_convex}]
    Let $H \geq 1,\, s \in \Rdim{K}$ and suppose that $\potfunction{H-1}{K}$ is a convex function. 
    The gambler's action is a mixture of $K$ decisive actions;
    that is, any vector $q \in \simplex{}$ can be written as a convex combination of the standard basis vectors $\stdbasis{K}$.  
    Let $X \in \stdbasis{K}$ be a one-hot random variable distributed according to $q$, i.e.,
    \begin{equation}\label{eq:proof_nash_def_of_X}
        \mathbb{P}(X = \mathbf{e}_k) = \vecind{q}{k} \quad \text{for all } k \in [K],
    \end{equation}
    and express \eqref{eq:compact_selection_value_function} as
    \begin{equation*}
        \pot{H}{K}{s} 
        =
        \min_{ r \in \corr_H(s) }\; \max_{ q \in \simplex{}} \;
                \pot{ H-1 }{K}{ s +  \mathbb{E}_{q} \left[X\right] \oslash r  }.
    \end{equation*}
    For any $r \in  \corr_H(s)$ it holds that
    \begin{align*}
        \max_{ q \in \simplex{}} 
                \pot{ H-1 }{K}{ s +  \mathbb{E}_{q} \left[X\right] \oslash r  }
        &\stackrel{(a)}{=} 
        \max_{q \in \simplex{}} \pot{ H-1 }{K}{\mathbb{E}_{q}\left[ s + X \oslash r \right ]}
        \\
        &\stackrel{(b)}{\leq} 
        \max_{q \in \simplex{}} \mathbb{E}_{q} \left[ \pot{ H-1 }{K}{s + X \oslash r} \right]
        \\ 
        &\stackrel{(c)}{=} 
        \max_{q \in \simplex{}} \sum_{k \in \nset{K}} 
            q(k) \cdot  \pot{ H-1 }{K}{s + \basis{k} \oslash r }
        \\ 
        &\stackrel{(d)}{=} 
        \max_{q \in \stdbasis{K}} \sum_{k \in \nset{K}} 
            q(k) \cdot  \pot{ H-1 }{K}{s + \basis{k} \oslash r }
        \\ 
        &=
            \max_{k \in \nset{K}} \pot{ H-1 }{K}{s + \basis{k} \oslash r },
    \end{align*}
    where: $(a)$ follows by the linearity of expectation;
    $(b)$ follows by Jensen's inequality and the assumption that $\potfunction{H-1}{K}$ is a convex function;
    $(c)$ follows by \eqref{eq:proof_nash_def_of_X};
    and $(d)$ holds since the sum is linear in $q$, and a maximum of a linear function over the simplex 
    is attained at a vertex.   
\end{proof}

As the state vector $s \in \mathbb{R}^K$ captures the payouts already committed to the possible outcomes, 
it is natural to expect that if $\hat{s} \succeq s$, 
then the bookmaker's optimal loss given $\hat{s}$ cannot be smaller than given $s$. 
Interestingly, even a single coordinate in $s$ being smaller than the corresponding coordinate in $\hat{s}$ 
results in a strictly smaller optimal loss. 
\begin{lemma}[Coordinate-wise Monotonicity of the Value Function] \label{lemma:cwm_lemma}
    Let  $\hat{s},s \in \Rdim{K}$.
    \begin{enumerate}[label={{\arabic*.}}, ref={\ref{lemma:cwm_lemma}.\arabic*}]
        \item  \label{lemma:cw_weak_m}
        \textbf{\emph{(Weak)}} 
        For every $H \in \Nfield$, \,  $\hat{s} \succeq s \implies \pot{H}{K}{\hat{s}} \geq \pot{H}{K}{s}$.
        \item \label{lemma:cw_strict_m}
        \textbf{\emph{(Strict)}}
        For every $H \in \Nplus$, \,  $\hat{s} \succ s \implies \pot{ H }{K}{\hat{s}} > \pot{ H }{K}{s}$. 
    \end{enumerate}
\end{lemma}
The proof of \lemmaref{lemma:cwm_lemma} is presented in \appendixref{appendix:proof_cwm_lemma}.

\begin{lemma} \label{lemma:proof_opt_bookmaker_under_convexity}
    Let \( H \geq 1 \) and suppose that the value function \( \potfunction{H-1}{K} \) is convex.  
    Then, for every state~$s$, 
    the optimal bookmaker's action $\optr$ is unique and satisfies
    \begin{equation} \tag{\ref{eq:main_results_thm_the_game_dynamics}}
        \pot{H}{K}{s} = \pot{ H-1 }{K}{s + q \oslash \optr } \quad \forall q \in \stdbasis{K}.
    \end{equation}
\end{lemma}
The proof of \lemmaref{lemma:proof_opt_bookmaker_under_convexity}, 
given in \appendixref{appendix:proof_opt_bookmaker_under_convexity}, 
proceeds by contradiction: 
any action $r$ that fails to satisfy \eqref{eq:main_results_thm_the_game_dynamics} enables a redistribution of probability mass that reduces the loss, contradicting optimality. Uniqueness follows from the coordinate-wise strict monotonicity of the value function (\lemmaref{lemma:cw_strict_m}), which ensures no two distinct actions yield the same loss.

\begin{lemma} \label{lemma:convexity_of_value_function}
    $\potfunction{H}{K}$ is a convex function for every $H \in \Nfield$.
\end{lemma}

\begin{proof}[Proof of \lemmaref{lemma:convexity_of_value_function}]
    We prove by induction on $H$.
    In the base case, $\potfunction{0}{K}$ is a convex function as it is the maximum of finite convex functions.   
    Let $H \in \Nplus$ and assume that $\potfunction{H-1}{K}$ is a convex function.
    Let \( a, b \in \Rdim{K},\, \lambda \in [0, 1] \),
    and let
    \( r_a, r_b, r \in \simplex{K-1} \) be the optimal bookmaker 
    actions for states \( a, b \) and \( \lambda a + \left( 1-\lambda \right) b \), 
    respectively.
    It holds that
    \begin{align*}
        \pot{ H }{K}{ \lambda a + \left(1-\lambda \right) b } 
        &\stackrel{(a)}{=} \max_{k \in \nset{K}} \pot{ H-1 }{K}{ 
                \lambda {a} + (1-\lambda) {b} + \basis{k} \oslash \left( \lambda {r} + \left(1-\lambda\right) {r} \right)  
            }
        \\
        &\stackrel{(b)}{\leq} \max_{k \in \nset{K}} \pot{ H-1 }{K}{
                \lambda {a} + (1-\lambda) {b} 
                + {\basis{k}} \oslash {\left( \lambda r_a + \left(1-\lambda\right) r_b \right)}
            }
        \\ 
        &\stackrel{(c)}{\leq} \max_{k \in \nset{K}} \pot{ H-1 }{K}{
                \lambda \left( a +  \basis{k} \oslash r_a \right) 
                + \left(1-\lambda\right) \left( b + \basis{k} \oslash r_b \right)
            }
        \\ 
        &\stackrel{(d)}{\leq} 
            \max_{k \in \nset{K}} \left( \lambda \pot{ H-1 }{K}{ a + \basis{k} \oslash r_a}  
            + \left(1-\lambda\right) \pot{ H-1 }{K}{ b + \basis{k} \oslash r_b }
            \right)
        \\
        &\stackrel{(e)}{\leq} 
            \lambda \max_{k \in \nset{K}} \pot{ H-1 }{K} { a + \basis{k} \oslash r_a } 
            + \left(1-\lambda\right) \max_{k \in \nset{K}} \pot{ H-1 }{K}{ b + \basis{k} \oslash r_b }
        \\
        &\stackrel{(f)}{=} 
            \lambda \pot{ H }{K}{a} + \left(1-\lambda\right) \pot{ H + 1}{K}{b},
    \end{align*}
    where the steps are justified as follows:
    \begin{enumerate}[label={$(\alph*)$}]
        \item 
        Follows from \lemmaref{lemma:proof_opt_bookmaker_under_convexity}.
        \item
        Follows from the optimality of \(r\), which implies that the action $\lambda r_a + (1-\lambda)r_b$
        induces a larger objective.
        \item 
        Follows from the convexity of the function \(x \mapsto \frac{1}{x}\) in $\mathbb{R}_+$ 
        and \lemmaref{lemma:cw_weak_m}. In particular, define the vectors 
        \begin{equation*}
            \begin{aligned}
                s &\eqdef \lambda {a} + (1-\lambda) {b} 
                    + {\basis{k}} \oslash {\left( \lambda r_a + \left(1-\lambda\right) r_b \right)},
                \\
                \hat{s} &\eqdef \lambda \left( a +  \basis{k} \oslash r_a \right) 
                    + \left(1-\lambda\right) \left( b + \basis{k} \oslash r_b \right).
            \end{aligned}
        \end{equation*}
        As $r_a(k),r_b(k) > 0$, it holds that $\lambda r_a(k) + \left(1-\lambda\right) r_b(k) > 0$.
        By weak coordinate-wise monotonicity of the value function,
        $\pot{H-1}{K}{s} \leq \pot{H-1}{K}{\hat{s}}$.
        \item
        By the induction hypothesis, \(\potfunction{H-1}{K}\) is a convex function.
        \item
        Holds as maximizing each function separately can only increase the overall value.
        \item 
        Follows from our choice of $r_a,r_b$ as the optimal actions for states $a$ and $b$, respectively. 
    \end{enumerate}
\end{proof}

\begin{proof}[Proof of Theorem \ref{theorem:nash}]
    Fix $H \in \Nplus$ and $s \in \Rdim{K}$.
    By \lemmaref{lemma:convexity_of_value_function}, the value function $\potfunction{H-1}{K}$ is convex.
    This implies, by \lemmaref{lemma:value_function_if_convex}, the existence of an optimal gambler who is decisive.
    Moreover, \lemmaref{lemma:proof_opt_bookmaker_under_convexity} guarantees uniqueness of the optimal 
    bookmaker action $\optr \in \simplex{}$, which satisfies
    \begin{equation} \tag{\ref{eq:main_results_thm_the_game_dynamics}}
        \pot{H}{K}{s} = \pot{ H-1 }{K}{s + q \oslash \optr } \quad \forall q \in \stdbasis{K}.
    \end{equation}
\end{proof}

\subsection{Proof of Lemma \ref{lemma:continuity_of_the_value_function}}
\label{appendix:compact_selection_and_continuity}

This section proves \lemmaref{lemma:continuity_of_the_value_function}, 
showing that for all $H \geq 1$, 
the minimization over $r$ in the value function (\definitionref{def:value_function}) can be restricted to a continuous, 
compact-valued correspondence 
\[
    \corr_H : \Rdim{K} \rightrightarrows \intsimplex{K-1}. 
\]
Building on this property, 
we establish the continuity of the value function $\potfunction{H}{K}$ with respect to $s$.

We begin with the following lemma:
\begin{lemma}[Uniform Translation] \label{lemma:uniform_translation_property}
    For any \(H \in \Nfield\), $s \in \Rdim{K}$ and $c \in \Rfield$, \;
    \[
        \pot{ H }{K}{s + \univec{c}{K}} = \pot{ H }{K}{s} + c.
    \]
\end{lemma}
The proof of \lemmaref{lemma:uniform_translation_property}, 
provided in \appendixref{appendix:technical_B}, proceeds by induction on $H$ using the recursive definition of the value function in \definitionref{def:value_function}.
\begin{corollary}
    \label{corollary:optimal_bookmaker_state_shift}
    Let $H \in \Nfield$, $s \in \Rdim{K}$ and $c \in \Rfield$. 
    For~the vector $\hat{s} = s + \univec{c}{K}$, it holds that 
    \begin{align*}
        \textstyle\arg\inf_{r \in \simplex{}} 
            \left\{ 
                \max_{q \in \simplex{}} \pot{ H }{K}{s + q \oslash r}                
            \right\}
        = 
        \textstyle\arg\inf_{r \in \simplex{}} 
            \left\{ 
                \max_{q \in \simplex{}} \pot{ H }{K}{\hat{s} + q \oslash r}                
            \right\}.
    \end{align*}
\end{corollary}

The key idea in the proof of \lemmaref{lemma:continuity_of_the_value_function}
is that the value function must be bounded above by the maximal loss incurred under any alternative bookmaking strategy. 
In particular, the value function is upper bounded by the performance of
a \naive bookmaker acting with the uniform distribution $r = \univec{\frac{1}{K}}{K}$ in all rounds.
For $H \in \Nplus$ define the function $\omega_H : \Rdim{K} \to \Rpplus{}$ as
\begin{equation}
    \label{eq:def_of_omega_H_s}
    \omega_{H}(s) 
    \; \eqdef \;
    \frac{1}
    {\max_{i \in [K]} s(i) - \min_{j \in [K]} s(j)  + H K },
\end{equation}    
and define the correspondence  $\corr_H : \Rdim{K} \rightrightarrows \intsimplex{K-1}$ as
\begin{equation}
    \label{eq:def_of_corr_H}
    \corr_{H}(s) 
    \; \eqdef \;
    \{ r \in \simplex{K-1} \mid r \succeq \univec{\omega_H(s)}{K} \}.
\end{equation}
\begin{lemma} \label{lemma:properties_of_corr}
    The correspondence $\corr_H$ is continuous, non-empty and compact-valued
    (see \appendixref{appendix:correspondences} for definitions) for all $H \in \Nplus$ .
\end{lemma}
The proof of \lemmaref{lemma:properties_of_corr} is provided in \appendixref{appendix:technical_B}.

\begin{proof}[Proof of Lemma \ref{lemma:continuity_of_the_value_function}]
Let $H \in \Nplus$, $s \in \Rdim{K}$ and define
\begin{equation}
    \hat{s} \eqdef s - \univec{\min_{j \in [K]} s(j)}{K}. \label{eq:proof_of_uniform_translation_def_of_hat_s}
\end{equation}
Consider a game with $H$ rounds and initial state vector \( \hat{s} \).
An optimal loss is bounded above by the loss of any other bookmaker, hence
\begin{align}
    \label{eq:value_bound_by_naive_bookmaker}
    \pot{ H }{K}{\hat{s}} \leq \max_{i \in [K]} \vecind{\hat{s}}{i} + H K,
\end{align}
where the RHS of \eqref{eq:value_bound_by_naive_bookmaker} is the maximal loss of the \naive bookmaker 
acting with the uniform distribution $r = \univec{\frac{1}{K}}{K}$ in all rounds.
We claim that
\begin{equation}
    \label{eq:proof_of_uniform_translation_hat_s}
     \textstyle\arg\inf_{r \in \simplex{}} 
            \left\{ 
                \max_{q \in \simplex{}} \pot{ H-1 }{K}{\hat{s} + q \oslash r}                
            \right\}
    \subseteq  \corr_H(s).
\end{equation}
Assume in contradiction there exists $r$ in the LHS of \eqref{eq:proof_of_uniform_translation_hat_s}
which is not in the RHS. 
Then there exist $k \in \nset{K}$ and $\kappa > 0$ such that 
\[
    \vecind{r}{k} = 
    \frac{1}{\max_{i \in [K]} s(i) - \min_{j \in [K]} s(j)  + H K  + \kappa}.
\] 
Since the $q$ term in the LHS of \eqref{eq:proof_of_uniform_translation_hat_s} can be $\basis{k}$, 
we obtain 
\begin{align*}
    \pot{H}{K}{\hat{s}} 
    \geq \;   
    & \max_{i \in [K]} s(i) - \min_{j \in [K]} s(j) + H K + \kappa 
    \\ 
    \stackrel{(a)}{=} \;
    & \max_{i \in [K]} \hat{s}(i) + H K + \kappa
    \\ 
    > \; 
    & \max_{i \in [K]} \hat{s}(i) + H K 
    \\ 
    \stackrel{(b)}{\geq}
    &\pot{H}{K}{\hat{s}},
\end{align*}
which yields a contradiction.
Here, $(a)$ follows from the construction of $\hat{s}$ in \eqref{eq:proof_of_uniform_translation_def_of_hat_s},
and $(b)$ follows from the upper bound on $\pot{H}{K}{\hat{s}}$ in \eqref{eq:value_bound_by_naive_bookmaker}.
Combining \eqref{eq:proof_of_uniform_translation_hat_s} with \corollaryref{corollary:optimal_bookmaker_state_shift}, 
we conclude that
\begin{equation}
    \label{eq:proof_of_uniform_translation_s_conclusion}
     \textstyle\arg\inf_{r \in \simplex{}} 
            \left\{ 
                \max_{q \in \simplex{}} \pot{ H-1 }{K}{s + q \oslash r}                
            \right\}
    \subseteq  \corr_H(s).
\end{equation}
Thus, the value function $\potfunction{H}{K}$ is given by
\begin{equation*}    
    \pot{H}{K}{s} = \min_{r \in \corr_H(s)} \max_{q \in \Delta} \pot{H-1}{K}{s + q \oslash r},
\end{equation*}
where by \lemmaref{lemma:properties_of_corr}, $\corr_H(s)$ is a continuous, non-empty and compact-valued correspondence.

We prove by induction on $H$ that $\potfunction{H}{}$ is continuous.
\begin{itemize}
\item \emph{Base case} \(( H = 0 )\):
In the base case, $\potfunction{0}{K}$ is continuous as it is the maximum of finite continuous functions.   
\item 
\emph{Inductive step} \( (H-1 \to H) \):
Let $H \in \Nplus$ and define the function $\mathcal{U} : \Rdim{K} \times \Rplus{K} \to \Rfield$ as
\begin{equation}
    \label{eq:def_of_U_func}
    \mathcal{U}(s,r) \; \eqdef \; \max_{q\in\Delta}\; \pot{H-1}{K}{\, s + q\oslash r\,}.
\end{equation}
We claim that \,$\mathcal{U}(s,r)$ is continuous in $(s,r)$:
for any fixed \(q \in \Delta\), the mapping 
\[
    (s,r) \mapsto s + q\oslash r
\]
is continuous in \((s, r)\) on \(\Rdim{K} \times \Rplus{K}\) as division by a positive number 
and addition of vectors are continuous. 
By the induction hypothesis, \(\potfunction{H-1}{K}\) is continuous, and thus 
the composition 
\(
    (s,r) \mapsto  \pot{H-1}{K}{s + q\oslash r}
\)
is continuous.
Since \(\Delta^{K-1}\) is a compact set, the function
$\mathcal{U}$ is the pointwise maximum of continuous functions over a compact set, and hence continuous in \((s,r)\).
We express the value function as:
\begin{align*}
    \pot{H}{K}{s} 
    &=
    \inf_{r\in \corr_H(s) }\; \max_{q\in\Delta}\; \pot{H-1}{K}{\, s + q\oslash r\,}
    \\ 
    &=
    \inf_{r\in \corr_H(s)} \mathcal{U}(s,r),
\end{align*}
where the first equality follows from \eqref{eq:proof_of_uniform_translation_s_conclusion},
and the second equality follows from the definition of $\mathcal{U}(s,r)$ in \eqref{eq:def_of_U_func}.
The function $\mathcal{U}(s,r)$ is continuous, as established above, 
and the correspondence $\corr_H : \Rdim{K} \rightrightarrows \Delta^{K-1}$ is compact-valued, 
continuous, and non-empty-valued by \lemmaref{lemma:properties_of_corr}. 
Hence, the conditions of Berge's Maximum Theorem (\lemmaref{lemma:berges_maximum_theorem}) are satisfied, 
and $\potfunction{H}{K}$ is continuous.
\end{itemize}
\end{proof}

\subsection{Proof of Lemma \ref{lemma:cwm_lemma}}
\label{appendix:proof_cwm_lemma}

This section provides the proof of \lemmaref{lemma:cwm_lemma}, which establishes coordinate-wise monotonicity of the value function. 
We first prove the strict case (\lemmaref{lemma:cw_strict_m}) and then derive the weak case (\lemmaref{lemma:cw_weak_m}) as a direct adaptation.

\subsubsection{Proof of Lemma \ref{lemma:cw_strict_m}}

The degenerate case \( K = 1 \) follows directly from \definitionref{def:value_function}:
Let $H \in \Nplus$ and \( \hat{s}, s \) be $1$-dimensional vectors such that \( \hat{s} \succ s \);
that is, $\hat{s}(1) > s(1)$.
It holds that
\begin{align*}
    \pot{H}{K}{\hat{s}} 
    = \;
    &\hat{s}(1)+ H 
    \\
    > \;
    &s(1)+ H 
    \\ 
    = \;
    &\pot{H}{K}{s} ,
\end{align*}
where the equalities follow from the fact that $q_t(1) = r_t(1) = 1$ for all $t \in [H]$.

To prove coordinate-wise strict monotonicity in the general case, 
we analyze the game at the last round, as its structure differs from other rounds. 
The following lemma establishes a Nash equilibrium of the game when \( H = 1 \).
\begin{lemma} \label{lemma:nash_last_round}
    In case $H=1$, for every $s \in \Rdim{K}$ 
    the optimal bookmaking action $\optr \in \corr_1(s)$ is unique and satisfies
    \begin{equation} \label{eq:nash_last_round}
        \pot{1}{K}{s} = \pot{ 0 }{K}{s + q \oslash \optr } \quad \forall q \in \stdbasis{K}.
    \end{equation}
\end{lemma}
The proof of \lemmaref{lemma:nash_last_round} is presented in \appendixref{appendix:technical_B}.

\begin{proof}[Proof of \lemmaref{lemma:cw_strict_m}]
Let \( \hat{s}, s \in \Rdim{K} \) be such that \( \hat{s} \succ s \).  
We prove by induction on~$H$. 

\begin{itemize}
\item \emph{Base case} (\( H=1 \)):  
By \lemmaref{lemma:nash_last_round}, it holds that 
\begin{equation*}
    \pot{1}{K}{s} = \vecind{s}{k} + \frac{1}{\vecind{r}{k}} \quad \forall k \in [K]
\end{equation*}
for some \( r \in \simplex{} \).
Assume in contradiction that \( \pot{1}{K}{\hat{s}} \leq \pot{1}{K}{s} \).  
Then there exists \( \hat{r} \in \simplex{} \) such that  
\begin{equation}
    \label{eq:proof_of_coordinate_wise_monotonicity_base}
    \vecind{\hat{s}}{k} + \frac{1}{\vecind{\hat{r}}{k}} \leq \vecind{s}{k} + \frac{1}{\vecind{r}{k}} \quad \forall k \in \nset{K}.
\end{equation}

Since \( \hat{s} \succ s \), there exists \( i \in \nset{K} \) such that \( \vecind{\hat{s}}{i} > \vecind{s}{i} \).  
By \eqref{eq:proof_of_coordinate_wise_monotonicity_base}, it must hold that \( \vecind{\hat{r}}{i} > \vecind{r}{i} \).  
Then, there must exist \( j \in \nset{K}\setminus \{i\} \) such that \( \vecind{\hat{r}}{j} < \vecind{r}{j} \).  
It follows that
\begin{align*}
    \vecind{\hat{s}}{j} + \frac{1}{\vecind{\hat{r}}{j}} 
    \stackrel{(a)}{\geq} 
    \vecind{s}{j} + \frac{1}{\vecind{\hat{r}}{j}} 
    \stackrel{(b)}{>} 
    \vecind{s}{j} + \frac{1}{\vecind{r}{j}},
\end{align*}
where $(a)$ holds as $\hat{s} \succeq s$, and $(b)$ holds since $\vecind{\hat{r}}{j} < \vecind{r}{j}$.
This contradicts \eqref{eq:proof_of_coordinate_wise_monotonicity_base}.

\item \emph{Inductive step} (\( H \to H+1 \)):  
Define 
\begin{align*}    
    \simplex{}_\omega 
    &= 
    \left\{ r \in \simplex{} \, \mid \, r \succeq \univec{ \min \{ \omega_{H+1}(s) ,\, \omega_{H+1}(\hat{s})\} }{K} \right\}.
\end{align*}
$\simplex{}_\omega$ is a compact set for which 
\(
    \corr_H(s), \corr_H(\hat{s}) \subseteq \simplex{}_\omega.
\)
Hence, by \lemmaref{lemma:continuity_of_the_value_function},
\begin{align}    
    \pot{H+1}{K}{s} 
    &= \min_{ r \in \simplex{}_\omega } \; \max_{ q \in \simplex{} } \pot{H}{K}{ s + q \oslash r }
    \label{eq:proof_of_coordinate_wise_monotonicity_inductive_step_value_functions_s}
    \\[6pt]
    \pot{H+1}{K}{\hat{s}} 
    &= \min_{ r \in \simplex{}_\omega } \; \max_{ q \in \simplex{} }  \pot{H}{K}{ \hat{s} + q \oslash r }.
    \label{eq:proof_of_coordinate_wise_monotonicity_inductive_step_value_functions_hat_s}
\end{align}

By \eqref{eq:proof_of_coordinate_wise_monotonicity_inductive_step_value_functions_hat_s}, 
there exists \( \hat{r} \in \simplex{}_\omega \) such that  
\begin{equation}
    \label{eq:monotonicity_of_value_function_b}
    \pot{H+1}{K}{\hat{s}} = \max_{ q \in \simplex{} } \pot{H}{K}{ \hat{s} + q \oslash \hat{r} }.
\end{equation}
Fix
\begin{equation}
    \label{eq:proof_of_coordinate_wise_monotonicity_def_of_bar_q}
    \bar{q} \in \arg \max_{q \in \simplex{}} \pot{H}{K}{ s + q \oslash \hat{r} }.
\end{equation}  
It holds that  
\begin{align*}
    \pot{H+1}{K}{s} 
    &\stackrel{(a)}{=} \;
    \min_{ r \in \simplex{}_\omega } \max_{ q \in \simplex{} } \pot{H}{K}{ s + q \oslash r }
    \\ 
    &\stackrel{(b)}{\leq} \;
    \max_{ q \in \simplex{} } \pot{H}{K}{ s + q \oslash \hat{r} }
    \\ 
    &\stackrel{(c)}{=} \;
    \pot{H}{K}{ s + \bar{q} \oslash \hat{r} }
    \\
    &\stackrel{(d)}{<} \;
    \pot{H}{K}{ \hat{s} + \bar{q} \oslash \hat{r} }
    \\
    &\stackrel{(e)}{\leq} \; 
    \max_{ q \in \simplex{} } \pot{H}{K}{ \hat{s} + q \oslash \hat{r} } 
    \\
    &\stackrel{(f)}{=} \;
    \pot{H+1}{K}{\hat{s}},
\end{align*}
where the steps are justified as follows:
\begin{enumerate}[label={$(\alph*)$}]
    \item 
    Follows from \eqref{eq:proof_of_coordinate_wise_monotonicity_inductive_step_value_functions_s}.
    \item 
    Holds since the minimizing $r$ yields an objective no larger than that induced by any $\hat{r}$.  
    \item 
    Follows from the choice of $\bar{q}$ in \eqref{eq:proof_of_coordinate_wise_monotonicity_def_of_bar_q}.
    \item 
    Follows from the induction hypothesis on $H$. 
    In particular, since \(s \prec \hat{s} \), it holds that
    \[
        s + \bar{q} \oslash \hat{r} \prec \hat{s} + \bar{q} \oslash \hat{r},
    \] 
    and thus, by the induction hypothesis,
    \[
        \pot{H}{K}{ s + \bar{q} \oslash \hat{r} } < \pot{H}{K}{ \hat{s} + \bar{q} \oslash \hat{r} }.
    \]
    \item 
    Holds since taking the maximum over $q$, with $\hat{s}$ and $\hat{r}$ fixed, can only increase the value.
    \item 
    Follows from the choice of $\hat{r}$ as one that satisfies \eqref{eq:monotonicity_of_value_function_b}.
\end{enumerate}
This establishes that \( \pot{H+1}{K}{s} < \pot{H+1}{K}{\hat{s}} \), completing the proof.
\end{itemize}
\end{proof}

\subsubsection{Proof of Lemma \ref{lemma:cw_weak_m}}

\begin{proof}[Proof of \lemmaref{lemma:cw_weak_m}]
    Let \( \hat{s}, s \in \Rdim{K} \) be such that \( \hat{s} \succeq s \).  
    For the case $H=0$, it holds that
    \begin{equation*}
        \pot{0}{K}{\hat{s}}
        \eqdef \;
        \max_{k \in [K]} \hat{s}(k) 
        \geq \;
        \max_{k \in [K]} s(k) 
        = \;
        \pot{0}{K}{s}.
    \end{equation*}
    Assume $H \geq 1$. 
    If $\hat{s} = s$ the statement is trivial.
    Otherwise, since $\hat{s} \succ s$, the result follows from \lemmaref{lemma:cw_strict_m}.
\end{proof}

\subsection{Proof of Lemma \ref{lemma:proof_opt_bookmaker_under_convexity}}
\label{appendix:proof_opt_bookmaker_under_convexity}

\begin{proof}[Proof of \lemmaref{lemma:proof_opt_bookmaker_under_convexity}]
    Let $H \geq 1,\, K \geq 2$ and $s \in \Rdim{K}$.
    The case $H=1$  is treated in \lemmaref{lemma:nash_last_round}, and for $K=1$ the result follows immediately.
    By \lemmaref{lemma:value_function_if_convex}, there exists $r \in \corr_H(s)$ for which
    \begin{equation*}
        \pot{H}{K}{s} = \max_{k \in [K]} \pot{H-1}{K}{s + \basis{k} \oslash r}.
    \end{equation*}
    By coordinate-wise strict monotonicity (\lemmaref{lemma:cw_strict_m}),
    there exists vector $u \in \Rpplus{K}$ such that 
    \begin{equation} \label{eq:opt_bookmaker_def_of_u}
        \pot{ H - 1 }{K}{s + \basis{k} \oslash r}
        = \pot{ H -1 }{K}{s} + u(k)  \quad \forall k \in [K].
    \end{equation}
    Assume, towards a contradiction, 
    that $r$ does not satisfy \eqref{eq:main_results_thm_the_game_dynamics};
    that is, there exists $i \in \nset{K}$ and $\varepsilon > 0$ such that
    \begin{equation}
        \label{eq:definition_of_kappa}
        \max_{k \in \nset{K}} u(k) - u(i) = \varepsilon.
    \end{equation}
    Relying on the continuity of the value function (\lemmaref{lemma:continuity_of_the_value_function})
    and the fact that 
    \begin{equation*}
        \lim_{r(i) \to 0} \pot{ H -1 }{K}{s + \frac{\basis{i}}{\vecind{r}{i}}} = \infty,
    \end{equation*}
    the intermediate value theorem guarantees the existence of a scalar 
    \(0 < \tilde{r} < r(i) \) for which
    \begin{equation}
        \label{eq:definition_of_new_tilde_r}
        \pot{ H -1 }{K}{s + \frac{\basis{i}}{\tilde{r}}}
        = \pot{ H -1 }{K}{s} + u(i) + \frac{\varepsilon}{2}.
    \end{equation}
    We construct a new action $\hat{r}$ as follows:
    \begin{align*}
        \hat{r}(k) = 
        \begin{cases}
            \; \tilde{r} & \text{if } k = i, \\[5pt]
            \; r(k) + \displaystyle\frac{r(i) - \tilde{r} }{K-1} & \text{if } k \neq i.
        \end{cases} 
    \end{align*}
    It is easy to verify that $\hat{r} \in \simplex{K-1}$.
    By acting with $\hat{r}$ the bookmaker's loss decreases:
    \begin{itemize}
    \item 
    \emph{For $k = i$}: 
    \begin{align*}
        \pot{ H -1 }{K}{s + \frac{\basis{i}}{\hat{r}(i)}} 
         &\stackrel{(a)}{=} 
        \pot{ H-1 }{K}{s} + u(i) + \frac{\varepsilon}{2} 
        \\
        &\stackrel{(b)}{ < } 
        \pot{ H -1 }{K}{s} + \max_{k \in \nset{K}} u(k),
    \end{align*}
    where $(a)$ follows from our choice of $\hat{r}(i) = \tilde{r}$ to satisfy \eqref{eq:definition_of_new_tilde_r},
    and $(b)$ follows from \eqref{eq:definition_of_kappa}, verifying a gap of $\frac{\varepsilon}{2} > 0$.
    \item 
    \emph{For $k \in \nset{K} \setminus \{i\}$}:
    \begin{align*}
        \pot{ H -1}{K}{s + \frac{\basis{k}}{\hat{r}(k)}}
        &\stackrel{(a)}{<} 
        \pot{ H -1}{K}{s + \frac{\basis{k}}{r(k)}}
        \\
        &\stackrel{(b)}{\leq} 
        \pot{ H -1}{K}{s} + \max_{k \in \nset{K}} u(k),
    \end{align*}
    where $(a)$ follows from coordinate-wise strict monotonicity (\lemmaref{lemma:cw_strict_m})
    combined with the fact that \( \hat{r}(k) > r(k) \),
    and $(b)$ follows from \eqref{eq:opt_bookmaker_def_of_u}.
    \end{itemize}
    It follows that $r$ is suboptimal, contradicting the optimality assumption.

    It remains to show that the optimal bookmaking action is unique.
    Let $\optr, \hat{r} \in \corr_H(s)$ be two vectors that satisfy \eqref{eq:main_results_thm_the_game_dynamics}.
    I.e., 
    \[
        \pot{ H-1 }{K}{s + \frac{\basis{k}}{\optr(k)}} 
        = \pot{ H-1 }{K}{s + \frac{\basis{k}}{r(k)}} \quad \forall k \in [K].
    \]
    Suppose, for the sake of contradiction,  
    that there exists $k \in [K]$ such that $\optr(k) \neq r(k)$,
    and without loss of generality, assume $r(k) > \optr(k)$.
    By coordinate-wise strict monotonicity (\lemmaref{lemma:cw_strict_m}),
    it follows that
    \[
        \pot{ H-1 }{K}{s + \frac{\basis{k}}{\optr(k)}} 
        < \pot{ H-1 }{K}{s + \frac{\basis{k}}{r(k)}},
    \]
    contradicting the optimality of $r$.
\end{proof}

\section{Proof of Lemma \ref{lemma:main_lemma_in_conditions}}
\label{appendix:value_vectors}

In this section, we prove \lemmaref{lemma:main_lemma_in_conditions}, which establishes necessary constraints for $H$-achievable vectors.
We begin by examining the partial derivatives of the polynomial $\denomfunction{H}{K}$, defined in \eqref{eq:main_results_def_of_denom_poly},
as presented in the following lemma.
\begin{lemma}[Partial Derivatives of $\denomfunction{H}{K}$] \label{lemma:der_of_denom}
For every \(H, K \in\Nplus\),
\begin{enumerate}[label={\arabic*.}, ref={\ref{lemma:der_of_denom}.\arabic*}]
\item \label{lemma:partial_derivatives_of_denom_1}
$\denomfunction{H}{K}$ is an infinitely differentiable function.
  
\item   \label{lemma:partial_derivatives_of_denom_2}
For every \(m \in \nset{K}\) and \( \mathfrak{I} \in \binom{\nset{K}}{m} \) 
\begin{align*}
    \frac{ \partial^m  \denomfunction{H}{K}}{\partial^m v_{\mathfrak{I}} } (v)
    = \denom{H}{K-m}{v\minind{\mathfrak{I}}},
\end{align*}
where for \( \mathfrak{I} = \{ i_1, \ldots, i_m \} \) the expression \( \partial^m v_{\mathfrak{I}} \) 
stands for \( \partial v(i_1) \ldots \partial v(i_m) \).
    
\item \label{lemma:partial_derivatives_of_denom_3}
For every \(k\in[K]\) and $m > 1$, 
\[
    \frac{\partial^m \denomfunction{H}{K}}{\partial v(k)^m}\, (v) = 0.
\]
\end{enumerate}
\end{lemma}
The proof of \lemmaref{lemma:der_of_denom} is provided in \appendixref{appendix:technical_C}
and builds on the recurrence relation established in \lemmaref{lemma:recurrence_relation_or_the_denom}:
\[
     \denom{ H }{K}{v} = \vecind{v}{k} \cdot \denomfunction{ H }{K-1}( v\minind{k} ) - H \cdot \denomfunction{ H-1 }{K-1}(v\minind{k} ),
\]
for all $H,K \in \Nplus$.

For the analysis, we use an equivalent form of the polynomial \(\denomfunction{H}{K}\), given by
\begin{equation}
  \label{eq:def_of_ppoly_explicit}
  \denom{H}{K}{v}
  = \sum_{m=0}^{K} (-1)^m \fallingfact{H}{m}\,\esp{K-m}{v}
\end{equation}
for all $H,K\in\Nfield$, which follows from \lemmaref{lemma:falling_factorial_identities_3}.
For convenience, we introduce an alternative notation for the polynomial \(\denomfunction{H}{K}\).
For $H \in \Nplus$ and $K \in \Nfield$, define the polynomial
\begin{equation}
    \label{eq:def_of_num_poly}
    \num{H}{K}{v} \; \eqdef \; H \cdot \denom{H-1}{K}{v}.
\end{equation}
This allows us to express \lemmaref{lemma:recurrence_relation_or_the_denom} in the following equivalent form:  
\begin{equation}    
    \label{eq:appendix_expansion_of_denom}
     \denom{ H }{K}{v} = \vecind{v}{k} \cdot \denomfunction{ H }{K-1}( v\minind{k} ) - \numfunction{ H }{K-1}(v\minind{k} ),
\end{equation}
for all \(H,K \in \Nplus \).

We prove \lemmaref{lemma:proof_of_bp_2} using the equivalent from of \eqref{eq:necessary_conditions_proof_value_of_v_k},
expressed with the notation $\num{H}{K}{v}$:
\begin{equation*} 
    \vecind{v}{k} 
    = \frac{
        \num{ H }{K-1}{v\minind{k}}
    }{
        \denom{ H }{K-1}{v\minind{k}}
    }.
\end{equation*}
To prove \lemmaref{lemma:proof_of_bp_3}, we will show that
\begin{equation} \label{eq:derivation_of_proof_of_bp_3}
    \textstyle \forall\, m \in [K],\; \forall \, \mathfrak{I} \in \binom{\nset{K}}{m}, \quad \denom{H}{K-m}{v\minind{\mathfrak{I}}} > 0.
\end{equation}
The following lemma establishes that this condition implies the same result:
\begin{lemma} \label{lemma:derivation_of_proof_of_bp_3}
    Let \(H,K\in\Nplus\) and \(v\in\Rdim{K}\).  
    If \(v\) satisfies the constraint in \lemmaref{lemma:proof_of_bp_2}, then the positivity condition  
    \eqref{eq:derivation_of_proof_of_bp_3} implies that \(v\) also satisfies the constraint in \lemmaref{lemma:proof_of_bp_3}.
\end{lemma}
\begin{proof}[Proof of \lemmaref{lemma:derivation_of_proof_of_bp_3}] 
For a fixed $k \in [K]$, it holds that
\begin{equation} \label{eq:proof_of_thm_partial_derivative}
    \frac{ \partial  \denomfunction{H}{K}}{\partial v(k) } (v)
    \stackrel{(a)}{=} \denom{H}{K-1}{v\minind{k}}  
    \stackrel{(b)}{>} 0\,,
\end{equation}
where $(a)$ follows from \lemmaref{lemma:partial_derivatives_of_denom_2} and $(b)$ follows by \eqref{eq:derivation_of_proof_of_bp_3}.
Thus, 
\begin{itemize}
    \item By \lemmaref{lemma:proof_of_bp_2}, $\denom{H}{K}{v} = 0$ .
    \item By \eqref{eq:proof_of_thm_partial_derivative}, \( \nabla \denom{H}{K}{v} \) is a strictly positive vector.
    \item By \lemmaref{lemma:partial_derivatives_of_denom_3}, for every \(k\in[K]\) and $m > 1$, 
    \[
        \frac{\partial^m \denomfunction{H}{K}}{\partial v(k)^m}\, (v) = 0,
    \]
    indicating \(\denomfunction{H}{K}\) has no concavity along any axis.
\end{itemize}
We conclude that for all \(  u \in \Rdim{K}, \ \ u \succ v \implies \denom{H}{K}{u} > 0 \),
as \lemmaref{lemma:proof_of_bp_3} states.
\end{proof}

The following lemma provides a lower bound on any $H$-achievable vector.
\begin{lemma}[Lower Bound on $H$-Achievable Vectors] \label{lemma:lower_bound_for_value_vectors}  
    For all $H,K \in \Nplus$, if $v \in \vvset{H,K}$ then $v \succeq \univec{H}{K}$.
    Moreover, when $K>1$, the inequality is strict: $v \succ \univec{H}{K}$.
\end{lemma}
This lower bound follows from the fact that the gambler may commit to a single outcome in each of the remaining $H$ rounds.
According to the definition of the ESP (\definitionref{def:esp}), it follows that for any $H$-achievable vector $v$, 
\begin{equation} \label{eq:pos_esp}
    \textstyle \forall\, m \in [K],\; \forall \, \mathfrak{I} \in \binom{\nset{K}}{m}, \quad \esp{K-m}{v\minind{\mathfrak{I}}} > 0.
\end{equation}
Consequently, throughout the proof of \lemmaref{lemma:main_lemma_in_conditions}, we assume that no division by zero occurs when dividing by sums or products of elements of $v$.
A formal proof of \lemmaref{lemma:lower_bound_for_value_vectors} is provided in \appendixref{appendix:technical_C}.

\begin{proof}[Proof of \lemmaref{lemma:main_lemma_in_conditions}]
The case $K=1$ is immediate from \definitionref{def:bpf}; 
for completeness, a formal proof is provided in \appendixref{appendix:technical_C}.  
We then consider \(K\ge2\), and prove by induction on \(H\).

\paragraph{Base case ($H=1$).}
If $v \in \vvset{1,K}$ then, by \definitionref{def:bpf}, there exists $r \in \Rdim{K}$ such that
\begin{enumerate}[label=1.\roman*., ref=1.\roman*]
    \item \label{base_case_condition_1} 
    $\sum_{k=1}^K \vecind{r}{k} = 1$.
    
    \item \label{base_case_condition_3} 
    $\forall k \in \nset{K}, \; \vecind{v}{k} = \frac{1}{\vecind{r}{k}}$.
\end{enumerate}
The following hold for every $k\in [K]$:
\begin{align}  
    \denom{0}{K}{v} &= \esp{K}{v} \label{eq:base_denom_0_K}
    \\
    \denom{0}{K-1}{ v\minind{k} } &= \esp{K-1}{ v\minind{k} }  \label{eq:base_denom_0_K-1}
    \\
    \denom{1}{K-1}{ v\minind{k} } &= 
        \textstyle\prod_{\indneq{i}{\nset{K}}{k}} \vecind{v}{i} \left(1 - \sum_{\indneq{j}{\nset{K}}{k}} \frac{1}{\vecind{v}{j}}\right)  \label{eq:base_denom_1_K-1}
    \\
    \num{1}{K-1}{ v\minind{k} } &= \esp{K-1}{ v\minind{k} } \label{eq:base_num_1_K-1}
\end{align}
Condition \ref{base_case_condition_3} is satisfied if and only if
\begin{equation}
    \label{eq:proof_of_bpe_r_equation}
    \forall k \in [K] \quad 
    \vecind{r}{k} 
    = \frac{1}{\vecind{v}{k}}.
\end{equation}
For every $k \in [K]$, it holds that 
\begin{align*}
    \frac{1}{\vecind{v}{k}} 
    &\stackrel{(a)}{=} \frac{
        \esp{K-1}{v\minind{k}}
    }{
        \esp{K}{v}
    }
    \\
    &\stackrel{(b)}{=} \frac{
        \denom{0}{K-1}{ v\minind{k} }
    }{
        \denom{0}{K}{v}
    },
\end{align*}
where  $(a)$ follows from \definitionref{def:esp} of the ESP, 
and $(b)$ follows from \eqref{eq:base_denom_0_K-1} and \eqref{eq:base_denom_0_K}.
Therefore, \lemmaref{lemma:proof_of_bp_1} holds for the base case.


The vector $r$, which is defined as in \eqref{eq:proof_of_bpe_r_equation}, should satisfy Condition \ref{base_case_condition_1}; i.e.,
\begin{equation} \label{eq:proof_of_bpe_base_case_sum_of_r_by_v}    
    \sum_{k \in [K]} \frac{1}{v(k)} = 1.
\end{equation}
For every $k \in \nset{K}$, it holds that
\begin{align*}
    \frac{1}{\vecind{v}{k}} 
    &= \;
    1 - \sum_{\indneq{j}{\nset{K}}{k}} \frac{1}{\vecind{v}{j}}
    \\
    &= \;
    \frac{ \prod_{\indneq{i}{\nset{K}}{k}} \vecind{v}{i} }
            { \prod_{\indneq{i}{\nset{K}}{k}} \vecind{v}{i} }
      \left(1 -  \sum_{\indneq{j}{\nset{K}}{k}} \frac{1}{\vecind{v}{j}} \right)
    \\ 
    &= \;  
    \frac{\denom{1}{K-1}{ v\minind{k} }}{\num{1}{K-1}{ v\minind{k} }},
\end{align*}
where the last equality follows from \eqref{eq:base_denom_1_K-1} and \eqref{eq:base_num_1_K-1}.
Therefore,
\begin{equation}
    \label{eq:value_of_v_k_in_terms_of_N_and_D_base_case}
    \vecind{v}{k} = \frac{\num{1}{K-1}{ v\minind{k} }}{\denom{1}{K-1}{ v\minind{k} }},
\end{equation}
and \lemmaref{lemma:proof_of_bp_2} holds for the base case.


We prove that \eqref{eq:derivation_of_proof_of_bp_3} holds by induction on \(m\).  
Combined with \lemmaref{lemma:derivation_of_proof_of_bp_3}, this completes the proof of \lemmaref{lemma:proof_of_bp_3} for the base case.
\begin{itemize}
\item 
\emph{Base case} \(( m = 1 )\):
For every $k \in [K]$, it holds that  
\begin{align*}
    0 
    &\stackrel{(a)}{=} 
    \denom{1}{K}{v}
    \\
    &\stackrel{(b)}{=} v(k) \cdot \denom{1}{K-1}{ v\minind{k} } - \num{1}{K-1}{ v\minind{k} }
    \\ 
    &\stackrel{(c)}{=}  v(k) \cdot \denom{1}{K-1}{ v\minind{k} } - \esp{K-1}{v\minind{k}}
\end{align*}
where 
$(a)$ follows from \lemmaref{lemma:proof_of_bp_2}; 
$(b)$ follows from the identity in \eqref{eq:appendix_expansion_of_denom};
and $(c)$ follows from \eqref{eq:base_num_1_K-1}.
By \eqref{eq:pos_esp}, both $v(k)$ and $\esp{K-1}{v\minind{k}}$ are $> 0$, 
and therefore,
\( \mathcal{D}_{1,K-1}\left( v\minind{k} \right) \) 
must be $ > 0$. 

\item
\emph{Inductive step} \( (m \to m+1) \):
Let \(m \in [K-1] \,, \mathfrak{I} \in \binom{\nset{K}}{m} \) and \( k \in \nset{K} \setminus \mathfrak{I} \) be fixed.
By \eqref{eq:appendix_expansion_of_denom}, it holds that
\begin{equation*}
    \denom{1}{K-m}{v\minind{\mathfrak{I}}} 
    = v(k) \cdot \denom{1}{K-(m+1)}{ v\minind{\mathfrak{I} \cup \{k\}} } 
        - \esp{K-(m+1)}{ v\minind{\mathfrak{I} \cup \{k\}} }.
\end{equation*}
From the induction hypothesis (on $m$) The LHS of the equation is positive. 
By \eqref{eq:pos_esp}, the term
\( \denom{1}{K-(m+1)}{ v\minind{\mathfrak{I} \cup \{k\}} } \)
must be $ > 0$. 
\end{itemize}


\bigbreak

\paragraph{Inductive step ($H \to H+1$).}
If $v \in \vvset{H+1,K}$ then
there exists an action $r \in \Rdim{K}$ such that 
\begin{enumerate}[label=2.\roman*., ref=2.\roman*]
    \item \label{condition_induction_1} 
    $\sum_{k=1}^K \vecind{r}{k} = 1$.
    
    \item \label{condition_induction_3} 
    For every $k \in \nset{K}$, the vector $\prescript{k}{}{v}$, defined as
    \begin{equation}
        \label{eq:lemma_definition_of_k_v}
        \vecind{\prescript{k}{}{v}}{i}
        \eqdef 
        \begin{cases}
            \vecind{v}{i} - \frac{1}{\vecind{r}{k}} & \text{if } k = i, \\
            \vecind{v}{i} & \text{otherwise}.
        \end{cases} 
    \end{equation} 
    is $H$-achievable.
\end{enumerate}
Note that 
\begin{equation} \label{eq:helper_eq_proof_res}
    \prescript{k}{}{v}\minind{k} = v\minind{k}, \quad \forall k \in [K].
\end{equation}
For every $k \in [K]$, it holds that
\begin{align}
    0 
    &\stackrel{(a)}{=} \;
    \denom{H}{K} {\prescript{k}{}{v} }  \label{eq:proof_of_bpe_inductive_step_eq_1}
    \\
    &\stackrel{(b)}{=} \; 
     \vecind{\prescript{k}{}{v}}{k} \cdot \denom{H}{K-1}{ \prescript{k}{}{v}\minind{k} }
        - \num{H}{K-1}{ \prescript{k}{}{v}\minind{k} } \nonumber
    \\ 
    &\stackrel{(c)}{=} \;
    \left( \vecind{v}{k} - \frac{1}{\vecind{r}{k}} \right)  \cdot \denom{H}{K-1}{ v\minind{k} }
        - \num{H}{K-1}{ v\minind{k} } \nonumber
    \\ 
    &= \;
    \vecind{v}{k} \cdot \denom{H}{K-1}{ v\minind{k} } - \num{H}{K-1}{ v\minind{k} }
        - \frac{\denom{H}{K-1}{ v\minind{k} }}{\vecind{r}{k}} \nonumber
    \\ 
    &\stackrel{(d)}{=} \;
    \denom{H}{K}{ v } - \frac{\denom{H}{K-1}{ v\minind{k} }}{\vecind{r}{k}},  \label{eq:last_eq_in_proof_of_bpe_inductive_step_eq_1}
\end{align}
where the steps are justified as follows:
\begin{enumerate}[label={$(\alph*)$}]
    \item Follows from Condition \ref{condition_induction_3} and the induction hypothesis on \lemmaref{lemma:proof_of_bp_2}.
    \item Follow from the identity in \eqref{eq:appendix_expansion_of_denom}.
    \item Follow from the definition of the vector $\prescript{k}{}{v}$ in \eqref{eq:lemma_definition_of_k_v} and \eqref{eq:helper_eq_proof_res}.
    \item Follow from the identity in \eqref{eq:appendix_expansion_of_denom}.
\end{enumerate}
Condition \ref{condition_induction_3}, \eqref{eq:helper_eq_proof_res} and the induction hypothesis on \lemmaref{lemma:proof_of_bp_3}, 
imply
\begin{equation}
    \label{eq:proof_of_bpe_inductive_step_eq_2}
    \denom{H}{K-1}{v\minind{k}} > 0 \quad \forall k \in [K].
\end{equation}
By \eqref{eq:last_eq_in_proof_of_bpe_inductive_step_eq_1}, \eqref{eq:proof_of_bpe_inductive_step_eq_2} 
and the fact that $r(k)$ must be $>0$, we conclude that  
\begin{equation}
    \label{eq:proof_of_bpe_denom_in_r_is_non_zero}
   \denom{H}{K}{v} \neq 0,
\end{equation}
and we can divide by this term.
We obtain that $r$ is given by
\begin{equation}
    \label{eq:lemma_value_vector_inductive_step_formula_for_r}
    \vecind{r}{k} = \frac{
                \denom{ H }{K-1}{ v\minind{k} }
            }{
                \denom{ H }{K}{v}
            }
        \quad \quad \forall k \in \nset{K}.
\end{equation}
This completes the inductive step for \lemmaref{lemma:proof_of_bp_1}.


Let $k \in \nset{K}$ be fixed and let $r$ be a vector that is generated as in
\eqref{eq:lemma_value_vector_inductive_step_formula_for_r}.
$r$ satisfies Condition \ref{condition_induction_3} if and only if 
\begin{equation}
    \label{eq:semi_value_vector_inductive_step_r_k_as_sum}
    r(k) = 1 - \sum_{\indneq{i}{\nset{K}}{k}} r(i).
\end{equation}
By \eqref{eq:proof_of_bpe_denom_in_r_is_non_zero}, \eqref{eq:semi_value_vector_inductive_step_r_k_as_sum} holds if and only if 
\begin{equation}
    \label{eq:semi_value_vector_inductive_step_r_k_as_sum_mult}
    \denom{H}{K}{v} \cdot r(k) = \denom{H}{K}{v} - \denom{H}{K}{v} \sum_{\indneq{i}{\nset{K}}{k}} r(i).
\end{equation}
By \eqref{eq:lemma_value_vector_inductive_step_formula_for_r},
the LHS of \eqref{eq:semi_value_vector_inductive_step_r_k_as_sum_mult} is 
\(
    \denom{H}{K-1}{v\minind{k}},
\)
and for every $i \in \nset{K} \setminus \{k\}$,
\begin{align*}
    \denom{H}{K}{v} \cdot\vecind{r}{i} 
    = \;
    &\denom{H}{K-1}{v\minind{i}}
    \\
    = \;
    &\vecind{v}{k} \cdot \denom{H}{K-2}{v\minind{\{i,k\}}} - \num{H}{K-2}{v\minind{\{i,k\}}},
\end{align*}
where the second equality follows from the identity in \eqref{eq:appendix_expansion_of_denom}. 
Thus,
\begin{equation*}
    \denom{H}{K}{v} \sum_{\indneq{i}{\nset{K}}{k}} \vecind{r}{i} 
    = \vecind{v}{k}  \sum_{\indneq{i}{\nset{K}}{k}} \denom{H}{K-2}{v\minind{\{i,k\}}} 
        - \sum_{\indneq{i}{\nset{K}}{k}} \num{H}{K-2}{v\minind{\{i,k\}}},
\end{equation*}
and \eqref{eq:semi_value_vector_inductive_step_r_k_as_sum_mult} holds if and only if
\[
    \denom{H}{K-1}{v\minind{k}}
    = 
    \denom{H}{K}{v} - 
    \vecind{v}{k}  \sum_{\indneq{i}{\nset{K}}{k}} \denom{H}{K-2}{v\minind{\{i,k\}}} 
    +
    \sum_{\indneq{i}{\nset{K}}{k}} \num{H}{K-2}{v\minind{\{i,k\}}}
    .
\]
Further expanding $\denom{H}{K}{v}$ using \eqref{eq:appendix_expansion_of_denom} and rearrange, we get that
\begin{equation}
    \label{eq:proof_of_bp_inductive_step_part_2_main_eq}
    \mathtt{A} = v(k) \cdot \mathtt{B},
\end{equation}
with
\begin{align} 
    \mathtt{A} 
    &= 
    \denom{H}{K-1}{v\minind{k}}
        + \num{H}{K-1}{v\minind{k}}
        - \sum_{\indneq{i}{\nset{K}}{k}} \num{H}{K-2}{v\minind{\{i,k\}}},
    \label{eq:proof_of_bp_inductive_step_part_2_def_A}
    \\
    \mathtt{B} 
    &=
    \denom{H}{K-1}{v\minind{k}} - \sum_{\indneq{i}{\nset{K}}{k}} \denom{H}{K-2}{v\minind{\{i,k\}}}.
    \label{eq:proof_of_bp_inductive_step_part_2_def_B}
\end{align}
A simplification of 
\equationref{eq:proof_of_bp_inductive_step_part_2_def_A,eq:proof_of_bp_inductive_step_part_2_def_B}
is presented in 
\appendixref{appendix:proof_of_bp_inductive_step_part_2_simplification}, 
yielding:
\begin{equation}
    \label{eq:proof_of_bp_inductive_step_part_2_simplification}
    \mathtt{A} = \num{H+1}{K-1}{v\minind{k}} 
    \quad 
    \mathtt{B} = \denom{H+1}{K-1}{v\minind{k}}.
\end{equation}
Thus, \eqref{eq:proof_of_bp_inductive_step_part_2_main_eq} holds if and only if 
\begin{equation}
    \label{eq:proof_of_bp_inductive_step_part_2_main_eq_expanded}
     \num{H+1}{K-1}{v\minind{k}} = v(k) \cdot \denom{H+1}{K-1}{v\minind{k}}.
\end{equation}
The LHS of \eqref{eq:proof_of_bp_inductive_step_part_2_main_eq_expanded} $ \neq 0$;
otherwise, by \eqref{eq:def_of_num_poly}, it implies that
\(
    \denom{H}{K-1}{v\minind{k}} = 0
\), which contradicts \eqref{eq:proof_of_bpe_inductive_step_eq_2}.
By \lemmaref{lemma:lower_bound_for_value_vectors}, $v(k) > 0$. 
It follows that
\(
    \denom{H+1}{K-1}{v\minind{k}} \neq 0,
\)
allowing us to divide both sides of \eqref{eq:proof_of_bp_inductive_step_part_2_main_eq_expanded}, yielding  
\[
    v(k) = \frac{\num{H+1}{K-1}{v\minind{k}}}{\denom{H+1}{K-1}{v\minind{k}}}.
\]
This completes the inductive step for \lemmaref{lemma:proof_of_bp_2}.


We prove that \eqref{eq:derivation_of_proof_of_bp_3} holds by induction on \(m\).  
Combined with \lemmaref{lemma:derivation_of_proof_of_bp_3}, this completes the proof of \lemmaref{lemma:proof_of_bp_3}.
\begin{itemize}
\item
\emph{Base case} \(( m = 1 ) \):
For every $k \in [K]$, it holds that
\begin{align*}
    0 
    &\stackrel{(a)}{=} 
    \denom{H+1}{K}{v}
    \\ 
    &\stackrel{(b)}{=} 
    v(k) \cdot \denom{H+1}{K-1}{ v\minind{k} } - (H+1) \cdot \denom{H}{K-1}{ v\minind{k} }    
\end{align*}
where $(a)$ follows from \lemmaref{lemma:proof_of_bp_2} 
and $(b)$ follows from \lemmaref{lemma:recurrence_relation_or_the_denom}. 
By \eqref{eq:proof_of_bpe_inductive_step_eq_2} and \lemmaref{lemma:lower_bound_for_value_vectors},
 \( \denom{H+1}{K-1}{ v\minind{k} } \) must be $ > 0$. 

\item
\emph{Inductive step} \(( m \to m+1 )\):
Let $m \in [K-1]$,  \( \mathfrak{I} \in \binom{\nset{K}}{m}\) and $k \in \nset{K} \setminus \mathfrak{I}$ be fixed.
By the recurrence relation in \lemmaref{lemma:recurrence_relation_or_the_denom},
\[
    \denom{H+1}{K-m}{v\minind{\mathfrak{I}}} 
    =
     v(k) \cdot \denom{H+1}{K-(m+1)}{ v\minind{\mathfrak{I} \cup \{k\}} } 
        - (H+1) \cdot \denom{H}{K-(m+1)}{ {v}\minind{\mathfrak{I} \cup \{k\}} }
\]
By the induction hypothesis on $m$, the LHS of the equation is $> 0$.
By \lemmaref{lemma:lower_bound_for_value_vectors}, $v(k) > 0$.
Similarly to \eqref{eq:helper_eq_proof_res}, 
it holds that 
\(
    {v}\minind{\mathfrak{I} \cup \{k\}} = \prescript{k}{}{v}\minind{\mathfrak{I} \cup \{k\}}
\).
By Condition \ref{condition_induction_3} and the induction hypothesis (on $H$) on \lemmaref{lemma:proof_of_bp_3},
\(
    \denom{H}{K-(m+1)}{ {v}\minind{\mathfrak{I} \cup \{k\}} } > 0
\).
Therefore, the term \( \denom{H}{K-(m+1)}{ v\minind{\mathfrak{I} \cup \{k\}} } \) must be $ > 0$.
\end{itemize}
\end{proof}

\begin{remark} \label{remark:validity_of_r}
For every \( v \in \vvset{H, K} \), the vector \( r \) defined in \eqref{eq:optimal_odd_thm} is a valid probability distribution with non-zero entries. 
This follows from the structure of \( \vvset{H, K} \): for each such \( v \), \( r \) is the only choice that satisfies the necessary constraints. 
For completeness, an explicit proof is provided in \appendixref{appendix:validity_of_r}.
\end{remark}

\section{Omitted Proofs for Section \ref{sec:main_results} (Main Results)}
\label{appendix:main_results}

This appendix contains the proofs of \theoremref{theorem:regret_factor} 
and the remarks from \sectionref{sec:main_results}.
\appendixref{appendix:proof_of_theorem_B} proves \theoremref{theorem:regret_factor},
while 
\appendixref{appendix:proof_remark_opt_loss_grows_to_infinity_for_fixed_T,appendix:proof_of_all_polys_are_denom,appendix:complexity_analysis,appendix:prev_algo_comparison,appendix:oracle_optimal_bookmaking} 
prove 
\remarkref{remark:opt_loss_grows_to_infinity_for_fixed_T,remark:all_polys_are_denom,remark:algo_complexity,remark:prev_algo_comparison,remark:approximate_root_finding}, respectively.

\subsection{Proof of Theorem \ref{theorem:regret_factor} (The Asymptotic Regret Factor)}
\label{appendix:proof_of_theorem_B}

To characterize the regret in \eqref{eq:main_results_optimal_loss_decomposition}, 
we begin by defining the polynomial 
\begin{equation}
    \label{eq:def_of_hatppoly}
   \hatppoly{T}{K}{x} \eqdef \ppoly{T}{K}{x+T}.
\end{equation}
By \theoremref{theorem:optimal_loss}, it holds that 
\begin{equation}
    \label{eq:optimal_loss_in_terms_of_hatppoly}
    \optimalloss{T}{K} = T + \argmaxroot{\hatppolyfunction{T}{K}}.
\end{equation}
Hence, the regret can be expressed as 
\begin{equation} \label{eq:appendix_def_of_regret}
    R_{T,K} = \argmaxroot{\hatppolyfunction{T}{K}}.
\end{equation}
We show in \appendixref{appendix:derivation_of_the_explicit_form_of_tilde_P} that 
$\hatppolyfunction{T}{K}$
can be written as
\begin{equation}
    \label{eq:expansion_of_hatppoly}
    \hatppoly{T}{K}{x} = \sum_{m=0}^{K} x^{K-m}\,\binom{K}{m}\left(\sum_{d=0}^{m} \sum_{i=0}^{d} (-1)^{d} \binom{m}{d}\, \stirling{d}{i}\, T^{m-(d-i)}\right),
\end{equation}
where $\stirling{\cdot}{\cdot}$ stands for the signed Stirling numbers of the first kind (see \definitionref{def:stirling_numbers_of_the_first_kind}).
We define the polynomial \(\betapolyfunction{T}{K}\) as
\[
    \betapoly{T}{K}{x} \,\eqdef\, \hatppoly{T}{K}{\sqrt{T} x}.
\]
Following \equationref{eq:optimal_loss_in_terms_of_hatppoly,eq:appendix_def_of_regret}, 
\[
    \optimalloss{T}{K} = T + \sqrt{T} \cdot\argmaxroot{\betapolyfunction{T}{K}},
\]
and 
\begin{equation}
    \label{eq:appendix_beta_def}
    \beta_{T,K} \eqdef \frac{R_{T,K}}{\sqrt{T}} = \argmaxroot{\betapolyfunction{T}{K}}.
\end{equation}
In \appendixref{appendix:derivation_of_the_explicit_form_of_betapoly} we show that
\begin{align}
    \betapoly{T}{K}{x} 
    &= 
    \sum_{m=0}^{K} x^{K-m}\,\binom{K}{m} \cdot T^{\frac{K}{2}} \cdot \widetilde{c}_{T,m}, 
    \label{eq:expnasion_of_betapoly}
    \\
    \text{where} \quad \quad
    \widetilde{c}_{T,m} &= 
    \sum_{n=0}^{m} T^{\frac{m}{2} - n} \sum_{d=0}^{m} (-1)^{d} \binom{m}{d}\, \stirling{d}{d-n}. \label{eq:expnasion_of_betapoly_coeff}
\end{align}
\begin{lemma} \label{lemma:max_power_of_T_in_betapoly}
    For all $m \in \Nfield$, the maximal power of $T$ in $\widetilde{c}_{T,m}$ is $\leq 0$.
    As a result, the maximal power of $T$ in $\betapolyfunction{T}{K}$ is at most $\frac{K}{2}$.
\end{lemma}
The proof of \lemmaref{lemma:max_power_of_T_in_betapoly} is deferred to \appendixref{appendix:proof_of_max_power_of_T_in_betapoly}.

As defined in \eqref{eq:regret_factor}, 
\(
    \beta_{K} = \lim_{T\to \infty} \beta_{T,K}
\).
By \eqref{eq:appendix_beta_def}, 
\(
    \beta_{K} = \argmaxroot{\lim_{T \to \infty} \betapolyfunction{T}{K}}
\).
With the established limit, we proceed to establish \theoremref{theorem:regret_factor}.
\begin{proof}[Proof of \theoremref{theorem:regret_factor}]
We prove  
\begin{equation} \label{eq:proof_of_thm_B_limit}
    \lim_{T\to \infty} \betapoly{T}{K}{x} = \hermp_K(x).
\end{equation}
By \lemmaref{lemma:max_power_of_T_in_betapoly}, 
the maximal power of $T$ in $\betapolyfunction{T}{K}$ is at most $\frac{K}{2}$.
The coefficient of $T^{\frac{K}{2}}$ in \eqref{eq:expnasion_of_betapoly}
contains the summand of $\widetilde{c}_{T,m}$  with $m = 2n$. 
Assuming that the coefficient of $T^{\frac{K}{2}}$ does not vanish, we have 
\begin{align*}
    \lim_{T\to \infty} \betapoly{T}{K}{x} 
    =
    & \sum_{n=0}^{\lfloor {K}/{2} \rfloor} x^{K-2n}\,\binom{K}{2n} \sum_{d=0}^{2n} (-1)^{d} \binom{2n}{d}\, \stirling{d}{d-n}  .
\end{align*}
As for $d < n$ it holds that $\stirling{d}{d-n} = 0$, we denote $d=n+m$, and obtain
\begin{align*}
    \lim_{T\to \infty} \betapoly{T}{K}{x} 
    &= \;
    \sum_{n=0}^{\lfloor {K}/{2} \rfloor} x^{K-2n}\,\binom{K}{2n} 
        \sum_{m=0}^{n} (-1)^{n+m} \binom{2n}{n+m}\, \stirling{n+m}{m}  
    \\
    &= \;
      K! \sum_{n=0}^{\lfloor {K}/{2} \rfloor} \frac{x^{K-2n} (-1)^{n}}{(K-2n)!(2n)!} 
        \, \sum_{m=0}^{n} (-1)^{m} \binom{2n}{n+m}\, \stirling{n+m}{m}.
\end{align*}
By \citet[][Theorem~1]{gould2015stirling},
\[
    \sum_{m=0}^{n} (-1)^{m} \binom{2n}{n+m}\, \stirling{n+m}{m} = \frac{(2n)!}{n! 2^n}.
\]
Therefore,
\begin{equation} \label{eq:last_eq_in_hermite_proof}
    \lim_{T\to \infty} \betapoly{T}{K}{x} 
    =  K!\,\sum_{n=0}^{\lfloor K/2 \rfloor} \frac{(-1)^{n}}{n!(K-2n)!}\, \frac{x^{K-2n}}{2^{n}}.
\end{equation}
The RHS of \eqref{eq:last_eq_in_hermite_proof} 
is precisely the $K$-th probabilist’s Hermite polynomial \(\hermp_K(x)\)
\citep[e.g.][Eq.~3]{patarroyo2020digressionhermitepolynomials}.
\end{proof}

\subsubsection{Proof of Lemma \ref{lemma:max_power_of_T_in_betapoly}}
\label{appendix:proof_of_max_power_of_T_in_betapoly}
We make use of the following two auxiliary results.

\begin{lemma}[Alternating Binomial Sum of Polynomials]
    \label{lemma:alternating_binomial_polynomial}
    Let \( P(x) \) be a polynomial of degree less than \( n \). 
    Then, the alternating sum of binomial coefficients weighted by \( P(j) \) satisfies the identity:
    \begin{equation}
        \sum_{j=0}^{n} (-1)^j \binom{n}{j} P(j) = 0.
    \end{equation}
\end{lemma}
\lemmaref{lemma:alternating_binomial_polynomial} follows from the theory of finite differences \cite[see][Ch.~2]{graham1994concrete}.

\begin{lemma}[Polynomiality of Stirling Numbers of the First Kind] 
    \label{lemma:polynomiality_of_stirling_numbers_of_the_first_kind}
    For a fixed \( n \in \Nfield \), there exists a polynomial \( G_n(d) \) of degree \( \leq 2n \) such that  
    \[
        \unstirling{d}{d-n} = G_n(d).
    \]
\end{lemma}
\lemmaref{lemma:polynomiality_of_stirling_numbers_of_the_first_kind} 
follows from the combinatorial interpretation of Stirling numbers and their polynomial nature
as established by \cite{GESSEL197824}.

\begin{proof}[Proof of \lemmaref{lemma:max_power_of_T_in_betapoly}]
    Let $0 \leq m \leq K$. 
    It holds that 
    \begin{align*}
    \widetilde{c}_{T,m} 
    \stackrel{(a)}{\eqdef}
    & \sum_{n=0}^{m} T^{\frac{m}{2} - n} \sum_{d=0}^{m} (-1)^{d} \binom{m}{d}\, \stirling{d}{d-n}
    \\
    \stackrel{(b)}{=}
    & \sum_{n=0}^{m} (-1)^{n} \, T^{\frac{m}{2} - n}  \sum_{d=0}^{m} (-1)^{d} \binom{m}{d}\, \unstirling{d}{d-n}
    \\
    \stackrel{(c)}{=}
    &\sum_{n=0}^{m} (-1)^n \, T^{\frac{m}{2} - n} \sum_{d=0}^{m} (-1)^{d} \binom{m}{d}\, G_n(d) ,
    \end{align*}
    where the steps are justified as follows: 
    \begin{enumerate}[label={$(\alph*)$}]
        \item Follows from the definition of the term $\widetilde{c}_{T,m}$ in
        \eqref{eq:expnasion_of_betapoly_coeff}.
        \item Follows from \definitionref{def:stirling_numbers_of_the_first_kind}, 
        where the signed Stirling numbers of the first kind are defined by
        \[
            \stirling{n}{k} = (-1)^{n-k} \unstirling{n}{k}.
        \]
        \item 
        Follows from \lemmaref{lemma:polynomiality_of_stirling_numbers_of_the_first_kind}, 
        which states that each term $\unstirling{d}{d-n}$ can be written as a polynomial $G_n(d)$ of degree at most $2n$.
    \end{enumerate}  
    Thus, if $\frac{m}{2} > n$, then $\text{deg}\left( G_n(d) \right) < m$,
    and by \lemmaref{lemma:alternating_binomial_polynomial} we obtain
    \begin{equation*}
         \sum_{d=0}^{m} (-1)^{d} \binom{m}{d}\, G_n(d) = 0.
    \end{equation*}
\end{proof}

\subsection{Proof of Remark \ref{remark:opt_loss_grows_to_infinity_for_fixed_T}}
\label{appendix:proof_remark_opt_loss_grows_to_infinity_for_fixed_T}

\begin{lemma} \label{lemma:simulating_bookmaker_m_times}
    For any $T, K, m \in \Nplus$ it holds that $\optimalloss{mT}{K} \leq m \cdot \optimalloss{T}{K}$.
\end{lemma}

\begin{proof}[Proof of Lemma~\ref{lemma:simulating_bookmaker_m_times}]
Fix \(T, K, m \in \Nplus\), and let \(\bm^T\) be an optimal bookmaker whose loss is \(L^{\star}_{T,K}\).  
Construct a bookmaker \(\tilde\bm\) for horizon \(mT\) by applying \(\bm^T\) independently to each of the \(m\) disjoint blocks of \(T\) rounds.  
That is, partition the gambler’s sequence \(q_{1},\dots,q_{mT}\) into \(m\) segments and apply \(\bm^T\) to each, starting from a zero state.  
The total loss of \(\tilde\bm\) is at most \(m \cdot L^{\star}_{T,K}\),  
and thus \(L^{\star}_{mT,K} \leq m \cdot L^{\star}_{T,K}\), by the optimality of the loss \( L^{\star}_{mT,K} \).
\end{proof}

\begin{proof}[Proof of \remarkref{remark:opt_loss_grows_to_infinity_for_fixed_T}]
Assume, towards a contradiction, there exists \( \widehat{T} \in \Nplus \) and \(C \in \Rfield \)
such that 
\[
    \lim_{K \to \infty} \frac{R_{\widehat{T},K}}{\sqrt{\widehat{T}}} = C.
\]
Since adding an outcome to the game can only increase the bookmaker's optimal loss,
\begin{align}
    \label{eq:false_assumption_on_bound_of_optimal_loss}
    \forall K \in \Nplus \quad 
    \optimalloss{\widehat{T}}{K} \leq \widehat{T} + C \sqrt{\widehat{T}}
\end{align}
\theoremref{theorem:regret_factor} states that
\begin{equation} \label{eq:lim_H_eq_proof_of_remark_1}
      \forall K \in \Nplus \quad \lim_{T \to \infty} \beta_{T,K} = \argmaxroot{\hermp_K}.
\end{equation}
Lower bound on the RHS of \eqref{eq:lim_H_eq_proof_of_remark_1} 
\citep[e.g.,][]{krasikov2004bounds} 
implies there exist \( m \in \Nplus \) for which 
\begin{equation} \label{eq:lower_bound_on_the_redundancy}
    \forall K \in \Nplus \quad \beta_{m \widehat{T}, \, K} \geq \sqrt{K}
\end{equation}
Define 
\begin{equation} \label{eq:def_of_hat_K_proof_of_remark_1}
    \widehat{K} \eqdef \left\lceil  m {C}^{2} \right\rceil + 1,
\end{equation}
and consider the optimal bookmaking loss in a game with $\widehat{K}$ outcomes and $m\widehat{T}$ rounds:
\begin{align*}
    \optimalloss{m\widehat{T}}{\widehat{K}} 
    &\stackrel{(a)}{\leq} \;
    m \cdot \optimalloss{\widehat{T}}{\widehat{K}}
    \\
    &\stackrel{(b)}{\leq} \;
    m \cdot \left(
        \widehat{T} + C \sqrt{\widehat{T}}
    \right) 
    \\ 
    &= \;
    m \widehat{T} + \sqrt{m {C}^{2} } \sqrt{m \widehat{T}} 
    \\
    &\stackrel{(c)}{<} \;
    m \widehat{T} + \sqrt{\widehat{K}} \sqrt{m \widehat{T}}
    \\
    &\stackrel{(d)}{\leq} \;
    m \widehat{T} + \beta_{m\widehat{T}, \widehat{K}} \sqrt{m \widehat{T}}
    \\
    &=   \;
    \optimalloss{m\widehat{T}}{\widehat{K}},
\end{align*}
where $(a)$ follows from \lemmaref{lemma:simulating_bookmaker_m_times}, stated above;
$(b)$ follows from the upper bound on $\optimalloss{\widehat{T}}{\widehat{K}}$ in \eqref{eq:false_assumption_on_bound_of_optimal_loss};
$(c)$ follows from our definition of $\widehat{K}$ in \eqref{eq:def_of_hat_K_proof_of_remark_1};
and, $(d)$ follows from the lower bound on the regret factor $ \beta_{m\widehat{T}, \widehat{K}}$ in \eqref{eq:lower_bound_on_the_redundancy}.
This yields a contradiction, completing the argument.
\end{proof}

\subsection{Proof of Remark \ref{remark:all_polys_are_denom}}
\label{appendix:proof_of_all_polys_are_denom}

\begin{proof}[Proof of \remarkref{remark:all_polys_are_denom}]
We prove that evaluating $\denomfunction{H}{K}$ at $\univec{x}{K} - s$, 
as defined in \eqref{eq:remark_def_of_biaseddenom}, 
yields the polynomial expression in \eqref{eq:biased_denominator_function}.
\begin{align*}
    \denom{ H }{K}{\univec{ x }{K} - s} 
    &\stackrel{(a)}{=} 
    \sum_{j=0}^{K} (-1)^{K-j} \fallingfact{ H }{K-j} \esp{j}{\univec{ x }{K} - s}\\
    &\stackrel{(b)}{=} 
    \sum_{j=0}^{K} (-1)^{K-j} \fallingfact{ H }{K-j} 
        \sum_{i=0}^{j} (-1)^{i} \esp{i}{s} \binom{K - i}{j - i} x^{j - i}\\
    &= \sum_{j=0}^{K} \sum_{i=0}^{j} 
        (-1)^{K-(j-i)} \fallingfact{ H }{K-j} \esp{i}{s} \binom{K - i}{j - i} x^{j - i}\\ 
    &\stackrel{(c)}= 
    \sum_{j=0}^{K} \sum_{n=0}^{j} 
        (-1)^{K-n} \fallingfact{ H }{K-j} \esp{j - n}{s} \binom{K - j + n}{n} x^{n}\\ 
    &= \sum_{n=0}^{K} (-1)^{K-n} \left( \sum_{j=n}^{K} 
         \fallingfact{ H }{K-j} \esp{j - n}{s} \binom{K - (j - n)}{n} 
        \right) x^{n}\\ 
    &\stackrel{(d)}{=} \sum_{m=0}^{K} (-1)^{m} \left( \sum_{j=K-m}^{K} 
        \fallingfact{ H }{K-j} \esp{j - (K - m)}{s} \binom{K - j + K - m}{K-m} 
        \right) x^{K-m} \\ 
    &\stackrel{(e)}{=} \sum_{m=0}^{K} 
        \left(  (-1)^{m} \sum_{n=0}^{m} 
        \fallingfact{ H }{m-n} \binom{K - n}{K-m} \esp{n}{s} 
        \right) x^{K-m},
\end{align*}
where: $(a)$ uses the definition of $\denomfunction{H}{K}$ in \eqref{eq:def_of_ppoly_explicit}; 
$(b)$ follows from \lemmaref{lemma:elementary_symmetric_polynomial_with_shift};
$(c)$ substitutes $n = j - i$;
$(d)$ by the substitution $m = K - n$;
and $(e)$ via the change of variable $n = j - K + m$.
Accordingly, we obtain the expression stated in \eqref{eq:biased_denominator_function}.
\end{proof}

\subsection{Proof of Remark \ref{remark:algo_complexity} (Algorithm \ref{algo:optimal_bookmaking} Computational Complexity)}
\label{appendix:complexity_analysis}

In this section, we analyze the computational complexity of \algorithmref{algo:optimal_bookmaking}.

\begin{remark}  \label{remark:esp_fft}
    All the elementary symmetric polynomials of a vector in $\Rdim{K}$ 
    can be computed in \( O(K \log K) \) 
    time using FFT-based polynomial multiplication techniques 
    \citep[][Thm.~7.1]{harvey2021integer,ben2021elliptic}.
    We refer to this method as the \algfont{FFT-ESP} algorithm.
\end{remark}

We first address the efficient computation of partial elementary symmetric polynomials, 
handled by \algorithmref{algo:partial_esps}.
\begin{algorithm2e}[htbp]
\caption{Partial Elementary Symmetric Polynomials}
\label{algo:partial_esps}
\SetKw{Initialization}{Initialization:}
\SetKw{output}{output}
\SetNoFillComment
\DontPrintSemicolon
\LinesNumbered
\KwIn{$v \in \Rdim{K}$}
\KwOut{$A = \left[ {a}_1 \cdots {a}_K  \right] \in \Rdim{K \times K}$ where $A_{k,m} = \esp{m-1}{v\minind{k}}$}
\Initialization{$A \gets \zerovec{}^{K \times K},\, s \gets \zerovec{K+1}$}

$\esp{0}{v}, \ldots, \esp{K}{v} \gets $ \algfont{FFT-ESP}$(v)$

${a}_1 \gets \mathbf{1}_K$

\For{$m = 1:K-1$}{

    $ {a}_{m+1} \gets \univec{\esp{m}{v}}{K} - v \odot {a}_m$
}
    
\output $A$
\end{algorithm2e}
\begin{lemma} \label{lemma:pesp}
    \algorithmref{algo:partial_esps}, hereafter referred to as \algfont{PESP}, runs in \( O(K^2) \) time and, 
    given a vector \( v \in \Rdim{K} \), returns a matrix \( A \in \Rdim{K \times K} \) such that  
    \[
        (A)_{k,m} = \esp{m-1}{v\minind{k}}
    \]
    for all \( k, m \in [K] \).
\end{lemma}

\begin{proof}[Proof of \lemmaref{lemma:pesp}]
Let $K \in \Nplus, k \in [K]$ and $v \in \Rdim{K}$.
By \lemmaref{lemma:elementary_symmetric_polynomial_recurrence_relation}, 
for every $m \in [K]$
\[
    \esp{m}{v} = v(k) \cdot \esp{m-1}{ v\minind{k} } + \esp{m}{ v\minind{k} }.
\]
Hence, for every $m \in [K-1]$,
\begin{equation}
    \label{eq:calc_of_pesp_by_esp}
    \esp{m}{ v\minind{k} } = \esp{m}{v} - v(k) \cdot \esp{m-1}{ v\minind{k} },
\end{equation}
with the base case $\esp{0}{ v\minind{k}} = 1$.
Thus, for every $k \in [K]$,
\[
    A_{k,1} = {a}_1(k) = 1 = \esp{0}{v\minind{k}},
\]
and by induction on $m-1$, for every $m \in [2:K]$,
\begin{align*}    
    A_{k,m}
    = \esp{m-1}{v} - v(k) \cdot {a}_{m-1}(k)
    = \esp{m-1}{v} - v(k) \cdot \esp{m-2}{v\minind{k}}
    = \esp{m-1}{v\minind{k}}
\end{align*}
In terms of complexity, \algfont{FFT-ESP} executes once in \( O(K \log K) \) time 
(\remarkref{remark:esp_fft}), 
followed by \( K \) iterations where each of the \( K \) entries performs \( O(1) \) operations. 
This results in a total complexity of \( O(K^2) \).
\end{proof}

\begin{proof}[Proof of \remarkref{remark:algo_complexity}]
To analyze the per-time-step complexity, 
consider an arbitrary round $t \in [2,T]$ with associated state vector $s \in \Rdim{K}$, 
and let $H = T - t + 1$ denote the remaining horizon.
The operations executed at round $t$ are as follows, along with their per-step time complexity.
\begin{enumerate}
    \item \emph{State and Residual loss Vectors Update} (Lines \ref{line:update_the_state_vector} and \ref{line:opt_res_loss_vector_update}): 
    element-wise vector operation, runs in \(O(K)\) time.

    \item \emph{Opportunistic Loss update}  (Line \ref{line:update_the_loss}):
    \begin{enumerate}
        \item \emph{Computing the polynomial \( \biaseddenomfunction{H}{K} \) \eqref{eq:biased_denominator_function}}:
        This process consists of the following steps:
        
        \begin{itemize}
            \item Computing the binomial coefficients 
            \[ 
                \binom{K-n}{m-n}  \quad \text{for all} \quad 0 \leq m, n \leq K .
            \]
            These coefficients can be precomputed before the algorithm begins. Using Pascal’s triangle, they can be computed in \( O(K^2) \) time.
    
            \item Computing the elementary symmetric polynomials
            using \algfont{FFT-ESP} in \( O(K \log K) \) (\remarkref{remark:esp_fft}).
    
            \item Computing the falling factorials 
            \[ 
                \fallingfact{H}{0}, \ldots, \fallingfact{H}{K}.
            \]
            These are computable in \( O(K) \) time directly via \definitionref{def:falling_factorial}.            
        \end{itemize}
    
        Once all components are precomputed in \( O(K^2) \), the polynomial is obtained by iterating over all \( 0 \leq n \leq m \leq K \), requiring an additional \( O(K^2) \) operations.

        \item 
        \emph{Root Finding}:
        Since \(\biaseddenomfunction{H}{s}\) is a degree-\(K\) polynomial, 
        its largest root can be computed numerically to precision \(\varepsilon\) in \(O(K \log(1/\varepsilon))\) time
        using standard root-finding methods,
        such as Newton's Method or Laguerre’s method.
        By applying the upper bound from \theoremref{theorem:regret_factor}, 
        we can efficiently compute a good initial point in constant time, 
        where the number of required iterations is only \( O(\log \log(1/\varepsilon)) \).
    \end{enumerate}

    \item \emph{Optimal Odds Update} (Line \ref{line:optimal_odds_update}):  
    Using \algfont{PESP}, we compute all the necessary ESPs for the polynomials  
    \[
        \denomfunction{H-1}{K-1}(v\minind{1}),\, \ldots, \denomfunction{H-1}{K-1}(v\minind{K})
    \]
    in \( O(K^2) \). The falling factorials required for these computations are obtained in \( O(K) \).  
    Each polynomial, for \( k \in [K] \), involves \( K+1 \) iterations, each requiring \( O(1) \) operations,  
    resulting in an overall complexity of \( O(K^2) \). Finally, computing the quotient for each term in the sum  
    incurs an additional \( O(K) \) cost.
\end{enumerate}
Thus, the overall complexity per time step is \( O(K^2) \), 
leading to a total complexity of \( O(TK^2) \) over \( T \) time steps.
\end{proof}

\subsection{Proof of Remark \ref{remark:prev_algo_comparison}}
\label{appendix:prev_algo_comparison}
\citet{bhatt2025optimal} consider the binary case (\(K = 2\)) and design two algorithms for the bookmaker, 
depending on the gambler's behavior. 
Against a decisive gambler, they propose the \emph{Optimal Strategy For Decisive Gamblers} algorithm (\algfont{ODG}), 
which coincides with \algorithmref{algo:optimal_bookmaking}. 
Against a non-decisive gambler, who may place continuous bets \( q_t \in [0,1] \),
they construct a mixture of \algfont{ODG} strategies, 
based on the observed betting sequence.
In particular, let \( r_t^{\text{\algfont{ODG}}}(x^{t-1}) \) 
denote the odds produced by \algfont{ODG} algorithm for some input \( x^{t-1} \in \{0,1\}^t \). 
Their idea is to view \( q_t \in [0,1] \) as the expected value of a binary random variable \( X_t \sim \text{Ber}(q_t) \). 
The strategy for continuous bets \citep[Eq.~25]{bhatt2025optimal} is
\begin{equation} \label{eq:bhaat_et_al_strategy}
    \bar{r}_t(q^{t-1}) = \mathbb{E} \left[ r_t^{\text{\algfont{ODG}}}(X^{t-1}) \right],
\end{equation}
where the expected value is taken with respect to the sequence of independent random variables \( X_i \sim \text{Ber}(q_i) \).

As computing the odds according to \eqref{eq:bhaat_et_al_strategy} can take exponential time,
they propose to evaluate it via a Monte Carlo simulation 
and provide \emph{Monte Carlo Based Efficient Strategy} Algorithm (referred to as \algfont{MC}).
At each round \( t \), the algorithm runs \( N \) parallel simulations of \algfont{ODG},
and the strategy is computed as the average over these \( N \) independent runs.
The computational complexity of \algfont{MC} is \(\mathcal{O}(NT)\), 
where \(N\) denotes the number of samples per round and \(T\) is the total number of betting rounds.
The computational efficiency of \algfont{MC} hinges on the number of samples required to achieve a desired approximation accuracy. 
Specifically, to ensure an additive error tolerance of \(\varepsilon\) with probability at least \(1 - \delta\), 
the required number of samples must satisfy:
\[
    N \geq \frac{T}{2\varepsilon^2} \log \left( \frac{2T}{\delta} \right).
\]
Substituting this bound on \(N\) yields an overall complexity of
\(
    O \left( T^{3+2\alpha} \log \left( \frac{T}{\delta} \right) \right).
\)

\begin{proof}[Proof of \remarkref{remark:prev_algo_comparison}]
We show there exist \(T \in \Nplus\), \(t \in \nset{T}\), and \(s \in \Rdim{2}\) 
for which \eqref{eq:bhaat_et_al_strategy} does not output the same odds as \algorithmref{algo:optimal_bookmaking}
(\algfont{OOPT}).
Let $T=4$.
At the first round, 
\[
    r_1^{\text{\algfont{OOPT}}} = \bar{r}_1 = (0.5, 0.5).
\]
Hence, after any gambler's action 
\(
    q_1 = \left( q_1(1),\, q_1(2) \right),
\)
both algorithms have the same state vector
\[
    s = \left( 2q_1(1) ,\, 2q_1(2) \right).
\]
It holds that 
\begin{align*}
    \bar{r}_2\left( \left(q_1(1)\right)  \right) 
    &= \mathbb{P}\left( \; (1) \; \mid \; \left(q_1(1)\right) \; \right) \cdot r_2^{\text{\algfont{ODG}}}\left((1)\right)
        + \mathbb{P}\left( \; (2) \; \mid \; \left(q_1(1)\right) \; \right) \cdot r_2^{\text{\algfont{ODG}}}\left((2)\right)
    \\ 
    &= q_1(1) \cdot \left(\frac{2}{3}, \frac{1}{3} \right) + q_1(2) \cdot \left(\frac{1}{3}, \frac{2}{3} \right) \\
    &= \frac{1}{3} \left( 2 \cdot q_1(1) + q_1(2),\; q_1(1) + 2 \cdot q_1(2) \right).
\end{align*}
At $t=2$, by \theoremref{theorem:optimal_bookmaking_algorithm}, 
\(
    \optimalloss{3}{2}(s) = \argmaxroot{\biaseddenomfunction{3}{s}}
\)
where 
\begin{equation*}
    \biaseddenomfunction{3}{s}(x) = x^{2} - \left(6 + 2q_1(1) + 2q_1(2) \right)x + \left( 6 + 6q_1(1) + 6q_1(2) + 4q_1(1) q_1(2) \right).
\end{equation*}
Choosing \(q_1 = (0.4, 0.6)\) gives
\(
    \biaseddenomfunction{3}{s}(x) = x^2 - 8x + 12.96,
\)
and the optimal loss is \(5.7435\). 
The residual loss is then
\[
    v \approx (5.7435 - 0.8, 5.7435 - 1.2) = (4.9435, 4.5435). 
\]
Using \eqref{eq:optimal_odd_thm}, we have that 
\[
    r_2^{\text{\algfont{OOPT}}} \approx (0.4635, 0.5364) \neq (0.4666, 0.5333) \approx \bar{r}_2\left( \left( q_1(1) \right) \right).
\]
\end{proof}

\subsection{Proof of Remark \ref{remark:approximate_root_finding} (The Effect of \teq{$\varepsilon$}-Approximate Root-Finding)}
\label{appendix:oracle_optimal_bookmaking}
In this section, 
we analyze a variant of \algorithmref{algo:optimal_bookmaking} that replaces exact root-finding with an $\varepsilon$-approximate oracle. 
The underlying principle of the algorithm is to obtain an upper approximation to the optimal opportunistic loss 
such that the resulting "residual loss vector" $v$ is still achievable.

Let $\varepsilon > 0$, and assume access to an oracle \(\mathbf{O}_\varepsilon\) that, 
given a univariate polynomial $\mathcal{P}$:
\begin{itemize}
    \item \emph{If $\mathcal{P}$ has real roots:}  returns a value $\rho^\varepsilon \in \Rfield$, such that  
    \begin{equation}
        \label{eq:return_of_the_oracle}
        \left|  \rho^{\varepsilon} - \rho \right| < \varepsilon,
    \end{equation}
    where $\rho$ is $\mathcal{P}$'s maximal real root.

    \item \emph{Otherwise:} returns $\infty$.
\end{itemize}
We present \algorithmref{algo:oracle_optimal_bookmaking}, 
where modified lines of \algorithmref{algo:optimal_bookmaking} are marked with $*$.

\begin{algorithm2e}
\caption{Opportunistic Bookmaking Algorithm with Maximal Real Root Oracle}
\label{algo:oracle_optimal_bookmaking}
\SetKw{Initialization}{Initialization:}
\SetKw{output}{output}
\SetNoFillComment
\DontPrintSemicolon
\LinesNotNumbered

\nlset{$*$} \KwIn{$K$, $T$, $\varepsilon$ (precision), $\mathbf{O}_\varepsilon$, $\seq{q}{T-1}$ (bets obtained sequentially)} 
\KwOut{$r_1,\dots,r_T$ (outputs $\seqelem{r}{t}$ after observing $\seq{q}{t-1}$)}
\nlset{$*$} \Initialization{$s \gets \univec{\varepsilon}{K}, \quad L \gets \mathbf{O}_{\varepsilon}(\biaseddenomfunction{T}{s})$}

\output $\seqelem{r}{1}\gets \univec{\frac{1}{K}}{K}$ 

\For{$t = 2:T$}{
    
    $s \gets s + \seqelem{q}{t-1} \oslash \seqelem{r}{t-1}$  
    \tcp*{update the state vector with additional $\univec{\varepsilon}{K}$} 

    \If
    {$\seqelem{q}{t-1} \not\in \stdbasis{K}$}
    {
    \nlset{$*$}  $s \gets s + \univec{\varepsilon}{K}$
    
    \nlset{$*$}  $L \gets \mathbf{O}_\varepsilon(\biaseddenomfunction{T-t+1}{s})$
        \tcp*{call the oracle instead of calculating $\argmaxroot{\biaseddenomfunction{T-t+1}{s}}$}
    }

    $v \gets \left( \univec{L}{K} - s \right)$

    \For{$k=1:K$}{
        $r(k) \gets  \denom{T-t}{K-1}{v\minind{k}}$
    }
    
    \output $r_t \gets r/ \| r\|_1$
}
\end{algorithm2e}

\begin{proof}[Proof of \remarkref{remark:approximate_root_finding}]
We show that for every non-decisive gambler action, 
the bookmaker’s loss exceeds the optimal opportunistic bookmaking loss by at most $2\varepsilon$.

Let $H \in \Nplus$ and $s \in \Rdim{K}$.
Denote by $L_{H,K}^{\varepsilon}(s)$ a fixed response of $\mathbf{O}_\varepsilon(\biaseddenomfunction{H}{s})$.
By \theoremref{theorem:optimal_bookmaking_algorithm}, 
$\optimalloss{H}{K}(s)$ is equal to the largest real root of the polynomial \(\biaseddenomfunction{H}{s}\);
thus, from \eqref{eq:return_of_the_oracle}, it follows that
\begin{equation}
    \label{eq:return_loss_of_oracle_root}
    \optimalloss{H}{K}(s) - \varepsilon < L_{H,K}^{\varepsilon}(s) < \optimalloss{H}{K}(s) + \varepsilon.
\end{equation}
By \definitionref{def:value_function} of the value function,
\begin{equation}
    \label{eq:return_potential_of_oracle_root}
    \pot{H}{H}{s} - \varepsilon < L_{H,K}^{\varepsilon}(s) < \pot{H}{H}{s} + \varepsilon.
\end{equation}
Observe a fixed value of 
\(
    L_{H,K}^{\varepsilon}(s + \univec{\varepsilon}{K}):
\)
By \eqref{eq:return_potential_of_oracle_root}, 
\begin{equation*}
    \potfunction{H}{K}(s + \univec{\varepsilon}{K}) - \varepsilon 
    < L_{H,K}^{\varepsilon}(s + \univec{\varepsilon}{K}) 
    < \potfunction{H}{K}(s + \univec{\varepsilon}{K}) + \varepsilon.
\end{equation*}
By uniform translation property of the value function (\lemmaref{lemma:uniform_translation_property}),
\begin{equation*}
    \left( \potfunction{H}{K}(s) + \varepsilon \right) - \varepsilon 
    < L_{H,K}^{\varepsilon}(s + \univec{\varepsilon}{K}) 
    < \left(\potfunction{H}{K}(s) + \varepsilon \right) + \varepsilon,
\end{equation*} 
i.e.,
\begin{equation*}
    \potfunction{H}{K}(s)  
    < L_{H,K}^{\varepsilon}(s + \univec{\varepsilon}{K}) 
    < \potfunction{H}{K}(s) + 2\varepsilon.
\end{equation*} 
By continuity and coordinate-wise strict monotonicity of the value functions 
(\lemmaref{lemma:continuity_of_the_value_function,lemma:cw_strict_m}), 
and since the value of $L_{H,K}^{\varepsilon}(s + \univec{\varepsilon}{K})$ is bounded,
there exists a vector $b \in \Rdim{K}$ such that $0 \prec b \prec \univec{2\varepsilon}{K}$ and 
\begin{equation}
    \label{eq:oracle_return_alternative_notation}
    L_{H,K}^{\varepsilon}(s + \univec{\varepsilon}{K}) = \optimalloss{H}{K}(s + b).
\end{equation}

Denote by $\mathbf{O}_0$ the oracle who returns the precise value of the maximal real root (if it exists).
By \eqref{eq:oracle_return_alternative_notation}, 
an equivalent interpretation of \algorithmref{algo:oracle_optimal_bookmaking} is as follows: 
\begin{itemize}
    \item 
    Before round $t \in [T]$ with some initial state $s_t$, 
    we sample a vector 
    \begin{equation}
        \label{eq:sample_b_t_oracle_equivalence}
        b_t \in \{b \in \Rplus{K} \; : \; b \prec \univec{2\varepsilon}{K} \}
    \end{equation}
    according to some unknown probability distribution.
    We interpret $b_t(k)$ as a given random tax that we are forced to pay, if outcome $k$ materializes.

    \item
    In general, the state $s_t$ accounts for some given loss when choosing a bookmaking action 
    (see the motivation for the OOBL in \sectionref{subsec:problem_opportunistic}).
    Thus, our effective state to consider is $s_t + b_t$.

    \item 
    After observing $b_t$, we have an access to the oracle $\mathbf{O}_0$.
    By \theoremref{theorem:optimal_bookmaking_algorithm}, the optimal strategy  
    is to query $\mathbf{O}_0$ with the polynomial $\biaseddenomfunction{T-t+1}{s_t + b_t}$.
\end{itemize}
A direct result of coordinate-wise strict monotonicity property (\lemmaref{lemma:cw_strict_m}),
is that the worst random tax $b_t$, as in \eqref{eq:sample_b_t_oracle_equivalence} is $\univec{2\varepsilon}{K}$ 
(formally, infinitesimally close to $\univec{2\varepsilon}{K}$).
By uniform translation property, 
we obtain that the bookmaker’s loss exceeds the optimal opportunistic bookmaking loss by at most $2\varepsilon$.

\end{proof}

\section{Supplementary Technical Details}
\label{appendix:technical_proofs}

This appendix collects additional technical proofs.
\appendixref{appendix:technical_A} supports the preliminary material in \appendixref{appendix:preliminaries},
\appendixref{appendix:technical_B} corresponds to \appendixref{appendix:the_adversarial_game_dynamics},
\appendixref{appendix:technical_C}  provides supporting proofs for both \sectionref{sec:bpf} and \appendixref{appendix:value_vectors},
and \appendixref{appendix:technical_D} presents technical arguments for \appendixref{appendix:main_results} .

\subsection{Proofs for Appendix \ref{appendix:preliminaries} }
\label{appendix:technical_A}

\begin{proof}[Proof of \lemmaref{lemma:elementary_symmetric_polynomial_sum_identity}]
For $m=0$, the lemma holds trivially due to the fact that $\esp{0}{\cdot}=1$. 
For $0 < m \leq K$, by \definitionref{def:esp}, it holds that
    \begin{align*}
        \sum_{i=1}^{K} \esp{m}{x\minind{i} }  
        = \sum_{i=1}^{K} \sum_{\mathfrak{I} \in \binom{\nset{K} \setminus \{ i \}}{m}} \prod_{k \in \mathfrak{I}} \vecind{x}{k}.
    \end{align*}
    Every subset of indices \(\mathfrak{J} \in \binom{[K]}{m}\) appears in \( K - m \) different configurations of $i$ within the summation. This results from choosing any of the \( K - m \) indices that do not belong to \(\mathfrak{J}\) for the value of $i$.
\end{proof}

\begin{proof}[Proof of \lemmaref{lemma:elementary_symmetric_polynomial_with_shift}]
Substituting the vector \( \univec{t}{K} - x \) into the generating function for ESPs (see \remarkref{remark:esps_gen_func}), we obtain:
\begin{align}
    \sum_{n=0}^{K} \esp{n}{\univec{t}{K}-x} y^n 
    &= \prod_{k=1}^{K} \left(1 + \left(t - x(k)\right) y\right)    \label{eq:proof_elementary_symmetric_polynomial_with_shift}
    \\
    &=  \prod_{k=1}^{K} \left( (1 + t y) - x(k) y \right) \nonumber 
\end{align}
Expanding this product involves selecting for each \( k \) either the \( 1 + t y \) term or the \( -x(k) y \) term.
Specifically, all the terms in the product which involves \( i \) of the \( -x(k) y \) terms 
will contribute to the generating function the term
\begin{equation} \label{eq:appendix_proof_elementary_symmetric_polynomial_with_shift}
    (-1)^i \sigma_i(x) y^i  (1 + t y)^{K - i}.
\end{equation}
Thus,
\begin{align}
    \prod_{k=1}^{K} \left(1 + \left(t - x(k)\right) y\right) 
    &\stackrel{(a)}{=} \sum_{i=0}^{K} (-1)^i \sigma_i(x) y^i \cdot (1 + t y)^{K - i} \nonumber
    \\
    &\stackrel{(b)}{=}  \sum_{i=0}^{K} (-1)^i \sigma_i(x) y^i \sum_{j=0}^{K - i} \binom{K - i}{j} t^j y^j \nonumber
    \\
    &=  \sum_{i=0}^{K} \sum_{j=0}^{K - i} (-1)^i \sigma_i(x) \binom{K - i}{j} t^j y^{i + j} \nonumber
    \\
    &\stackrel{(c)}{=} 
    \sum_{n=0}^{K} \left( \sum_{i=0}^{n} (-1)^i \sigma_i(x) \binom{K - i}{n - i} t^{n - i} \right) y^n,
    \label{eq:proof_elementary_symmetric_polynomial_with_shift_last}
\end{align}
where
$(a)$ follows by \eqref{eq:appendix_proof_elementary_symmetric_polynomial_with_shift},
$(b)$ follows from the expansion of \( (1 + t y)^{K - i} \) using the binomial theorem,
and $(c)$ uses the variable substitution $n = i + j$.
By equating the coefficients of \( y^n \) on the LHS of \eqref{eq:proof_elementary_symmetric_polynomial_with_shift} 
with those in \eqref{eq:proof_elementary_symmetric_polynomial_with_shift_last}, we obtain:
\[
    \sigma_n(\univec{t}{K} - x) = \sum_{i=0}^{n} (-1)^i \sigma_i(x) \binom{K - i}{n - i} t^{n - i}
\] 
\end{proof}

\subsection{Proofs for Appendix \ref{appendix:the_adversarial_game_dynamics} }
\label{appendix:technical_B}

\begin{proof}[Proof of \lemmaref{lemma:uniform_translation_property}]
Let $s \in \Rdim{K}$ and $c \in \Rfield$. We prove by induction on $H$.
\begin{itemize}
    \item \emph{Base Case} (\(H=0\)):  
    It holds that 
    \begin{align*}
        \pot{0}{K}{s + \univec{c}{K}}
        \eqdef \;
        &\max_{k \in [K]} \left(s + \univec{c}{K} \right)(k) 
        \\
        = \;
        &\max_{k \in [K]} s(k) + c
        \\
        = \;
        &\pot{0}{K}{s} + c.
    \end{align*}
    \item \emph{Inductive step} (\( H \to H+1 \)):     
     \begin{align*}
        \pot{ H + 1}{K}{s + \univec{c}{K}}
        &\eqdef 
        \inf_{ r \in \simplex{K-1} } \; \max_{ q \in \simplex{K-1}} 
            \pot{ H }{K}{
                 s + \univec{c}{K} + q \oslash r 
             } 
        \\
        &\stackrel{(*)}{=} 
        \inf_{ r \in \simplex{K-1} } \; \max_{ q \in \simplex{K-1}} 
            \pot{ H }{K}{
                 s + q \oslash r 
             } + c
        \\
        &= 
        \pot{ H + 1}{K}{s} + c,
    \end{align*}
    where $(*)$ follows from the induction hypothesis.
\end{itemize}
\end{proof}

\begin{proof}[Proof of Lemma \ref{lemma:properties_of_corr}]
The correspondence $\corr_H$ is compact-valued by its definition in \eqref{eq:def_of_corr_H},
and is non-empty as for every \( s \in \mathbb{R}^K \), \( \frac{1}{K} \in \corr_H(s) \).

A correspondence is continuous if it is both upper and lower hemicontinuous (see \definitionref{def:continuity_of_a_correspondence}).
From the continuity of the maximum and minimum functions on $\mathbb{R}^K$, 
the function $\omega_{H}$, defined in \eqref{eq:def_of_omega_H_s}, is continuous. 
Hence for any $\bar{s} \in \Rdim{K}$ and $\delta > 0$, 
there exists an open neighborhood \(U\) of \(\bar{s}\) such that for all $s \in U$,
\(
    |\omega_H(s) - \omega_H(\bar{s})| < \delta.
\)

Let \(\bar{s} \in \Rdim{K}\) and let \(V \subset \Delta^{K-1}\) be an open set with \(\corr_H(\bar{s}) \subset V.\)
Since every \(r \in \corr_H(\bar{s})\) is an interior point of \(V\), for each there exists \(\varepsilon_r > 0\) such that \(B(r,\varepsilon_r) \subset V\). 
By compactness of \(\corr_H(\bar{s})\), a finite subcover yields a uniform margin \(\delta > 0\) ensuring that if
\(
    |\omega_H(s)-\omega_H(\bar{s})|<\delta,
\)
then for every \(r \in \Delta^{K-1} \) and $ k \in [K]$,  if
\(
    r(k) \geq \omega_H(s) 
\)
then
$r \in V$.
By continuity of \(\omega_H\), 
there exists an open neighborhood \(U\) of \(\bar{s}\), such that for every $s \in U,$
$\corr_H(s)\subset V$,
and thus $\corr_H$ is upper hemicontinuous.

Let \(\bar{s} \in \Rdim{K}\) and suppose \(V \subset \Delta^{K-1}\) is open with
\(
\corr_H(\bar{s})\cap V\neq\emptyset.
\)
Choose \(\bar{r}\in \corr_H(\bar{s})\cap V\). Since \(V\) is open, there exists \(\varepsilon>0\) with \(B(\bar{r},\varepsilon)\subset V\). 
\begin{itemize}
    \item 
    \emph{If \(\bar{r}\) is an interior point of \(\corr_H(\bar{s})\): }
   For all \(k \in [K]\), \(\bar{r}(k) > \omega_H(\bar{s})\). 
   By the continuity of \(\omega_H\), 
   there exists a neighborhood \(U\) of \(\bar{s}\) such that for all \(s\in U\) and $k \in [K]$,
   $\omega_H(s) < \bar{r}(k)$.
   Hence, \(\bar{r}\in \corr_H(s)\cap V\).

   \item 
   \emph{If \(\bar{r}\) lies on the boundary of \(\corr_H(\bar{s})\):}  
   Then for some coordinate, \(\bar{r}(k)=\omega_H(\bar{s})\). Since \(\bar{r}\) is in \(V\), we can select \(r'\in B(\bar{r},\varepsilon)\) with \(r'(k)>\omega_H(\bar{s})\) for all \(k\). Again, by continuity of \(\omega_H\), for \(s\) in a neighborhood \(U\) of \(\bar{s}\) we have
   $\omega_H(s) < r'(k)$ for every $k \in [K]$,
   so that \(r'\in \corr_H(s)\cap V\).
\end{itemize}
In either case, there exists an open neighborhood \(U\) of \(\bar{s}\) such that 
\(\corr_H(s)\cap V\neq\emptyset\) for all \(s\in U\),
and thus $\corr_H$ is lower hemicontinuous.

\end{proof}

\begin{proof}[Proof of \lemmaref{lemma:nash_last_round}]
Let $s \in \Rdim{K}$. 
By \definitionref{def:value_function}, $\pot{0}{K}{s} \eqdef \max_{k \in \nset{K}} s(k)$ and hence a convex function.
By \lemmaref{lemma:value_function_if_convex},
\[
    \pot{1}{K}{s} = \min_{r \in \corr_1(s)} \max_{k \in [K]} s(k)+ \frac{1}{r(k)}.
\]
Assume, towards a contradiction, 
that there exists an optimal bookmaker action $r \in \corr_1(s)$ that does not satisfy 
\eqref{eq:nash_last_round};
that is, $r$ achieves the optimal loss \( \opt\ell  \), yet there exists an index \( i \in \nset{K} \) such that  
\begin{equation}
    \label{eq:proof_nash_last_round_2_cont_assumption}
    \vecind{s}{i} + \frac{1}{\vecind{r}{i}} = \opt\ell - \kappa  \quad \text{for some } \kappa > 0.
\end{equation}
Assume $K \geq 2$ (for the case \( K = 1 \), the statement is trivially false).
It holds that  
\[
    \frac{K}{K-1} \vecind{r}{i} \; , \; \frac{K \kappa (\vecind{r}{i}^2) }{(K-1)(1 + \kappa \vecind{r}{i})} > 0.
\]
Therefore, there exists \( \varepsilon \in \Rpplus{} \) such that  
\[
    0 < \varepsilon < \min \left\{ 
            \frac{K}{K-1} \vecind{r}{i} \; , \;
            \frac{K \kappa (\vecind{r}{i}^2) }{(K-1)(1 + \kappa \vecind{r}{i})}
        \right\}.
\]
We construct a new distribution \( \hat{r} \in \simplex{K-1} \) as follows:  
\[
    \vecind{\hat{r}}{k} = \begin{cases}
        \vecind{r}{k} - \frac{K-1}{K} \varepsilon & \text{if } k = i, \\
        \vecind{r}{k} + \frac{1}{K} \varepsilon & \text{if } k \neq i.
    \end{cases} 
\]
\( \hat{r} \) is indeed in \( \simplex{K-1} \) as it satisfies:
\begin{itemize}
    \item \emph{Sum Constraint}: 
    \[
        \sum_{k=1}^K \vecind{\hat{r}}{k} 
        = \vecind{r}{i} - \frac{K-1}{K} \varepsilon + \sum_{k \neq i} \left(\vecind{r}{k} + \frac{1}{K} \varepsilon\right)
        = 1.
    \]   
    \item \emph{Positivity Constraint}:  
    For \( k = i \), since \( \varepsilon < \frac{K}{K-1} \vecind{r}{i} \) we obtain:  
    \[
        \vecind{\hat{r}}{i} 
        = \vecind{r}{i} - \frac{K-1}{K} \varepsilon     
        > \vecind{r}{i} - \frac{K-1}{K} \frac{K}{K-1} \vecind{r}{i}
        = 0.
    \]
    For \( k \in \nset{K} \setminus \{i\} \):  
    \[
        \vecind{\hat{r}}{k} = \vecind{r}{k} + \frac{1}{K} \varepsilon > \vecind{r}{k} > 0.
    \]
\end{itemize}
We show that acting with \( \hat{r} \) produces a loss that is \(< \opt\ell \).
\begin{itemize}
    \item \emph{For \( k = i \)}:  
        \begin{align}
            \label{eq:nash_last_round_2_a}
            \frac{1}{\vecind{\hat{r}}{i}} - \frac{1}{\vecind{r}{i}}    
            = \frac{\vecind{r}{i} - \vecind{\hat{r}}{i}}{\vecind{r}{i} \vecind{\hat{r}}{i}}
            = \frac{\frac{K-1}{K} \varepsilon}{\vecind{r}{i} (\vecind{r}{i} - \frac{K-1}{K} \varepsilon)}.
        \end{align}
        From the definition of \( \varepsilon \), we have:
        \begin{align}
            \label{eq:nash_last_round_2_b}
            \varepsilon < \frac{\kappa \vecind{r}{i}^2 K}{(K-1)(1 + \kappa \vecind{r}{i})} 
            \quad \implies
            \frac{\frac{K-1}{K} \varepsilon}{\vecind{r}{i} \left( \vecind{r}{i} - \frac{K-1}{K} \varepsilon \right)} < \kappa.
        \end{align}

        Thus, by \equationref{eq:proof_nash_last_round_2_cont_assumption,eq:nash_last_round_2_a,eq:nash_last_round_2_b}
        \begin{align*}
            \vecind{s}{i} + \frac{1}{\vecind{\hat{r}}{i}} 
            = \opt\ell - \kappa + \frac{1}{\vecind{\hat{r}}{i}} - \frac{1}{\vecind{r}{i}}
            < \opt\ell.
        \end{align*}

    \item \emph{For \( k \in \nset{K} \setminus \{i\} \)}:
    \[
            \frac{1}{\vecind{r}{k}} > \frac{1}{\vecind{r}{k} + \frac{1}{K} \varepsilon} \implies
            \vecind{s}{k} + \frac{1}{\vecind{r}{k}} > \vecind{s}{k} + \frac{1}{\vecind{\hat{r}}{k}} 
    \]
    Hence,
    \[
        \opt\ell = \max_{k \in \nset{K} \setminus \{i\}}  \vecind{s}{k} + \frac{1}{\vecind{r}{k}}
            > \max_{k \in \nset{K} \setminus \{i\}} \vecind{s}{k} + \frac{1}{\vecind{\hat{r}}{k}} .
    \]
\end{itemize}
Combining both cases, we obtain:
\[
    \max_{k \in \nset{K}} \left( \vecind{s}{k} + \frac{1}{\vecind{\hat{r}}{k}} \right) < \opt\ell.
\]
It follows that $r$ is suboptimal, contradicting the optimality assumption.

It remains to show that the optimal bookmaking action is unique.
Let $\optr, \hat{r} \in \corr_{1}(s)$ be two vectors that satisfy \eqref{eq:nash_last_round},
i.e.,
\begin{equation*}
    s(k) + \frac{1}{\optr(k)} = s(k) + \frac{1}{\hat{r}(k)} = \pot{1}{K}{s} \quad \forall k \in [K]. 
\end{equation*}
As the function $x \mapsto \frac{1}{x}$ is strictly decreasing in $\Rplus{}$,
for every $k \in [K]$, $\hat{r}(k) = \optr(k)$, 
and thus the optimal bookmaking action $\optr$ is unique.
\end{proof}

\subsection{Proofs and Technical Details for Section \ref{sec:bpf} and Appendix \ref{appendix:value_vectors}}
\label{appendix:technical_C}

\begin{proof}[Proof of \lemmaref{lemma:last_gamblers_bet}]
    Let $k \in [K]$ be the index for which $q_H = \basis{k}$.
    When $K=1$, the claim is immediate.
    Suppose $K \geq 2$, and assume, towards a contradiction, that some 
    $j \in [K] \setminus \{k\}$ maximizes the RHS of \eqref{eq:conclusion_from_nash_and_value_function}.
    By \lemmaref{lemma:continuity_of_the_value_function} (see \appendixref{appendix:the_adversarial_game_dynamics}),
    any optimal bookmaker's action $r_H$ assigns a positive probability mass $r_H(i) > 0$ for every $i \in [K]$.
    Thus, for the betting sequence $(q_1, \ldots, q_{H-1}, \basis{j}) \in (\stdbasis{K})^H$, 
    the resulting loss would exceed $\pot{H}{K}{s}$, a contradiction to the bookmaker's optimality assumption. 
\end{proof}

\begin{proof}[Proof of \lemmaref{lemma:recurrence_relation_or_the_denom}]
By \lemmaref{lemma:elementary_symmetric_polynomial_recurrence_relation}, we have for \(m\ge1\) 
\[
v(k)\,\esp{m-1}{v\minind{k}} + \esp{m}{v\minind{k}} = \esp{m}{v}.
\]
In particular, with \(m\) replaced by \(K-m\) (noting that \(K-m\ge1\) for \(m\le K-1\) ) we obtain
\begin{align}\label{eq:proof_Drecurrence_ESPrecKM}
v(k)\,\esp{K-1-m}{v\minind{k}} = \esp{K-m}{v} - \esp{K-m}{v\minind{k}}.    
\end{align}
By \eqref{eq:proof_Drecurrence_ESPrecKM}, we have  
\begin{align}\label{eq:proof_Drec_1}
    \vecind{v}{k}\cdot\denom{H}{K-1}{v\minind{k}}&= \sum_{m=0}^{K-1} (-1)^m\,\fallingfact{H}{m}\Bigl(\esp{K-m}{v}-\esp{K-m}{v\minind{k}}\Bigr).
\end{align}
We combine the second summand of \eqref{eq:proof_Drec_1} and $H \cdot \denom{ H-1 }{K-1}{ v\minind{k} }$ as follows:
\begin{align}\label{eq:proof_Drec_2}
    &H \cdot \denom{ H-1 }{K-1}{ v\minind{k} } + \sum_{m=0}^{K-1} (-1)^m\,\fallingfact{H}{m}\esp{K-m}{v\minind{k}}
    \nonumber\\
    &\stackrel{(a)}= \sum_{m=0}^{K-1} (-1)^m\,\fallingfact{H}{m+1} \,
\esp{K-1-m}{v\minind{k}} + \sum_{m=0}^{K-1} (-1)^m\,\fallingfact{H}{m}\esp{K-m}{v\minind{k}}\nonumber\\
&\stackrel{(b)}= \sum_{m=1}^{K} (-1)^{m-1}\,\fallingfact{H}{m}\,\esp{K-m}{v\minind{k}} + \sum_{m=0}^{K-1} (-1)^m\,\fallingfact{H}{m}\esp{K-m}{v\minind{k}}\nonumber\\
&\stackrel{(c)}= (-1)^{K-1}\,\fallingfact{H}{K}\,\esp{0}{v\minind{k}} + \esp{K}{v\minind{k}}\nonumber\\
&= (-1)^{K-1}\,\fallingfact{H}{K}\,\esp{0}{v\minind{k}},
\end{align}
where $(a)$ follows from the identity $H\fallingfact{(H-1)}{m} = \fallingfact{H}{m+1}$, 
$(b)$ follows by re-indexing with \(m+1\), and $(c)$ follows by noting that the two sums are telescoping.
We can now combine \eqref{eq:proof_Drec_1} and \eqref{eq:proof_Drec_2} to complete the proof
\begin{align*}
&\vecind{v}{k} \cdot \denom{ H }{K-1}{ v\minind{k} } - H \cdot \denom{ H-1 }{K-1}{ v\minind{k} } \\
= \,&\sum_{m=0}^{K-1} (-1)^m\,\fallingfact{H}{m}\esp{K-m}{v} + (-1)^{K}\,\fallingfact{H}{K}\,\esp{0}{v\minind{k}}\nonumber\\
=\, &\sum_{m=0}^{K} (-1)^m\,\fallingfact{H}{m}\esp{K-m}{v}\nonumber\\
=\, &\denom{ H }{K}{ v },
\end{align*}
where we use the fact that $\esp{0}{v\minind{k}} = \esp{0}{v}=1$. 
\end{proof}

\begin{proof}[Proof of \lemmaref{lemma:main_lemma_in_conditions}, case $K=1$]
When \(K=1\), the only possible state is \(v=(v(1))\), 
and, by \definitionref{def:bpf}, \(v\in\vvset{H,1}\) if and only if \(v(1)=H\).  Moreover,
\[
  \denom{H}{1}{v} = v(1)-H =0,
  \qquad
  \denom{H-1}{0}{\cdot}=1.
\]
Thus the unique action \(r(1)=1\) satisfies
\[
  \frac{\denom{H-1}{0}{v\minind{1}}}{\denom{H-1}{1}{v}}
  \;=\;
  \frac{1}{H - (H-1)}
  \;=\;1
  \;=\;r(1),
\]
verifying \lemmaref{lemma:proof_of_bp_1}, and
\[
  v(1)
  \;=\;
  H
  \;=\;
  \frac{H \cdot \denom{H-1}{0}{\cdot}}{\denom{H}{0}{\cdot}}
  \;=\;
  \frac{\num{H}{0}{v\minind{1}}}{\denom{H}{0}{v\minind{1}}},
\]
verifying \lemmaref{lemma:proof_of_bp_2}.  
Since \(\denom{H}{0}{v\minind{1}}=1>0\), \lemmaref{lemma:proof_of_bp_3} holds.  
\end{proof}

\begin{proof}[Proof of \lemmaref{lemma:der_of_denom}] Let $H,K \in \Nplus$.
    \begin{enumerate}
        \item 
        By \eqref{eq:main_results_def_of_denom_poly}, \(\denom{H}{K}{v}\) is a finite sum of elementary symmetric polynomials
        (each of which is a polynomial in \(v\)).  
        Hence \(\denomfunction{H}{K}\) is itself a multivariate polynomial, and therefore \(C^\infty\) on \(\Rdim{K}\).

  \item We prove by induction on $m$.

      $\bullet$ \; \emph{Base case} ($m=1$):
      Let $k \in [K]$ be some index. 
      By \lemmaref{lemma:recurrence_relation_or_the_denom}, 
      \begin{align*}
        \frac{\partial}{\partial v(k)} \denom{H}{K}{v} 
        &=
        \frac{\partial}{\partial v(k)} \left(
            \vecind{v}{k} \cdot \denom{ H }{K-1}{ v\minind{k} } - H \cdot \denom{ H-1 }{K-1}{v\minind{k}} 
            \right)
        \\
        &=
        \denom{H}{K-1}{v^{\setminus k}}
      \end{align*}
      
      $\bullet$ \; \emph{Inductive step} ($m \to m+1$):
      Let $m \in [K-1]$. Fix any subset of indices \(\mathfrak I=\{i_1,\dots,i_{m}\} \subset [K]\) of size $m$, and $k \in [K]\setminus\mathfrak{I}$.
      It holds that 
      \begin{align*}
        \frac{ \partial^m  \denomfunction{H}{K}}{\partial^m v_{\mathfrak{I}\cup \{ k \}} } (v)
        &=
        \frac{ \partial }{\partial v(k) } 
        \frac{ \partial^m  \denomfunction{H}{K}}{\partial^m v_{\mathfrak{I}} } (v)
        \\
        &\stackrel{(a)}{=} 
        \frac{ \partial }{\partial v(k) } \denom{H}{K-m}{v\minind{\mathfrak{I}}}
        \\
        &\stackrel{(b)}{=} 
        \frac{ \partial }{\partial v(k) } \left(
            v(k) \cdot  \denom{H}{K-(m+1)}{v\minind{\mathfrak{I}\cup \{k\}}} - H \cdot \denom{H-1}{K-(m+1)}{v\minind{\mathfrak{I}\cup \{k\}}}
            \right)
        \\
        &= 
         \denom{H}{K-(m+1)}{v\minind{\mathfrak{I}\cup \{k\}}}
      \end{align*}
      where $(a)$ is implied by the induction hypothesis, 
      and $(b)$ follows from \lemmaref{lemma:recurrence_relation_or_the_denom}.

  \item 
  Let $k \in [K]$ and $m > 1$. 
  By \lemmaref{lemma:partial_derivatives_of_denom_2} it holds that 
   \begin{equation*}
       \frac{\partial^m}{\partial v(k)^m} \denom{H}{K}{v} 
       = 
       \frac{\partial^{m-1}}{\partial v(k)^{m-1}} \denom{H}{K-1}{v\minind{k}}
       =
       0.
   \end{equation*}
  \end{enumerate}
\end{proof}

\begin{proof}[Proof of \lemmaref{lemma:lower_bound_for_value_vectors}]
    Let \( \seq{r}{H} \) be any sequence of actions chosen by the bookmaker \( \bm^{H} \). 
    Then, if the gambler bets on a single outcome $k \in [K]$ in each of the remaining $H$ rounds, 
    the total payout for outcome $k$ will be 
    \begin{equation*}  
        \sum_{h=1}^{H} \frac{1}{\vecind{\seqelem{r}{h}}{k}}  
        \geq \sum_{h=1}^{H} 1  
        = H.  
    \end{equation*}  
    In case \( K = 1 \), we have \( r_h(1) = 1 \) for all \( h \), and the inequality becomes an equality.
    Otherwise, setting $ r_h(k) = 1$ at any $h$ would cause the remaining components of the vector to become infinite.
\end{proof}

\subsubsection{Simplification of Equations \teq{\eqref{eq:proof_of_bp_inductive_step_part_2_def_A}}
and \teq{\eqref{eq:proof_of_bp_inductive_step_part_2_def_B}}}
\label{appendix:proof_of_bp_inductive_step_part_2_simplification}

For any $H \geq 1$ and $K \geq 2$, it holds that
\begin{align} 
    \sum_{\indneq{i}{\nset{K}}{k}} \denom{H}{K-2}{v\minind{\{i,k\}}} 
    &\stackrel{(a)}{=}  \sum_{\indneq{i}{\nset{K}}{k}}
        \sum_{m=0}^{K-2} (-1)^{m} \cdot \fallingfact{H}{m} \cdot \esp{K - 2 - m}{v\minind{\{i,k\}}} 
    \nonumber \\
    &=  \sum_{\indneq{i}{\nset{K}}{k}}
        \sum_{m=0}^{K-2} (-1)^{(m+1)-1} \cdot \fallingfact{H}{(m+1)-1} \cdot \esp{(K - 1) - (m + 1)}{v\minind{\{i,k\}}} 
    \nonumber \\
    &\stackrel{(b)}{=} 
    \sum_{\indneq{i}{\nset{K}}{k}}
        \sum_{n=1}^{K-1} (-1)^{n-1} \cdot \fallingfact{H}{n-1} \cdot \esp{(K - 1) - n}{v\minind{\{i,k\}}} 
    \nonumber  \\
    &=  
     \sum_{n=1}^{K-1} (-1)^{n-1} \cdot \fallingfact{H}{n-1} \cdot 
        \sum_{\indneq{i}{\nset{K}}{k}} \esp{(K - 1) - n}{v\minind{\{i,k\}}} 
    \nonumber \\
    &\stackrel{(c)}{=} 
     \sum_{n=1}^{K-1} (-1)^{n-1} \cdot n \cdot \fallingfact{H}{n-1} \cdot \esp{(K - 1) - n}{v\minind{k}},
    \label{eq:simplification_of_A_and_B_expansion_of_sum}
\end{align}
where: 
$(a)$ follows from the expression of $\denom{H}{K}{v}$ in \eqref{eq:def_of_ppoly_explicit};
$(b)$ follows by reindexing the summation via $n=m+1$;
and $(c)$ follows from \lemmaref{lemma:elementary_symmetric_polynomial_sum_identity}.
We obtain
\begin{align}
    \mathtt{B} \;
    &\stackrel{(a)}{\eqdef}
    \denom{H}{K-1}{v\minind{k}} 
        -  \sum_{\indneq{i}{\nset{K}}{k}} \denom{H}{K-2}{v\minind{\{i,k\}}}
    \nonumber
    \\ 
    &\stackrel{(b)}{=} 
     \sum_{m=0}^{K-1} (-1)^{m} \cdot \fallingfact{H}{m}  \cdot \esp{(K - 1) - m}{v\minind{k}}  
        + \sum_{m=0}^{K-1} (-1)^{m} \cdot m \cdot \fallingfact{H}{m-1} \cdot \esp{(K-1) - m}{v\minind{k}}
    \nonumber \\ 
    &=
     \sum_{m=0}^{K-1} (-1)^{m} \cdot \esp{(K - 1) - m}{v\minind{k}} \cdot 
        \left(  \fallingfact{H}{m} + m \cdot \fallingfact{H}{m-1}  \right)
    \nonumber \\ 
    &\stackrel{(c)}{=}
    \sum_{m=0}^{K-1} (-1)^{m} \cdot \fallingfact{(H+1)}{m} \cdot \esp{(K - 1) - m}{v\minind{k}}
    \nonumber \\ 
    &= 
    \denom{H+1}{K-1}{v\minind{k}},
\end{align}
where:
$(a)$ follows from our definition of $\mathtt{B}$ in \eqref{eq:proof_of_bp_inductive_step_part_2_def_B};
$(b)$ follows from the expression of $\denom{H}{K}{v}$ in \eqref{eq:def_of_ppoly_explicit} and a substitution of  \eqref{eq:simplification_of_A_and_B_expansion_of_sum};
and $(c)$ follows from the identity $\fallingfact{T}{m} + m \cdot \fallingfact{T}{m-1} = \fallingfact{(T+1)}{m}$
in \lemmaref{lemma:falling_factorial_identities_2}.

Similarly to the derivation of \eqref{eq:simplification_of_A_and_B_expansion_of_sum}, 
by the definition of $\num{H}{K}{v}$ in \eqref{eq:def_of_num_poly},
it holds that
\begin{equation} \label{eq:simplification_of_A_and_B_expansion_of_sum_similarly}
    \sum_{\indneq{i}{\nset{K}}{k}} \num{H}{K-2}{v\minind{\{i,k\}}} 
    = \sum_{n=0}^{K-1} (-1)^{n-1} \cdot n \cdot \fallingfact{H}{n} \cdot \esp{(K - 1) - n}{v\minind{k}}.
\end{equation}
It follows that
\begin{align*}
    \mathtt{A}  
    &\stackrel{(a)}{\eqdef} 
    \denom{H}{K-1}{v\minind{k}}
        + \num{H}{K-1}{v\minind{k}}
        - \sum_{\indneq{i}{\nset{K}}{k}} \num{H}{K-2}{v\minind{\{i,k\}}}
    \\ 
    &\stackrel{(b)}{=}   
    \sum_{m=0}^{K-1} (-1)^{m} \cdot \fallingfact{H}{m} \cdot \esp{(K - 1) - m}{v\minind{k}} 
        + H \cdot \sum_{m=0}^{K-1} (-1)^{m} \cdot \fallingfact{(H-1)}{m} \cdot \esp{(K - 1) - m}{v\minind{k}}
        \\ &\qquad + \sum_{m=0}^{K-1} (-1)^{m} \cdot m \cdot \fallingfact{H}{m} \cdot \esp{(K-1) - m}{v\minind{k}}
    \\ 
     &\stackrel{(c)}{=}   
    \sum_{m=0}^{K-1} (-1)^{m} \cdot \esp{(K - 1) - m}{v\minind{k}}
        \left(
            (m+1) \cdot \fallingfact{H}{m} + \fallingfact{H}{m+1}
        \right)
    \\
    &\stackrel{(d)}{=}
    \sum_{m=0}^{K-1} (-1)^{m} \cdot \fallingfact{(H+1)}{m+1} \cdot \esp{(K - 1) - m}{v\minind{k}}
    \\ 
    &= 
    \num{H+1}{K-1}{v\minind{k}},
\end{align*}
where the steps are justified as follows: 
\begin{enumerate}[label={$(\alph*)$}]
    \item 
    Follows from our definition of $\mathtt{A}$ in \eqref{eq:proof_of_bp_inductive_step_part_2_def_A}.
    \item 
    Follows from the expression of $\denom{H}{K}{v}$ in \eqref{eq:def_of_ppoly_explicit},
    the definition of $\num{H}{K}{v}$ in \eqref{eq:def_of_num_poly},
    and a substitution of \eqref{eq:simplification_of_A_and_B_expansion_of_sum_similarly}.
    \item 
    Follows from the identity $H \cdot \fallingfact{(H-1)}{m} = \fallingfact{H}{m+1}$.
    \item 
    Follows from the identity  $\fallingfact{T}{m} + m \cdot \fallingfact{T}{m-1} = \fallingfact{(T+1)}{m}$\, in \lemmaref{lemma:falling_factorial_identities_2}.
\end{enumerate}

\subsubsection{Proof of Remark \ref{remark:validity_of_r} } \label{appendix:validity_of_r}

\begin{lemma}\label{lemma:validity_of_r_sum_to_1}
    Let $H,K \in \Nplus$. If $v \in \vvset{H,K}$ then
    \begin{equation*}   
        \sum_{k=1}^{K} \denom{H-1}{K-1}{v\minind{k}}  = \denom{H-1}{K}{v}
    \end{equation*}
\end{lemma}
The lemma must hold by construction:
for $H=1$, $v$ was defined to satisfy condition \ref{base_case_condition_1};
for $H>1$, it satisfy condition \ref{condition_induction_1}.
For the sake of completeness, we provide an explicit proof.

\begin{proof}[Proof of \lemmaref{lemma:validity_of_r_sum_to_1}]
Let $H,K \in \Nplus$ and $v \in \vvset{H,K}$.
It holds that
\begin{align*}
    \sum_{k=1}^{K} \denom{H-1}{K-1}{v\minind{k}} 
    &= 
    \sum_{k=1}^{K} \sum_{m=0}^{K-1} (-1)^{m} \fallingfact{(H-1)}{m} \esp{(K-1)-m}{v\minind{k}} 
    \\
    &=
    \sum_{m=0}^{K-1} (-1)^{m} \fallingfact{(H-1)}{m} \sum_{k=1}^{K} \esp{K-(m+1)}{v\minind{k}} 
    \\
    &\stackrel{(a)}{=}
    \sum_{m=0}^{K-1} (-1)^{m} \fallingfact{(H-1)}{m} \cdot (m+1) \cdot  \esp{K-(m+1)}{v} 
    \\
    &\stackrel{(b)}{=}
    \sum_{n=1}^{K} (-1)^{n-1} \fallingfact{(H-1)}{n-1} \cdot (n) \cdot \esp{K-n}{v}
    \\
    &\stackrel{(c)}{=}
    \sum_{n=1}^{K} (-1)^{n-1} \left(  \fallingfact{H}{n} - \fallingfact{(H-1)}{n} \right) \cdot \esp{K-n}{v} 
    \\
    &=
    - \left( {\sum_{n=1}^{K} (-1)^{n}  \fallingfact{H}{n} \cdot \esp{K-n}{v}} \right)
    + \left( {\sum_{n=1}^{K} (-1)^n  \fallingfact{(H-1)}{n} \esp{K-n}{v}} \right)
    \\
    &=
    - \left( \denom{H}{K}{v} - \esp{K}{v} \right)
    + \left( \denom{H-1}{K}{v} - \esp{K}{v} \right)
    \\
    &= 
    - \denom{H}{K}{v} + \denom{H-1}{K}{v}
    \\
    &\stackrel{(d)}{=}
    \denom{H-1}{K}{v}
\end{align*}
where the steps are justified as follows:  
\begin{enumerate}[label={$(\alph*)$}]
\item 
Follows from \lemmaref{lemma:elementary_symmetric_polynomial_sum_identity}.
\item 
Follows from the substitution $n=m+1$.
\item 
Follows from the identity $n \cdot \fallingfact{(H-1)}{n-1} + \fallingfact{(H-1)}{n} = \fallingfact{H}{n}$ stated in \lemmaref{lemma:falling_factorial_identities_2}.
\item 
By \lemmaref{lemma:proof_of_bp_2},  any $v \in \vvset{H,K}$ satisfies $\denom{H}{K}{v} = 0$.
\end{enumerate}
\end{proof}

\begin{proof}[Proof of \remarkref{remark:validity_of_r}]
Let $H,K \in \Nplus$ and $v \in \vvset{H,K}$.
We prove the vector $r$, given by
\begin{equation} \tag{\ref{eq:optimal_odd_thm}}    
    r(k) = \frac{\denom{H-1}{K-1}{v\minind{k}}}{\denom{H-1}{K}{v}}
\end{equation}
for every $k \in [K]$, is a probability distribution  with non-zero elements.
\begin{itemize}
    \item 
    \emph{Sum Constraint:}
    Holds by \lemmaref{lemma:validity_of_r_sum_to_1}.

    \item 
    \emph{Positivity Constraint:} 
    We distinguish two cases based on the value of $H$.
    \begin{itemize}
        \item 
        \emph{Case $H=1$:} 
        For every $k \in [K]$
        \begin{equation} \tag{\ref{eq:proof_of_bpe_r_equation}}
            \vecind{r}{k} 
            = \frac{1}{\vecind{v}{k}}.
        \end{equation}
        By \lemmaref{lemma:lower_bound_for_value_vectors}, $v \succeq \mathbf{1}_K$; 
        thus, $r$ satisfies the positivity constraint.

        \item 
        \emph{Case $H>1$:}
        For every $k \in [K]$ the vector $\prescript{k}{}{v} \eqdef v - \basis{k}/r(k)$
        must be in $\vvset{H-1,K}$. 
        By the proof of \lemmaref{lemma:proof_of_bp_3} (see \lemmaref{lemma:derivation_of_proof_of_bp_3}) for every $m \in [K]$ and $\mathfrak{I} \in \binom{[K]}{m}$, 
        \(
            \denom{H-1}{K-m}{\prescript{k}{}{v}\minind{\mathfrak{J}}} > 0.
        \)
        Choosing $\mathfrak{I}=\{k\}$ gives 
        \[
            \denom{H-1}{K-1}{v\minind{k}} 
            = \denom{H-1}{K-1}{\prescript{k}{}{v}\minind{k}}  
            > 0.
        \]
        Thus, for every $k \in [K]$, the numerator in \eqref{eq:optimal_odd_thm} is positive. 
        By \lemmaref{lemma:validity_of_r_sum_to_1}, 
        the denominator is a sum of positive numbers and hence positive.
        Thus, $r$ satisfies the positivity constraint.
    \end{itemize}
\end{itemize}

\end{proof}

\subsection{Technical Details for Appendix \ref{appendix:main_results} }
\label{appendix:technical_D}

\subsubsection{Derivation of the Explicit Form of \teq{$\hatppolyfunction{T}{K}$} in \teq{\eqref{eq:expansion_of_hatppoly}}}
\label{appendix:derivation_of_the_explicit_form_of_tilde_P}
\begin{align*}
    \hatppoly{T}{K}{x} 
    &\stackrel{(a)}{\eqdef}
    \mathcal{P}_{T,K}(x+T) 
    \\ 
    &= 
    \sum_{m=0}^{K} (-1)^m \binom{K}{K-m} \,\fallingfact{T}{m}\, (x+T)^{\,K-m}
    \\ 
    &\stackrel{(b)}{=} 
        \sum_{n=0}^{K} (-1)^{K-n} \binom{K}{n}\,\fallingfact{T}{K-n}\,(x+T)^{n}.
    \\ 
    &\stackrel{(c)}{=} 
    \sum_{n=0}^{K} (-1)^{K-n} \binom{K}{n}\,\fallingfact{T}{K-n}\, \sum_{m=0}^{n} \binom{n}{m} x^{m} T^{n-m}
    \\
    &=
    \sum_{n=0}^{K} \sum_{m=0}^{n} 
         x^m (-1)^{K-n} \binom{K}{n} \binom{n}{m}\,\fallingfact{T}{K-n} \,T^{n-m} 
   \\ 
   &\stackrel{(d)}{=}  
   \sum_{n=0}^{K} \sum_{m=0}^{n} 
     x^m  \binom{K}{m} (-1)^{K-n} \binom{K-m}{K-n}\,\fallingfact{T}{K-n} \,T^{n-m} 
   \\ 
   &= 
    \sum_{m=0}^{K} x^{m}\,\binom{K}{m}\left(\sum_{n=m}^{K} (-1)^{K-n} \binom{K-m}{K-n}\, T^{n-m}\,\fallingfact{T}{K-n}\right)
    \\ 
    &\stackrel{(e)}{=}     
    \sum_{m=0}^{K} x^{m}\,\binom{K}{m}\left(\sum_{d=0}^{K-m} (-1)^{d} \binom{K-m}{d}\, T^{K-m-d}\,\fallingfact{T}{d}\right)
    \\ 
    &\stackrel{(f)}{=} 
    \sum_{m=0}^{K} x^{m}\,\binom{K}{m}\left(\sum_{d=0}^{K-m} \sum_{i=0}^{d} (-1)^{d} \binom{K-m}{d}\, \stirling{d}{i}\, T^{K-(m+d-i)}\right)
    \\
    &\stackrel{(g)}{=} 
    \sum_{n=0}^{K} x^{K-n}\,\binom{K}{n}\left(\sum_{d=0}^{n} \sum_{i=0}^{d} (-1)^{d} \binom{m}{d}\, \stirling{d}{i}\, T^{n-(d-i)}\right),
\end{align*}
where the steps are justified as follows: 
\begin{enumerate}[label={$(\alph*)$}]
\item 
By definition of $\hatppolyfunction{T}{K}$ in \eqref{eq:expansion_of_hatppoly}.
\item 
Substitution  $n=K-m$.
\item 
From the binomial expansion
\(
    (x + T)^n \;=\; \sum_{m=0}^n \binom{n}{m}\,x^m\,T^{\,n-m}.
\)
\item 
By the identity
\(
    \binom{K}{n} \binom{n}{m} = \binom{K}{m} \binom{K-m}{n-m}.
\)
\item 
Substituting $d=K-n$.
\item 
By \lemmaref{lemma:falling_factorial_expansion}, 
which expresses the falling factorial in terms of Stirling numbers of the first kind.
\item 
Substituting $n=K-m$.
\end{enumerate}

\subsubsection{Derivation of the Explicit Form of \teq{$\betapolyfunction{T}{K}$} in 
\teq{\eqref{eq:expnasion_of_betapoly}}}
\label{appendix:derivation_of_the_explicit_form_of_betapoly}

\begin{align*}
    \betapoly{T}{K}{x} 
    &\stackrel{(a)}{=}
    \hatppoly{T}{K}{\sqrt{T} x}
    \\
    &\stackrel{(b)}{=} 
        \sum_{m=0}^{K} T^{\frac{K-m}{2}} x^{K-m} \,\binom{K}{m}\left(\sum_{d=0}^{m} \sum_{i=0}^{d} (-1)^{d} \binom{m}{d}\, \stirling{d}{i}\, T^{m-(d-i)}\right)
    \\
    &=
    \sum_{m=0}^{K} x^{K-m} \,\binom{K}{m} \, T^{\frac{K}{2}}
        \left(\sum_{d=0}^{m} \sum_{i=0}^{d} (-1)^{d} \binom{m}{d}\, \stirling{d}{i}\, T^{\frac{m}{2} -(d-i)}\right),
\end{align*}
where $(a)$ follows from our definition of $\betapolyfunction{T}{K}$ in \eqref{eq:expnasion_of_betapoly}, 
and in $(b)$ we substitute \eqref{eq:expansion_of_hatppoly}.
Denote
\begin{equation*}
    \widetilde{c}_{T,m} \eqdef \sum_{d=0}^{m} \sum_{i=0}^{d} (-1)^{d} \binom{m}{d}\, \stirling{d}{i}\, T^{\frac{m}{2}-(d-i)}.
\end{equation*}
It holds that 
\begin{align*}
    \widetilde{c}_{T,m}
    &= \sum_{d=0}^{m} \sum_{n=0}^{d} (-1)^{d} \binom{m}{d}\, \stirling{d}{d-n}\, T^{\frac{m}{2} -n }
    \\ 
    &=
    \sum_{n=0}^{m} \sum_{d=n}^{m} (-1)^{d} \binom{m}{d}\, \stirling{d}{d-n}\,T^{\frac{m}{2} -n }
    \\ 
    &= 
    \sum_{n=0}^{m} T^{\frac{m}{2} - n} \sum_{d=0}^{m} (-1)^{d} \binom{m}{d}\, \stirling{d}{d-n}  ,
\end{align*}
where the first equality follows by re-indexing with $n=d-i$.

\end{document}